\documentclass{article}

\PassOptionsToPackage{numbers,compress}{natbib}
\usepackage[main, final]{neurips_2026}

\usepackage[utf8]{inputenc} 
\usepackage[T1]{fontenc}    
\usepackage{hyperref}       
\usepackage{url}            
\usepackage{booktabs}       
\usepackage{amsfonts}       
\usepackage{microtype}      
\usepackage{xcolor}         
\newcommand{\bestmean}[1]{\begingroup\setlength{\fboxsep}{1pt}\colorbox{black!10}{\ensuremath{#1}}\endgroup}
\usepackage{amsmath}
\usepackage{amsthm}
\usepackage{amssymb}
\usepackage{algorithm}
\usepackage{algorithmic}
\usepackage{graphicx}
\usepackage{subcaption}
\usepackage{pdflscape}
\usepackage{tikz}
\usetikzlibrary{arrows.meta,calc}

\theoremstyle{plain}
\newtheorem{theorem}{Theorem}[section]
\newtheorem{lemma}[theorem]{Lemma}

\newtheorem{proposition}[theorem]{Proposition}

\theoremstyle{definition}
\newtheorem{definition}[theorem]{Definition}
\newtheorem{assumption}[theorem]{Assumption}

\theoremstyle{remark}
\newtheorem{remark}[theorem]{Remark}

\title{Active Learning for Conditional Generative Compressed Sensing}

\author{%
  Alexander DeLise \\ 
  Department of Scientific Computing \\
  Department of Mathematics \\
  Florida State University \\
  \texttt{ard22l@fsu.edu} \\
  \And
  Nick Dexter \\
  Department of Scientific Computing \\
  Florida State University \\
  \texttt{nick.dexter@fsu.edu}
}

\DefineNamedColor{named}{Purple}{cmyk}{0.45,0.86,0,0}

\DefineNamedColor{named}{JungleGreen} {cmyk}{0.99,0,0.52,0}

\DefineNamedColor{named}{orange} {rgb}{1,0.55,0}

\DefineNamedColor{named}{burgundy} {rgb}{0.61, 0, 0}

\DefineNamedColor{named}{forestgreen} {rgb}{0.13,0.54,0.13}

\DefineNamedColor{named}{oceanfoam} {rgb}{0.34, 0.52, 0.54}

\newcommand*{\snorm}[2]{\left\vert\!\left\vert\!\left\vert #1 \right\vert\!\right\vert\!\right\vert_{#2}}

\begin{document}

\maketitle

\begin{abstract}
Generative compressed sensing uses the range of a pretrained generator as a nonlinear model for recovering structured signals from limited measurements. We study a conditional version of this problem for image recovery from subsampled Fourier measurements using prompt-conditioned generative models. Our framework separates two roles of conditioning: the prompt used to design the sampling distribution and the prompt used to define the recovery model. For ReLU and Lipschitz conditional generators, we prove stable recovery bounds showing that prompt-matched Christoffel sampling retains the same Christoffel complexity constant as existing near-optimal generative compressed sensing theory, while prompt mismatch incurs an explicit compatibility penalty. Experiments with Stable Diffusion show that prompts meaningfully reshape Christoffel sampling distributions and influence image recovery. Overall, our results suggest treating prompts as design variables for sensing, approximation, and recovery, opening new directions for the practical use of conditional generative models in compressed sensing, where conditioning, sampling, and optimization must be considered jointly.
\end{abstract}

\section{Introduction}
Compressed sensing~\cite{donoho2006compressed} is a popular signal processing technique to recover a signal $\mathbf{f}^* \in \mathbb{C}^n$ from an undersampled observation $\mathbf{y} \in \mathbb{C}^m$, with $m \ll n$, obtained via some forward process
\begin{equation}
    \mathbf{y} = \mathbf{A} \mathbf{f}^* + \mathbf{e}, \label{eq:CS}
\end{equation}
with numerous applications throughout science and engineering~\cite{chen2008prior,lustig2007sparse,duarte2008single,yuan2021snapshot,herrmann2012fighting}.
Here $\mathbf{A} \in \mathbb{C}^{m \times n}$ is the sampling matrix (with possibly random entries) and $\mathbf{e} \in \mathbb{C}^m$ is a vector of additive noise. 

The inverse problem in Equation~\eqref{eq:CS} is underdetermined. Thus one typically assumes some application-driven structure, usually sparsity of $\mathbf{f}^*$ in some domain, to usher a unique solution~\cite{candes2006robust, candes2006stable, tibshirani1996regression, hastie2015statistical, pati1993orthogonal, tropp2007signal, chen2001atomic, adcock2021compressive}, provided the sampling matrix $\mathbf{A}$ satisfies some condition like the Restricted Isometry Property~\cite{donoho2006compressed, candes2006stable}. Generative compressed sensing, on the other hand, replaces the sparsity assumption by instead requiring that $\mathbf{f}^*$ lies close to the range of a generative model $G \colon \mathbb{R}^k \to \mathbb{R}^n$, as first introduced in~\cite{bora2017compressed}. In this regime, $\mathbf{f}^*$ is recoverable with high probability when the sampling matrix $\mathbf{A}$ with Gaussian entries satisfies the so-called Set-Restricted Eigenvalue Condition. Several subsequent studies investigate bounds for signal recovery in generative compressed sensing~\cite{kamath2019lower, kamath2020power, liu2020information} as well as recovery guarantees under the presence of outlier data~\cite{jalal2020robust} or with nonlinear sampling operators~\cite{chen2023unified}.

While these results establish a powerful alternative to sparsity-based recovery, they rely on assumptions about the measurement process that are often not met in practice. First, Gaussian measurements can be unrealistic: in many practical applications, sensing is dictated by hardware constraints and corresponds to structured, often subsampled unitary transforms, e.g., the discrete Fourier transform in MRI~\cite{lustig2007sparse, lustig2008compressed}. Second, even after restricting to such physically realizable operators, the choice of sampling pattern remains critical. Common strategies such as Gaussian or uniform subsampling are uninformed, as they fail to exploit the structure of the signal class being measured and may lead to suboptimal performance~\cite{adcock2017breaking, krahmer2015compressive, poon2017structure}.

Several studies propose sampling strategies for generative compressed sensing that exploit the geometry of the underlying generative model, a process called \emph{active learning}. For example,~\cite{adcock2023cs4ml} develops the Christoffel Sampling for Machine Learning framework based on the \emph{generalized Christoffel function}, while~\cite{berk2023modeladapted, berk2022fourier, plan2025denoising} use \emph{local coherences} to design sampling distributions that place more mass on measurement coordinates that are most informative for distinguishing signals in the model class. Such active learning approaches relate closely to the well-studied \emph{leverage score sampling}, with various applications in standard regression and imaging~\cite{chen2016statistical,derezinski2018leveraged,ordozgoiti2022generalized,erdelyi2020fourier,avron2019universal}, as well as optimal sampling strategies for compressed sensing in general~\cite{adcock2021compressive, krahmer2015compressive, poon2017structure, candes2007compressive, krahmer2013stable, puy2011variable}. 

Despite these advances, existing work in generative compressed sensing has largely relied on unconditional generative priors, thereby omitting auxiliary information that may be available at acquisition or inference time. In many realistic settings, such information can sharply restrict the set of plausible signals. For example, in medical imaging, modality, anatomy, or patient metadata may improve reconstruction~\cite{chung2025contextmri}; similarly, in natural image reconstruction, subject, scene category, acquisition protocol, or other metadata may narrow the relevant signal class and improve reconstruction quality~\cite{chung2024prompttune, kim2025regularization}. Prompted or otherwise conditional generators provide a natural mechanism for incorporating this information by steering the recovery class toward signals compatible with the available context, rather than requiring a single unconditional model to represent all possible signal types.

Conditioning can also improve sensing itself. If the relevant conditional signal class is narrower, then one may design sampling distributions supported on more informative measurement coordinates, leading to lower sample complexity for stable recovery. However, conditioning introduces interactions absent from the unconditional setting: the prompt used for sampling need not match the prompt used for recovery, and neither need match the class containing the true signal when $\mathbf f^*$ lies in the range of the generator. Thus, overly strong or misaligned conditioning may bias the reconstruction toward the prior and away from the observed measurements. This motivates \emph{conditional generative compressed sensing} as an extension of the standard framework in which both sampling design and recovery guarantees depend explicitly on the interaction between conditional signal classes. Understanding how these interactions affect stable recovery and reconstruction is the central question of this paper.

\subsection{Contributions}
We adapt an active learning perspective via Christoffel sampling from~\cite{adcock2023cs4ml} together with generative compressed sensing arguments from~\cite{bora2017compressed, berk2023modeladapted, berk2022fourier} to study prompt-conditioned signal recovery under subsampled Fourier measurements. The resulting theory characterizes how interactions between the sampling, recovery, and true-signal prompts determine the sample complexity required for stable recovery and the resulting reconstruction error. Our contributions may be summarized as follows:
\begin{itemize}
    \item We formulate a prompt-conditioned Christoffel sampling framework for generative compressed sensing under subsampled Fourier measurements with three conditioning prompts: the true-signal prompt $c_*$, the recovery prompt $c_r$, and the sampling prompt $c_s$.

    \item We identify a single prompt compatibility factor $\Lambda(c_1,c_2,c_3)$ and show that evaluating it at $(c_r,c_r,c_s)$ controls sample complexity for stable reconstruction, while evaluating it at $(c_*,c_r,c_s)$ controls the residual error in recovery.

    \item We derive explicit sample complexity results under ReLU and Lipschitz assumptions on the conditional generator and combine these with agnostic signal recovery error bounds.

    \item We empirically test the theoretical framework through a suite of experiments demonstrating that prompt conditioning systematically influences both the induced sampling distributions and reconstruction performance.
\end{itemize}

\section{Active learning problem setup}\label{sec: problem setup}

We follow the general Christoffel Sampling for Machine Learning (CS4ML) framework established in~\cite{adcock2023cs4ml} for our active learning regime, however we specialize our results to subsampled Fourier measurements. For a full description of the general CS4ML framework, we refer the reader to Appendix~\ref{app:cs4ml}.

\subsection{Setup and prompt-conditioned model classes}
Let $\mathbb{X} = \mathbb{R}^n$, deemed the \emph{object space}, be equipped with Euclidean norm $\left\| \cdot \right\|_2$. Let $\mathbf{F} \in \mathbb{C}^{n \times n}$ denote the unitary discrete Fourier transform, let $D = \{1, \dots, n\}$, and let $\mathbf{P}_i$ be the row selector matrix for the $i$th Fourier coefficient. For a signal $\mathbf{f} \in \mathbb{X}$, define
\[
  L_i(\mathbf{f}) = \mathbf{P}_i \mathbf{F}\mathbf{f} \in \mathbb{C}, \qquad i \in D
\]
to be the \emph{sampling operator}, which extracts the $i$th Fourier coefficient of $\mathbf{f}$. We assume that the full family of sampling operators $\{L_i\}_{i \in D}$ is nondegenerate on $\mathbb{X}$, i.e., that there exist constants $0 < \alpha \le \beta < \infty$ such that
\[
  \alpha \left\|\mathbf{f}\right\|_2^2
  \le
  \sum_{i \in D} |L_i(\mathbf{f})|^2
  \le
  \beta \left\|\mathbf{f}\right\|_2^2,
  \qquad
  \mathbf{f} \in \mathbb{X}.
\]
In our setting, this holds with $\alpha = \beta = 1$ by unitarity of the discrete Fourier transform $\mathbf{F}$. Later, after introducing subsampling, we will require a sampling analogue of this property, deemed \emph{empirical nondegeneracy}, however we defer this notion to Section~\ref{sec: theory}. These properties ensure that the measurement map is injective and norm-preserving on $\mathbb{X}$, so that differences between signals are reflected in the measurements, thereby enabling stable and distinguishable recovery.

Next, let $G \colon \mathbb{R}^k\times\mathcal C\to\mathbb{X}$ be a conditional generative model which takes a latent vector $\mathbf{z} \in \mathbb{R}^k$ and conditioning $c \in \mathcal{C}$ as input, and outputs a signal $\mathbf{f} = G(\mathbf{z},c) \in \mathbb{X}$. Here, $\mathcal C$ denotes an abstract conditioning space, which may represent labels, metadata, measurements, or text embeddings, though in this work, we specialize to the case where conditioning is given by text prompts. In our subsequent analyses, we further assume the input latent vector $\mathbf{z}$ has bounded norm and let \mbox{$B_2^k(R) = \left\{ \mathbf{z} \in \mathbb{R}^k \, : \, \left \| \mathbf{z} \right \|_2 \leq R \right\}$}. Many generative models (e.g., diffusion-based~\cite{song2021scorebased,ho2020denoising,sohl2015deep}, GANs~\cite{goodfellow2014generative,radford2016unsupervised}, and VAEs~\cite{kingma2014auto}) assume the entries of $\mathbf{z}$ are Gaussian, therefore this assumption omits only an exponentially small fraction of inputs in practical applications~\cite{bora2017compressed}. Lastly, define
\[
  \mathbb{F}_c = \left\{ G(\mathbf{z}, c) \colon \mathbf{z} \in B_2^k(R)\right\} \subseteq \mathbb{X},
\]
termed the \emph{approximation space}, to be the range of the generative model conditioned on a prompt $c$. 

In practice, we distinguish three different prompts throughout our analyses. First, we consider the signal prompt $c_*$, which manifests when our quantity of interest $\mathbf{f}^*$ is explicitly in the range of the generative model, i.e., $\mathbf{f}^* \in \mathbb{F}_{c_*}$. In general, $c_*$ may be unknown, or $\mathbf{f}^*$ might be out of distribution, in which case we also consider the recovery prompt $c_r$ which defines the set of candidate signals used to reconstruct $\mathbf{f}^*$ using $G$, that is, using signals in $\mathbb{F}_{c_r}$. Finally, we consider the sampling prompt $c_s$ used to construct our sampling distribution $\mu_{c_s}$, which we introduce in the following subsection. 

As we will see, a mismatch $c_* \neq c_r$ is a \emph{signal-recovery mismatch} and enters as an approximation error, whereas a mismatch $c_s \neq c_r$ is a \emph{sampling-recovery mismatch} and enters through the sample complexity required for stable signal recovery. Likewise, a mismatch $c_* \neq c_s$ is a \emph{signal-sampling mismatch}, meaning that the measurements are designed using prompt information that may not reflect the true type of signal being acquired, and in our analyses this manifests through how well the resulting sampling distribution captures the features of the target signal relevant for recovery.

\subsection{Christoffel sampling and prompt compatibility}
A central component of active learning is the principled selection of measurements according to an optimized sampling strategy, which may be determined prior to data collection or updated adaptively as observations are acquired, in contrast to fixed or uniform sampling. In our setting, this corresponds to selecting indices $i \in D$ at which to query the sampling operators $L_i$. The goal is to allocate measurements so that they are maximally informative for distinguishing elements of the model class under consideration.

One such active sampling scheme is \emph{Christoffel sampling}, which harnesses the \emph{generalized Christoffel function}~\cite{adcock2023cs4ml, adcock2025optimal}. For any nonempty set $\mathbb S \subseteq \mathbb X$, define the Christoffel function of $\mathbb S$ in our subsampled Fourier measurement setting by
\[
  K(\mathbb S)(i)
  =
  \sup_{\mathbf{h} \in \mathbb S\setminus\left\{ \mathbf{0} \right\}}
  \frac{\left| L_i(\mathbf{h}) \right|^2}{\left\| \mathbf{h} \right\|_2^2}
  =
  \sup_{\mathbf{h} \in \mathbb S\setminus\left\{ \mathbf{0} \right\}}
  \frac{\left| \mathbf{P}_i \mathbf{F}\mathbf{h} \right|^2}{\left\| \mathbf{h} \right\|_2^2},
  \qquad i \in D,
\]
as well as the related quantity $\kappa(\mathbb S) = \sum_{i=1}^n K(\mathbb S)(i)$. In particular, for a prompt $c \in \mathcal{C}$ we apply this construction to the self-difference class $\mathbb{F}_c - \mathbb{F}_c = \left\{ \mathbf{f}_1 - \mathbf{f}_2 \colon \mathbf{f}_1, \mathbf{f}_2 \in \mathbb{F}_c \right\}$. We refer to elements in such sets as \emph{secants}. Since recovery is performed over a prompt-conditioned class $\mathbb{F}_{c_r}$, the notion of informativeness is inherently model-dependent and governed by how signals in $\mathbb{F}_{c_r}$ differ from one another. Stable recovery requires that measurements capture directions along which elements of $\mathbb{F}_{c_r}$ can vary, i.e., the geometry of the self-difference class $\mathbb{F}_{c_r} - \mathbb{F}_{c_r}$. This is quantified by the Christoffel function $K(\mathbb{F}_{c_r} - \mathbb{F}_{c_r})(i)$, which measures how much secants of $\mathbb{F}_{c_r}$ can concentrate in the $i$th Fourier coordinate. Larger values therefore correspond to coordinates that are more informative for distinguishing elements of the prompt-conditioned model class.

Under the standard active learning setup, we assume we have freedom to query $L_i$ for any $i \in D$, which allows us to design model-adapted sampling distributions using the generalized Christoffel function. The corresponding regularized Christoffel sampling distribution is the probability distribution
\begin{equation}\label{eq: christoffel sampling law}
  \mu_c(i)
  =
  (1-\zeta)\frac{K(\mathbb{F}_c - \mathbb{F}_c)(i)}{\kappa(\mathbb{F}_c - \mathbb{F}_c)}
  + \frac{\zeta}{n},
  \qquad i \in D,
\end{equation}
where $\zeta \in (0, 1)$ is a regularization parameter~\cite{adcock2025optimal,pozharska2022note}. Here, $\mu_c$ allocates more probability mass to Fourier coordinates that capture larger differences between signals in the prompt-conditioned model class and $\zeta$ controls the tradeoff between model-adapted Christoffel sampling and uniform sampling.

In the prompt-matched case, where the same prompt is used for both recovery and acquiring measurements, i.e., $c_r = c_s$, one would sample according to $\mu_{c_r}$. However, our framework allows the recovery prompt $c_r$ and the sampling prompt $c_s$ used to construct $\mu_{c_s}$ to differ. To quantify these prompt interactions with a single statistic, we define the following quantity:

\begin{definition}[Prompt compatibility factor]
  Let $c_1,c_2,c_3 \in \mathcal{C}$. Define the \emph{prompt compatibility factor} by
  \[
    \Lambda(c_1,c_2,c_3)
    =
    \max_{i \in D}
    \frac{K(\mathbb{F}_{c_1}-\mathbb{F}_{c_2})(i)}{\mu_{c_3}(i)}.
  \]
  \label{def: prompt compatibility factor}
\end{definition}

\begin{remark}[Reduction to known sample complexity constants and novelty] 
    For the prompt-matched case in which we recover over the same sampling prompt ($c_r = c_s$), we have $\Lambda(c_r,c_r,c_r) \leq \kappa(\mathbb F_{c_r}-\mathbb F_{c_r}) / (1 - \zeta)$, reducing to the Christoffel sample complexity constant from~\cite{adcock2023cs4ml} up to the constant factor $\left(1 - \zeta\right)^{-1}$. More generally, recovering under a prompt different than that used to create the sampling distribution requires evaluating $\Lambda(c_r,c_r,c_s)$. Under our subsampled Fourier model, this recovers the model-adapted coherence sample complexity constant in~\cite{berk2023modeladapted, plan2025denoising}, which indeed is a special case of $\kappa$. What makes $\Lambda$ distinct in our conditional agnostic signal recovery setting is that it extends the usual self-difference analysis to cross-difference classes. When the target signal $\mathbf f^*$ is generated with a prompt $c_*$, the prompt compatibility factor $\Lambda$ not only controls sample complexity through $\mathbb F_{c_r}-\mathbb F_{c_r}$, but also quantifies the amplification of approximation error through $\mathbb F_{c_*}-\mathbb F_{c_r}$. As we later show, combining these two geometries under the same sampling distribution yields direct control of the approximation error in recovery.
\end{remark}

In the sampling-recovery mismatch case, $\Lambda(c_r,c_r,c_s)$ might blow up if Fourier frequencies needed for the recovery geometry under $c_r$ receive little positive mass under the sampling distribution induced by $c_s$. Of course, this immediately suggests that one should always set $c_r = c_s$ to circumvent this issue, though this is not always feasible in practice. Acquisition and reconstruction may occur at different times and with different levels of information. For example, at sensing time, one may only know coarse metadata or an initial prompt, and must therefore design measurements using $c_s$, whereas at reconstruction time one may have access to a more refined prompt and recover over $\mathbb{F}_{c_r}$. Moreover, one may wish to design a single sampling distribution that is effective across a family of possible recovery prompts, so that $c_s$ is chosen not to match one particular $c_r$ (or $c_*$), but rather as a compromise or universal design for several downstream recovery classes. Our theory quantifies how these mismatches change both the sample complexity needed for stable recovery and the resulting reconstruction error.

\section{Main theoretical results}\label{sec: theory}

Throughout this section, let $I_1,\dots,I_m$ be drawn independently from a sampling distribution $\mu_{c_s}$, write $\Omega= \{ I_1,\dots,I_m \}$, and let $\mathbf P_\Omega \in \mathbb R^{m\times n}$ denote the sampling matrix whose $j$th row is $\mathbf P_{I_j}$. To acquire measurements, we define the weighted subsampled Fourier operator
\[
  \mathbf A_\Omega
  =
  \frac{1}{\sqrt{m}} \mathbf W_\Omega^{1/2}\mathbf P_\Omega \mathbf F,
  \qquad
  \mathbf W_\Omega
  =
  \operatorname{diag}\left(
    \frac{1}{\mu_{c_s}(I_1)},
    \dots,
    \frac{1}{\mu_{c_s}(I_m)}
  \right),
\]
where $\mathbf F$ denotes the unitary discrete Fourier transform. We observe noisy weighted measurements
$
  \mathbf y = \mathbf A_\Omega \mathbf f^* + \mathbf e,
$
where $\mathbf e = \frac{1}{\sqrt{m}}\mathbf W_\Omega^{1/2}\mathbf u$ for some
$\mathbf{u}\in\mathbb C^m$ denoting unknown additive noise.

Next, we introduce the two near-isometry principles that allow for stable signal recovery. 

\begin{definition}[S-REC~\cite{bora2017compressed}]
  For $\mathbb{S} \subseteq \mathbb X$, we say that $\mathbf{A}_\Omega$ satisfies the Set-Restricted Eigenvalue Condition (S-REC) with parameters $\gamma>0$ and $q\geq0$, denoted $\mathrm{S\text{-}REC}(\mathbb S,\gamma,q)$, if
  \[
    \left\|\mathbf A_\Omega\left(\mathbf f_1-\mathbf f_2\right)\right\|_2
    \geq
    \gamma \left\|\mathbf f_1-\mathbf f_2\right\|_2
    -
    q,
    \qquad
    \mathbf f_1,\mathbf f_2\in\mathbb{S}.
  \]
\end{definition}

\begin{definition}[Empirical nondegeneracy]\label{def:empirical-nondegeneracy}
  Let $0<\tau<1$. We say that $\mathbf A_\Omega$ is empirically nondegenerate on $\mathbb{F}_{c_r}-\mathbb{F}_{c_r}$ with distortion $\tau$ if
  \[
    \left(1-\tau\right)\left\|\mathbf h\right\|_2^2
    \leq
    \left\|\mathbf A_\Omega \mathbf h\right\|_2^2
    \leq
    \left(1+\tau\right)\left\|\mathbf h\right\|_2^2,
    \qquad
    \mathbf h\in \mathbb{F}_{c_r}-\mathbb{F}_{c_r}.
  \]
\end{definition}

These conditions formalize the requirement that the measurement operator preserves the geometry of the relevant signal class. Empirical nondegeneracy ushers in the sample complexity guarantees via Christoffel sampling, which we then leverage for our uniform signal recovery guarantees. As empirical nondegeneracy implies the S-REC, we lay out the proposition below: 

\begin{proposition}[Empirical nondegeneracy implies S-REC]\label{prop: end implies srec}
  If $\mathbf A_\Omega$ is empirically nondegenerate on $\mathbb{F}_{c_r}-\mathbb{F}_{c_r}$ with distortion $\tau\in(0,1)$, then $
    \mathbf A_\Omega$ satisfies $ \mathrm{S\text{-}REC}(\mathbb{F}_{c_r},\sqrt{1-\tau},0)$.
\end{proposition}

The proof follows directly from Definition~\ref{def:empirical-nondegeneracy}, so we omit it. Unless explicitly stated otherwise, every high-probability statement below is with respect to the i.i.d.\ draw of $\Omega$ from $\mu_{c_s}$.

\subsection{Sampling-recovery mismatch sample complexity}
We now give explicit sample complexity conditions guaranteeing that the sampling matrix $\mathbf A_\Omega$ is stable on the recovery class under the assumption that the generator $G(\cdot, c_r)$ is a ReLU network, or more generally $L$-Lipschitz. For notational convenience, we define
$
  \underline{\mu}
  =
  \min_{i \in D}\mu_{c_s}(i) > 0
$
to be the minimum probability assigned to a Fourier frequency under the sampling distribution $\mu_{c_s}$. We first state the sample complexity for ReLU networks, which we define below, based on~\cite[Definition 2]{berk2022fourier}.

\begin{definition}[Conditional bias-free depth-$d$ ReLU generator]
  \label{def:bias-free-relu}
  Let $\mathcal C$ be a conditioning space and let $d\geq 2$. We say that
  $G \colon\mathbb R^k\times\mathcal C\to\mathbb R^n$ is a conditional
  bias-free depth-$d$ feedforward ReLU generator with layer widths
   $k=k_0\leq k_1,\dots,k_{d-1}$ and $k_d=n$ if, for each fixed condition $c$, the induced map $G(\cdot, c) \colon  \mathbb{R}^k \to \mathbb{R}^n$ has the form
  \[
    G(\mathbf z,c)
    =
    W_d(c)\sigma\!\left(
      W_{d-1}(c)\sigma\!\left(
        \cdots \sigma\!\left(W_1(c)\mathbf z\right)
      \right)
    \right),
  \]
  where $W_\ell(c)\in\mathbb R^{k_\ell\times k_{\ell-1}}$ for
  $\ell=1,\dots,d$, $\sigma(t)=\max\{t,0\}$ is applied coordinate-wise, and
  there are no additive bias vectors. We model conditioning abstractly by allowing the network weights to depend on $c$, so that each $c\in\mathcal C$ indexes a generator $G(\cdot,c)$.
\end{definition}

Appendix~\ref{app:relu-specialization} records the more general piecewise-linear theorem from which the following ReLU sample complexity bound follows. 

\begin{theorem}[Sample complexity for ReLU networks]
  \label{thm:relu-srec}
  Let $c_r\in\mathcal C$ be the recovery prompt used to condition the generator $G(\cdot, c_r)$, and let $c_s\in\mathcal C$ be the sampling prompt used to draw $\Omega$ from $\mu_{c_s}$ and build $\mathbf A_\Omega$.
  Assume that $G(\cdot,c_r)$ is a bias-free depth-$d$ feedforward ReLU network in the sense of Definition~\ref{def:bias-free-relu} with layer widths
  $
    k = k_0 \leq k_1,\dots,k_{d-1},
    k_d = n
  $, with $d \geq 2$, and define
  \[
    \overline{k}_{c_r}
    =
    \left(
      \prod_{\ell=1}^{d-1} k_\ell
    \right)^{1/(d-1)}.
  \]
  Fix $0<\tau,\delta<1$. If
  \[
m \gtrsim \tau^{-2}\Lambda(c_r,c_r,c_s)\left[
  k(d-1)\left(1+\log\left(\frac{\overline{k}_{c_r}}{k}\right)\right)
  +
  k\left(1+\log\left(\frac{1}{\tau\sqrt{\underline{\mu}}}\right)\right)
  +
  \log\left(\frac{1}{\delta}\right)
\right],
\]
  then with probability at least $1-\delta$ over the draw of $\Omega$, $\mathbf A_\Omega$ satisfies
  $
  \mathrm{S\text{-}REC}
  (
    \mathbb{F}_{c_r},
    \sqrt{1-\tau},
    0
  )
  $.
\end{theorem}

We next record the sample complexity for stable recovery under the assumption that $G(\cdot, c_r)$ is $L$-Lipschitz, whose proof is given in Appendix~\ref{app:lipschitz-specialization}.

\begin{theorem}[Sample complexity for $L$-Lipschitz networks]
  \label{thm:lipschitz-srec}
  Let $c_r\in\mathcal C$ be the recovery prompt used to condition the generator $G(\cdot, c_r)$, and let $c_s\in\mathcal C$ be the sampling prompt used to draw $\Omega$ from $\mu_{c_s}$ and build $\mathbf A_\Omega$.
  Assume that $G(\cdot,c_r)$ is $L$-Lipschitz on $B_2^k(R)$ with $L,R>0$. Fix $0<\tau,\delta<1$ and $\xi>0$.
 If
  \[
    m
    \gtrsim
    \tau^{-2}
    \Lambda(c_r,c_r,c_s)
    \left[
      k\log\left(
        1+
        \frac{LR}
        {\xi\tau\sqrt{\underline{\mu}}}
      \right)
      +
      \log\left(\frac{1}{\delta}\right)
    \right],
  \]
  then with probability at least $1-\delta$ over the draw of $\Omega$, $\mathbf A_\Omega$ satisfies
  $
  \mathrm{S\text{-}REC}
  (
    \mathbb{F}_{c_r},
    \sqrt{1-\tau},
    \sqrt{1-\tau}\,\xi
  )
  $.
\end{theorem}

\begin{remark}[Refinement and extension of \cite{adcock2024unified, adcock2026unifiedframeworklearningnonlinear}]
Theorem \ref{thm:relu-srec} extends the ReLU generative compressed sensing results of \cite[Corollary 5.1]{adcock2024unified} by establishing high-probability S-REC guarantees for conditional generators under subsampled Fourier measurements, with sampling-recovery mismatch quantified by the prompt compatibility factor. The more recent work \cite[Theorem 6.1]{adcock2026unifiedframeworklearningnonlinear} also considers Lipschitz generators, but additionally assumes positive homogeneity and restricts recovery to a target-dependent subset of the generator range. Our Theorem \ref{thm:lipschitz-srec} removes these restrictions, requiring only Lipschitz continuity on a bounded latent domain and establishing a high-probability S-REC guarantee with additive tolerance on the entire conditional model class. These refinements apply to subsampled Fourier measurements, whereas \cite{adcock2024unified,adcock2026unifiedframeworklearningnonlinear} accommodate more general linear measurement operators. 
\end{remark}

In both cases, conditioning affects sample complexity through alignment and coverage. Under sampling-recovery mismatch, the prompt compatibility factor $\Lambda(c_r,c_r,c_s)$ measures whether the sampling prompt $c_s$ places sufficient mass on the Fourier coordinates that distinguish signals in the recovery class $\mathbb F_{c_r}$, while $\underline{\mu}$ quantifies how close the sampling distribution is to neglecting coordinates altogether. If instead $c_s=c_r$, then $\Lambda(c_r,c_r,c_r) \leq \kappa(\mathbb F_{c_r}-\mathbb F_{c_r}) / \left(1 - \zeta\right)$, and we once again recover the existing Christoffel- and coherence-based sample complexity results for generative compressed sensing up to the constant  $\left(1 - \zeta\right)^{-1}$~\cite{adcock2023cs4ml, berk2023modeladapted, plan2025denoising}. Under prompt mismatch, however, $\Lambda$ can grow and $\underline{\mu}$ can shrink, reflecting the increased number of samples required to capture poorly represented directions for stable recovery.

\subsection{Recovery under the S-REC}\label{subsec:recovery}
Under the event that $\mathbf A_\Omega$ satisfies
$\mathrm{S\text{-}REC}(\mathbb{F}_{c_r},\gamma,q)$
for some $\gamma>0$ and $q\geq 0$, we can bound the recovery error. To do so, let $\widehat{\mathbf z}\in B_2^k(R)$ be an $\omega$-minimizer, meaning that
\begin{equation}\label{eq: omega minimizer}
  \left\|
  \mathbf y
  -
  \mathbf A_\Omega G(\widehat{\mathbf z},c_r)
  \right\|_2
  \le
  \inf_{\mathbf z\in B_2^k(R)}
  \left\|
  \mathbf y
  -
  \mathbf A_\Omega G(\mathbf z,c_r)
  \right\|_2
  +
  \omega,
\end{equation}
and likewise define $\widehat{\mathbf f}=G(\widehat{\mathbf z},c_r)$. Suppose the target signal $\mathbf f^*$ belongs to a different prompt-conditioned class $\mathbb F_{c_*}$. In this setting, $\Lambda(c_*,c_r,c_s)$ quantifies how the signal, recovery, and sampling prompt model classes interact, giving us the bound below:

\begin{theorem}[Prompt-mismatched recovery]
  \label{thm:prompt-mismatched-recovery}
 Let $c_r\in\mathcal C$ be the recovery prompt and let $\widehat{\mathbf z}$ denote an $\omega$-minimizer in the sense of Equation~\eqref{eq: omega minimizer} with estimator $\widehat{\mathbf f}=G(\widehat{\mathbf z},c_r)$. Let $c_s\in\mathcal C$ be the sampling prompt used to draw $\Omega$ from $\mu_{c_s}$ and build $\mathbf A_\Omega$, and let $c_* \in\mathcal C$ be the true-signal prompt so that $\mathbf f^* \in\mathbb{F}_{c_*}$.
  Suppose that $\mathbf A_\Omega$ satisfies $\mathrm{S\text{-}REC}(\mathbb{F}_{c_r},\gamma,q)$
  for some $\gamma>0$ and $q\geq 0$. Then
    \[
      \left\|\widehat{\mathbf f}-\mathbf f^*\right\|_2
      \le
      \left(
        1+
        \frac{2\sqrt{\Lambda(c_*,c_r,c_s)}}{\gamma}
      \right)
      \inf_{\mathbf f\in\mathbb{F}_{c_r}}
      \left\|\mathbf f^*-\mathbf f\right\|_2
      +
      \frac{1}{\gamma}
      \left(
        \frac{2}{\sqrt{m\underline{\mu}}}
        \left\|\mathbf u\right\|_2
        +
        \omega+q
      \right).
    \]
\end{theorem}

Substituting Theorem~\ref{thm:relu-srec} yields $\gamma=\sqrt{1-\tau}$ and $q=0$, while substituting Theorem~\ref{thm:lipschitz-srec} yields $\gamma=\sqrt{1-\tau}$ and $q=\sqrt{1-\tau}\xi$, producing the additional additive error term $\xi$. The recovery bound decomposes into an approximation error over the recovery class $\mathbb F_{c_r}$, an amplification factor governed by $\Lambda(c_*,c_r,c_s)$, and an additive term due to measurement noise, optimization error, and S-REC slack. If $\mathbf f^* \in \mathbb F_{c_r}$, this approximation term vanishes. Otherwise, the mismatch is controlled by the cross-class Christoffel function $K(\mathbb F_{c_*}-\mathbb F_{c_r})$ through $\Lambda(c_*,c_r,c_s)$, which measures how well $\mu_{c_s}$ covers Fourier directions separating the true-signal and recovery classes. 

\section{Experiments}\label{sec: experiments}
In this section we examine both the prompt compatibility factor $\Lambda$ and its practical effect on conditional generative compressed sensing with Stable Diffusion 1.5 (SD15)~\cite{rombach2022high} through the Hugging Face Diffusers library~\cite{von-platen-etal-2022-diffusers}. Throughout, we produce images with SD15 at resolution $512\times 512$ with $20$ DDIM~\cite{song2021denoising} denoising steps to produce signals $\mathbf{f} \in \mathbb{R}^{3 \times 512 \times 512}$ from a latent variable $\mathbf z \in \mathbb R^{4\times 64\times 64}$. We estimate the empirical Christoffel sampling distributions using classifier-free guidance (CFG)~\cite{ho2021classifierfree} of $7.5$ and perform recovery with CFG $1$, corresponding to conditional generation without additional guidance. In the experiments, we collect measurements via the weighted subsampled Fourier operator,
\[
  \mathbf A_{\Omega}
  =
  \frac{1}{\sqrt{m}} \mathbf W_\Omega^{1/2} \mathbf P_{\Omega} \mathbf F,
  \qquad
  \mathbf y
  =
  \mathbf A_{\Omega} \mathbf f^*,
\]
so the prompt dependence first enters through the selected Fourier mask $\Omega$. We do not introduce noise to the measurements. Since the images are multi-channel, we use a block Fourier sampling model: each sampled frequency in $\Omega$ is measured in all three color channels, with $\mathbf A_\Omega$ acting channel-wise. We write $n=512\times 512$ for the size of the flattened image per channel, so the full ambient signal dimension is $3n$. To facilitate reproducibility, we provide the code for computing the empirical Christoffel sampling distributions and running the reconstruction experiments.\footnote{Code available at: \href{https://github.com/alexdelise/ActiveConditionalGCS.git}{\texttt{github.com/alexdelise/ActiveConditionalGCS}}} Experiments were run on five NVIDIA RTX A5000 GPUs with 24\,GB of VRAM each.

\subsection{Empirical investigation into the prompt compatibility factor  \texorpdfstring{$\Lambda$}{Lambda}}
We begin by studying how the prompt compatibility factor changes under sampling-recovery mismatch $c_s \neq c_r$. To isolate this effect, we consider the prompt family
\begin{equation}\label{eq: prompt family}
    c_{\mathrm{uc}} = \texttt{""}, \qquad
    c_{\mathrm{db}} = \texttt{"daytime beach"}, \qquad
    c_{\mathrm{sb}} = \texttt{"sunset beach"}, \qquad
    c_{\mathrm{ca}} = \texttt{"cat"}.
\end{equation}
For each sampling prompt $c_s$, we use Algorithm~\ref{alg: ktilde} in Appendix~\ref{app: ktilde estimation} with $10{,}000$ iterations to estimate the empirical Christoffel function $\widetilde K$, construct the empirical sampling distribution $\widetilde\mu_{c_s}$, and compute $\widetilde{\Lambda}(c_r,c_r,c_s)$ following Equation~\eqref{eq: empirical lambda}. These choices provide an unconditioned baseline, two closely related prompts, and a distinct prompt, allowing us to compare the induced sampling distributions and their compatibility with the recovery classes defined by this prompt family.

Figure~\ref{fig: distribution and heatmap} displays the resulting prompt-conditioned sampling distributions and the corresponding relative values of $\widetilde{\Lambda}$ to quantify how rapidly sample complexity deteriorates as the sensing prompt moves away from the recovery class. Immediately, one can see in Figure~\ref{fig: distribution and heatmap}(a) that the prompt-conditioned sampling distributions concentrate probability mass differently, but share a common bias toward low frequencies. The heatmap in Figure~\ref{fig: distribution and heatmap}(b) makes the interaction between these differences readily apparent as the diagonal entries, corresponding to a matched sampling-recovery regime ($c_s = c_r$) minimize $\widetilde \Lambda$, with larger penalties occurring for distinct pairs in the sampling-recovery mismatch regime, e.g., sampling with $c_s=c_{\mathrm{ca}}$ and recovering with $c_r=c_{\mathrm{uc}}$. 

\begin{figure}
    \centering

    \begin{subfigure}{0.46\linewidth}
        \centering
        \includegraphics[width=\linewidth]{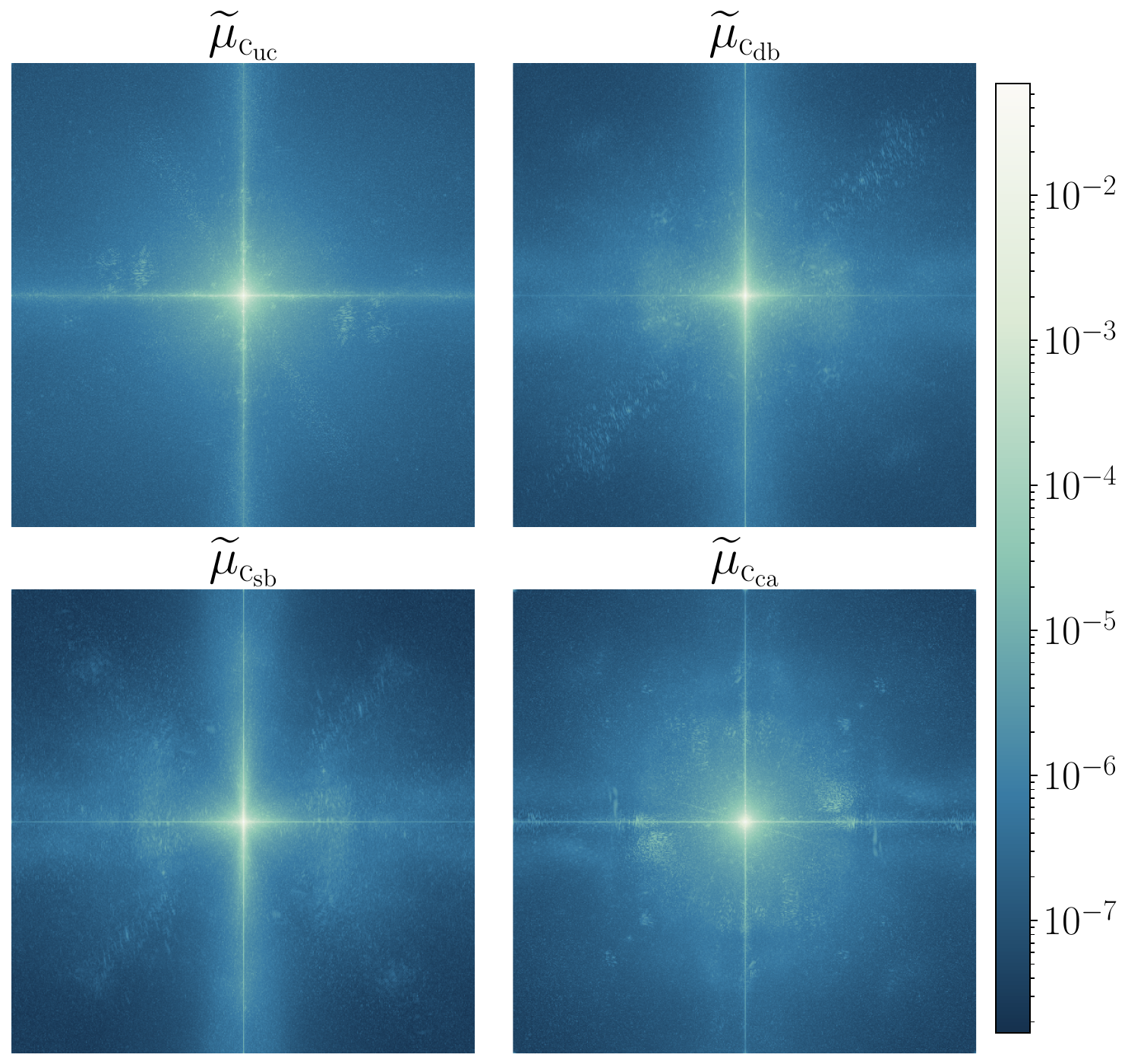}
        \caption{}
        \label{fig: sampling_distributions}
    \end{subfigure}
    \hfill
    \begin{subfigure}{0.53\linewidth}
        \centering
        \includegraphics[width=\linewidth]{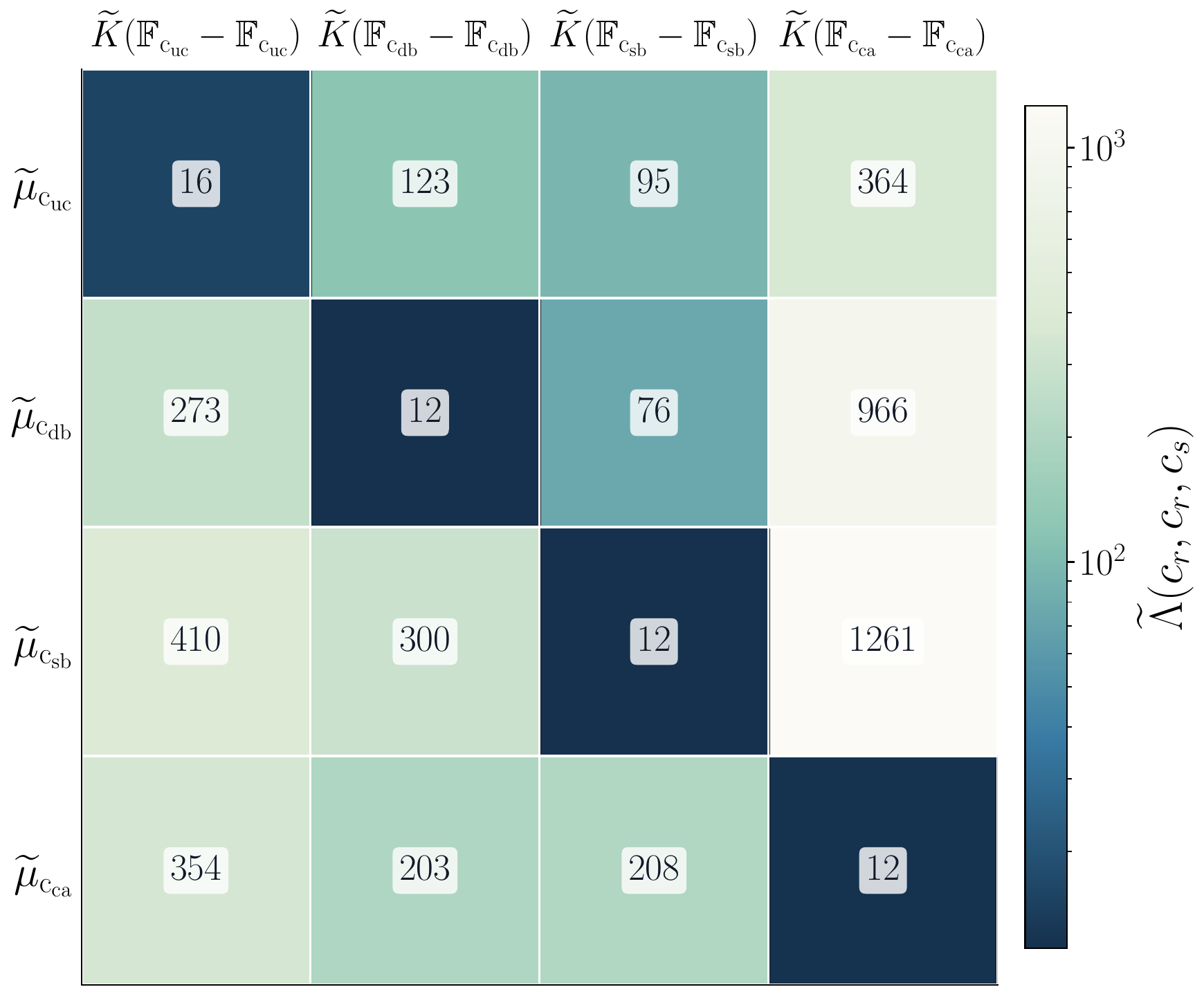}
        \caption{}
        \label{fig: lambda_heatmap}
    \end{subfigure}

    \caption{
    \textbf{(a)} Empirical Christoffel sampling distributions $\widetilde{\mu}_c$ for prompts $c \in \{c_{\mathrm{uc}}, c_{\mathrm{db}}, c_{\mathrm{sb}}, c_{\mathrm{ca}}\}$. Color indicates sampling probability over Fourier frequencies.
    \textbf{(b)} The corresponding empirical prompt compatibility factor $\widetilde{\Lambda}(c_r,c_r,c_s)$ (rows: sampling, columns: recovery). Diagonal entries are minimal, while off-diagonal entries indicate the sampling-recovery mismatch penalty. Both panels use $\zeta=0$. Heatmap values are rounded to the nearest integer.
    }
    \label{fig: distribution and heatmap}
\end{figure}

\subsection{Conditional generative compressed sensing}
We next connect these prompt-conditioned sampling distributions to actual reconstruction performance. For each pair of sampling prompt $c_s$ and recovery prompt $c_r$ in Equation~\eqref{eq: prompt family}, we reconstruct the same target image across sampling ratios $m / n$ using $5$ independently drawn masks per setting. For each mask, we always sample the $0$ frequency, a standard assumption in compressive imaging~\cite{adcock2021compressive}, and then sample the remaining $m-1$ frequencies via $\widetilde \mu_{c_s}$ without replacement.
We use the weighted loss in Equation~\eqref{eq: least squares} and the learned perceptual image patch similarity (LPIPS)~\cite{zhang2018unreasonable} as the primary reconstruction metrics, where lower LPIPS indicates greater perceptual similarity to the target.

Given an image $\mathbf{f}^*$ and measurements $\mathbf y = \mathbf A_{\Omega} \mathbf f^*$, we reconstruct under prompt $c_r$ by seeking $\widehat{\mathbf f}=G(\widehat{\mathbf z},c_r)$, where $\widehat{\mathbf z}$ is obtained from the least-squares problem
\begin{equation}\label{eq: least squares}
  \widehat{\mathbf z}
  \in
  \arg\min_{\mathbf z}
  \frac{1}{2C}
  \left\|
    \mathbf A_{\Omega} G(\mathbf z,c_r) - \mathbf y
  \right\|_2^2,
\end{equation}
and $C=3$ denotes the three color channels of the image.

We initialize the generator from the zero-filled discrete inverse Fourier transform of our subsampled measurement
$
  \mathbf f_{\mathrm{zf}}
  =
  \mathbf A_{\Omega}^{\dagger} \mathbf y,
$
by encoding $\mathbf f_{\mathrm{zf}}$ into latent space and adding Gaussian noise once at the first DDIM timestep to obtain $\mathbf z_{\mathrm{init}}$. We treat the composition of the $20$-step DDIM denoising process and final VAE decode as the conditional generative model $G(\cdot,c_r)$.

For these experiments, we compare the four empirical Christoffel distributions regularized with $\zeta = 1/2$ with uniform Monte Carlo sampling ($\mu_{\mathrm{MCS}}$) and inverse-square variable-density sampling ($\mu_{\mathrm{IS}}$), given by
\[
    \mu_{\mathrm{MCS}}(i)=\frac{1}{n} \qquad
    \text{and} \qquad
    \mu_{\mathrm{IS}}(i)=\frac{(1+r_i^2)^{-1}}{\sum_{j\in D}(1+r_j^2)^{-1}},
\]
where $r_i$ is the Euclidean distance from the zero frequency on the centered Fourier grid. We use sampling ratios $m/n\in\{0.01,0.02,0.03,0.04,0.05\}$ and perform five trials per setting. We optimize Equation~\eqref{eq: least squares} with Adam~\cite{kingma2015adam} for $2{,}000$ iterations, setting the learning rate to $0.1$ through iteration $400$, followed by cosine decay to $10^{-3}$. Appendix~\ref{app: weighted main tables} shows the ground-truth images and corresponding example reconstructions.

\subsubsection{In-range prompt-matched image recovery}\label{subsec: weighted prompt matched}

For the \emph{in-range prompt-matched} task, the target image is generated by SD15 using $c_*=c_{\mathrm{sb}}$. Mean LPIPS across all six sampling distributions, four recovery prompts, and five trials decreases from $0.595$ at $m/n=0.01$ to $0.388$ at $m/n=0.05$. Averaged equally over the five sampling ratios and five trials, sampling with $\mu_{\mathrm{IS}}$ and recovering with $c_r=c_{\mathrm{sb}}$ gives the lowest LPIPS ($0.343$), compared with $0.403$ for sampling with $\widetilde{\mu}_{c_{\mathrm{sb}}}$ under the same recovery prompt. The corresponding weighted losses are $39.2$ and $22.6$, respectively. The lowest mean weighted loss is $21.2$, obtained with $c_s=c_r=c_{\mathrm{db}}$, so the best pair differs between the two metrics. Further averaging over the four recovery prompts, the empirical Christoffel distributions give LPIPS between $0.375$ and $0.413$ and weighted loss between $25.3$ and $30.2$, substantially below $0.857$ and $257.5$ for $\mu_{\mathrm{MCS}}$.

\subsubsection{In-range prompt-mismatched image recovery}\label{subsec: weighted prompt mismatched}
For the \emph{in-range prompt-mismatched} task, the target is generated by SD15 using \mbox{$c_*=\texttt{"sunset over a sandy coast"}$}, which differs from every sampling and recovery prompt in Equation~\eqref{eq: prompt family}. Mean LPIPS across all six sampling distributions, four recovery prompts, and five trials decreases from $0.546$ at $m/n=0.01$ to $0.397$ at $m/n=0.05$. As before, sampling with $\mu_{\mathrm{IS}}$ gives the lowest LPIPS when averaged equally over the four recovery prompts, five sampling ratios, and five trials, at $0.368$. Under the same averaging, the four empirical Christoffel distributions remain competitive, giving mean LPIPS between $0.382$ and $0.413$ and weighted loss between $35.2$ and $40.3$, compared with $0.832$ and $266.3$ for $\mu_{\mathrm{MCS}}$. Across all six sampling distributions, five sampling ratios, and five trials, recovery with $c_r=c_{\mathrm{sb}}$ gives the lowest average LPIPS at $0.454$. This recovery prompt also attains the lowest mean weighted loss ($25.3$) when paired with sampling prompt $c_s=c_{\mathrm{sb}}$, averaging over the five sampling ratios and five trials. Thus, the sunset-beach prompt remains useful even though it differs from the prompt used to generate the target.

\subsubsection{Out-of-range image recovery}\label{subsec: weighted out of range}
For the \emph{out-of-range} task, the sensed image comes from~\cite{magnificSunsetBeachPalmTree}, which lies outside the prompt-conditioned generator ranges. Mean LPIPS across all six sampling distributions, four recovery prompts, and five trials decreases from $0.691$ at $m/n=0.01$ to $0.563$ at $m/n=0.05$. Like the in-range experiments, sampling with $\mu_{\mathrm{IS}}$ gives the lowest LPIPS when averaged equally over the four recovery prompts, five sampling ratios, and five trials, at $0.538$, although it degrades for the higher sampling ratios. Under the same averaging, the empirical Christoffel distributions give mean LPIPS between $0.565$ and $0.595$ and weighted loss between $52.2$ and $58.6$, compared with $0.853$ and $511.5$ for $\mu_{\mathrm{MCS}}$. Recovery with $c_r=c_{\mathrm{sb}}$ gives the lowest mean LPIPS across all six sampling distributions, five sampling ratios, and five trials ($0.615$). Figure~\ref{fig:ood-sampling-comparison} makes the benefit of nonuniform sampling visible: with $c_r=c_{\mathrm{sb}}$, both $\widetilde{\mu}_{c_{\mathrm{sb}}}$ and $\mu_{\mathrm{IS}}$ recover recognizable trees and the sunset, whereas $\mu_{\mathrm{MCS}}$ fails to recover the scene throughout the displayed sampling ratios, even at $m/n=0.05$.

\begin{figure}
    \centering
    \includegraphics[width=\linewidth]{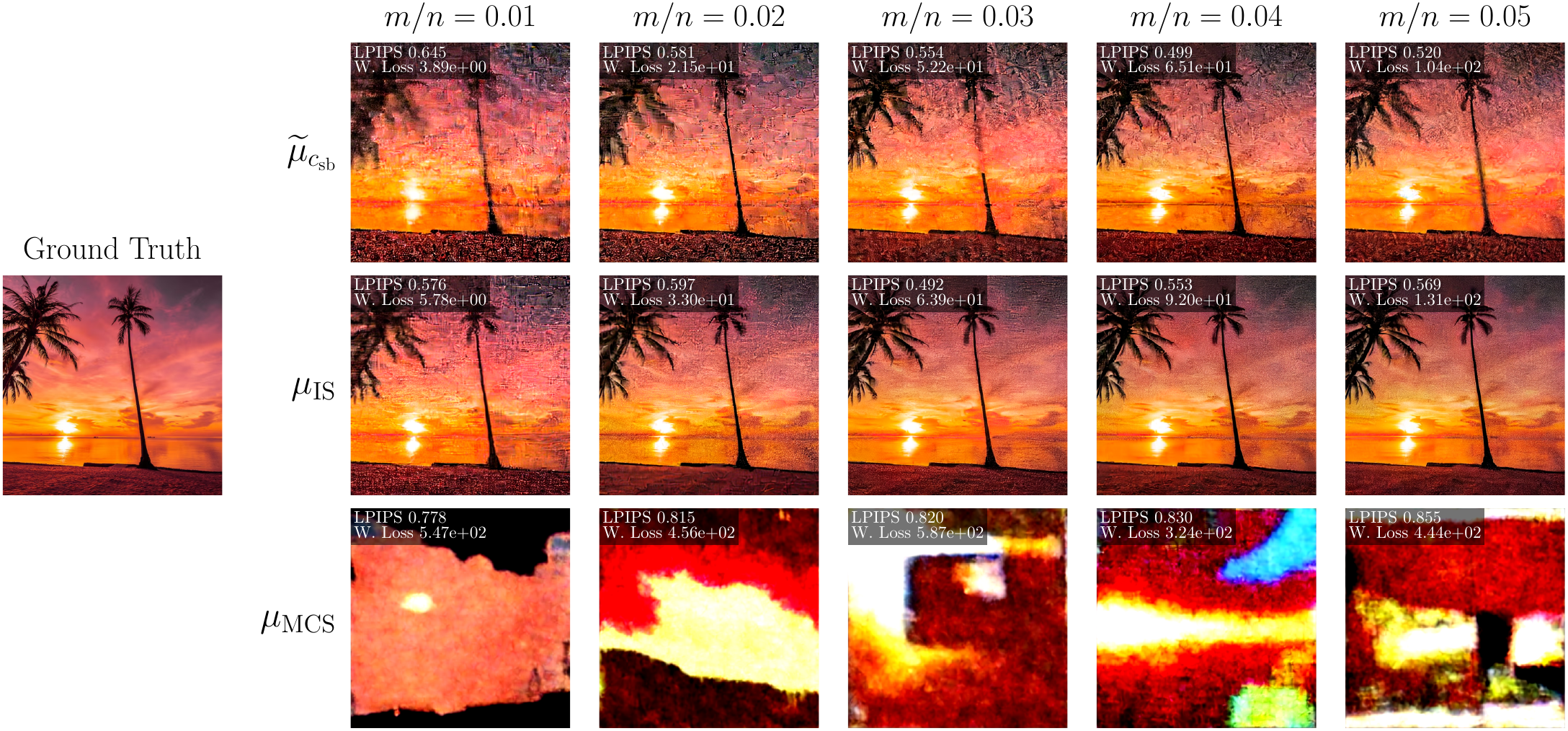}
    \caption{Out-of-range image recovery with fixed recovery prompt $c_r=c_{\mathrm{sb}}$. Rows compare sampling with $\widetilde\mu_{c_{\mathrm{sb}}}$, $\mu_{\mathrm{IS}}$, and $\mu_{\mathrm{MCS}}$ across sampling ratios $m/n\in\{0.01,0.02,0.03,0.04,0.05\}$.}
    \label{fig:ood-sampling-comparison}
\end{figure}

\subsection{Results discussion}\label{sec: results discussion}

From these results, it is clear that sampling with an empirical Christoffel distribution $\widetilde\mu_{c_s}$ consistently improves recovered image quality and adherence to the observed measurements over uniform sampling with $\mu_{\mathrm{MCS}}$. Sampling with $\mu_{\mathrm{IS}}$, however, often gives better image quality in LPIPS compared to its empirical Christoffel distribution counterparts, though the latter generally attain lower weighted loss. The generic low-frequency bias of $\mu_{\mathrm{IS}}$ appears well suited to the imaging experiment here, whereas $\mu_{\mathrm{MCS}}$ provides neither model adaptation nor a low-frequency bias, explaining its poor performance. Altogether, these results suggest measurement fit and image quality need not favor the same sampling strategy, and Appendix~\ref{app: weighted main tables} reports additional metrics and results analysis.

For a fixed recovery class $\mathbb F_{c_r}$, the prompt compatibility factor $\Lambda(c_r,c_r,c_s)$ describes how the sampling distribution $\mu_{c_s}$ affects the sample complexity required for stable recovery. When the target belongs to another class $\mathbb F_{c_*}$, we approximate it using signals in $\mathbb F_{c_r}$, and $\Lambda(c_*,c_r,c_s)$ controls the measurement amplification of this approximation error. The bound also depends on how well the recovery class approximates the target and how accurately the optimization problem is solved. Moreover, it controls $\ell^2$ error rather than perceptual metrics such as LPIPS.  Thus, although the theory supports Christoffel sampling for stable recovery, it does not guarantee that matching all prompts yields the best reconstruction of an individual image. In our experiments, the Christoffel functions and sampling distributions are estimated from finite samples of the range of SD15. Sampling and recovery also use different CFG values, so even identical prompts do not specify the same generator.

Nonetheless, when averaged across all six sampling distributions, five sampling ratios, and five trials, recovery with $c_r=c_{\mathrm{sb}}$, the prompt that most closely describes the target images, gives the lowest mean LPIPS and lowest mean best weighted loss in all three scenarios. This suggests that choosing a recovery class suited to the target content can be useful for measurement fit. The recovery class also depends on CFG, which changes the generator and can substantially affect the latent optimization problem. We use recovery CFG $1$ as the default, but CFG is another parameter that can be tuned. For a fixed prompt, matching sampling and recovery CFG aligns the generator classes used for acquisition and reconstruction, as motivated by the theory, but need not give the best practical recovery. Appendix~\ref{app: cfg ablation} studies these effects, showing that lower recovery CFG often helps in recovery.

\section{Conclusion}\label{sec: conclusion}

We develop an active learning framework for conditional generative compressed sensing in which prompt-conditioned generative models are used for signal recovery and constructing sampling distributions for subsampled Fourier measurements. We introduce a prompt compatibility factor $\Lambda$ that quantifies alignment between the sampling, recovery, and ground-truth prompt-conditioned signal classes, and demonstrate how it governs sample complexity under ReLU and Lipschitz architecture assumptions and bounds reconstruction error. Experiments with Stable Diffusion show that prompt conditioning can substantially alter the induced Christoffel sampling distributions and can influence reconstruction performance under subsampled Fourier measurements.

Our work suggests several directions for future research. Since the prompt compatibility factor $\Lambda$ is defined through a max-divergence quantity, future work could investigate KL- or R\'{e}nyi-divergence versions of $\Lambda$ and the resulting sample complexity for (nonuniform) recovery. Also, our theory applies to idealized generator classes and is not yet well suited to modern conditional generative models. Although a model such as SD15 can be viewed as inducing a Lipschitz map on compact latent sets after fixing the prompt embedding, sampler, denoising schedule, and CFG scale, this remains an idealized abstraction of practical recovery. The effective Lipschitz constant may be large and sensitive to modeling choices such as CFG scale. Finally, our experiments raise several questions about how to choose the sampling distribution, recovery class, and optimization procedure together:

\begin{enumerate}
    \item Our theory supports Christoffel sampling with $\widetilde\mu_{c_s}$ and a weighted objective for near-optimal sample complexity, yet this combination can produce noisier images than sampling with $\mu_{\mathrm{IS}}$. Appendix~\ref{app: unweighted results} shows that unweighted recovery can improve image quality within $500$ iterations, suggesting that the weighted objective is harder to optimize. However, unweighted recovery can introduce hallucinated artifacts.
    \item Different CFG values induce different generator classes. We begin exploring their effect on recovery in Appendix~\ref{app: cfg ablation} with results suggesting that lower CFG can help. Future work should also vary sampling CFG and study how to choose the two jointly.
    \item 
    How should sampling and recovery prompts be chosen when the target class is unknown? Appendix~\ref{app: weighted main tables} shows that prompts need not describe the target perfectly to achieve good recovery, motivating more robust prompt-selection strategies.
    \item  We fix the sampling regularization parameter $\zeta=1/2$ arbitrarily. Its effects on stability, sample complexity, and reconstruction quality, including whether regularization is necessary or if it can be optimized, remain open.
    \item  Future work should investigate alternative generative architectures, including bespoke recovery models, and applications to other inverse problems.
\end{enumerate}

Conditional generative models introduce many interacting choices in both what we measure and how we reconstruct. Our framework is a first step toward understanding them together, and opens many avenues to explore for improving data acquisition and signal recovery.

\begin{ack}
The authors thank Ben Adcock and the anonymous reviewers for their helpful comments and suggestions. The authors declare no external funding or competing interests.
\end{ack}

\bibliographystyle{unsrtnat}
\bibliography{references}


\appendix
\newpage
\section{Christoffel sampling for machine learning (CS4ML)}\label{app:cs4ml}

In this appendix, we fully outline the abstract CS4ML framework in~\cite{adcock2023cs4ml} we harness for Christoffel sampling and bridge their objects to our weighted subsampled conditional Fourier model.

\subsection{Objects, definitions, and assumptions}
We begin with an underlying probability space $(\Theta,\mathcal{F},\mathbb{P})$, a separable Hilbert space $\mathbb{X}$, and a normed vector subspace $\mathbb{X}_0 \subseteq \mathbb{X}$, termed the \emph{object space}. Our goal is to recover an object $\mathbf{f}^* \in \mathbb{X}_0$ from limited measurements. We consider the measurement domain $(D,\mathcal{A},\rho)$, the semi-inner-product measurement space $\mathbb{Y}$, and the sampling operators $L_i \in \mathcal B(\mathbb{X}_0,\mathbb{Y})$ for $i\in D$, where $\mathcal B(\mathbb{X}_0,\mathbb{Y})$ denotes the space of bounded linear operators from $\mathbb{X}_0$ to $\mathbb{Y}$. Thus, for each $i \in D$, the quantity $L_i(\mathbf{f}^*) \in \mathbb{Y}$ represents the measurement of $\mathbf{f}^*$ at parameter value $i$ (where we employ the classical active learning assumption that we may query $L_i$ for any $i \in D$). We also fix an \emph{approximation space} $\mathbb{F} \subseteq \mathbb{X}_0$ within which we perform reconstruction, together with a sampling measure $\mu$ on $(D,\mathcal{A})$ that is absolutely continuous with respect to $\rho$, with Radon--Nikodym derivative $\nu = \mathrm{d} \mu / \mathrm{d} \rho > 0$.

To ensure that the sampling operators are sufficiently expressive to preserve the $\mathbb{X}$-norm, we make the following assumption:

\begin{assumption}[Nondegeneracy of the sampling operators]
    There exist constants $0<\alpha\leq \beta<\infty$ such that
    \[
      \alpha\left\|\mathbf{f}\right\|_{\mathbb{X}}^2
      \leq
      \int_{D}
      \left\|L_i(\mathbf f)\right\|_{\mathbb{Y}}^2
      \, \mathrm{d}\rho(i)
      \leq
      \beta\left\|\mathbf{f}\right\|_{\mathbb{X}}^2,
      \qquad
      \mathbf{f}\in\mathbb{X}_0.
    \]
\end{assumption}
Later, we will define the discrete analogue of nondegeneracy in which we hope to approximately preserve norms.

For a fixed space $\mathbb{F}$, the generalized Christoffel function and its related quantity are defined as
\[
  K(\mathbb{F})(i)
  =
  \sup_{\mathbf{f}\in \mathbb{F}\setminus\{0\}}
  \frac{\left\|L_i(\mathbf{f})\right\|_{\mathbb{Y}}^2}
       {\left\|\mathbf{f}\right\|_{\mathbb{X}}^2},
  \qquad i\in D,
  \qquad
  \kappa(\mathbb{F})
  =
  \int_{D} K(\mathbb{F})(i)\, \mathrm{d} \rho(i).
\]
The corresponding optimal sampling density is
\[
  \nu^*(i)
  =
  \frac{K(\mathbb{F})(i)}{\kappa(\mathbb{F})},
  \qquad 
  i \in D,
\]
 given by~\cite[Lemma~4.6]{adcock2023cs4ml}. Equivalently, the optimal sampling measure $\mu^*$ is the measure with density $\nu^*$ with respect to $\rho$. Finally, the sampling seminorm for a density $\nu$ is given by
\[
  \snorm{\mathbf{f}}{\nu}
  =
  \left(
    \operatorname*{ess\,sup}_{i \sim \rho}
    \frac{\left\|L_i(\mathbf{f})\right\|_{\mathbb{Y}}^2}{\nu(i)}
  \right)^{1/2}.
\]

\subsection{Specialization to weighted subsampled Fourier measurements}

For our paper, the aforementioned objects above specialize to the following for subsampled Fourier measurements:
\[
  \mathbb{X}=\mathbb{X}_0=\mathbb{R}^n,
  \qquad
  \left\|\cdot\right\|_{\mathbb{X}}=\left\|\cdot\right\|_2,
  \qquad
  D=\{1,\dots,n\},
\]
\[
  \rho=\text{counting measure on }D,
  \qquad
  \mathbb{Y}=\mathbb{C}
  \text{ with }
  \left\|z\right\|_{\mathbb{Y}}=|z|,
  \qquad
  L_i(\mathbf f)=\mathbf P_i\mathbf F\mathbf f.
\]
Since $\mathbf{F}$ is the unitary discrete Fourier transform, Parseval's theorem gives
\[
  \sum_{i=1}^n |L_i(\mathbf f)|^2
  =
  \sum_{i=1}^n \left|\mathbf P_i\mathbf F\mathbf f\right|^2
  =
  \left\|\mathbf f\right\|_2^2,
  \qquad
  \mathbf f\in\mathbb{R}^n,
\]
so the nondegeneracy constants are $\alpha=\beta=1$.

\subsection{Prompt-conditioned discrete analogues}

We now return to the prompt-conditioned recovery setting from the main text by fixing a prompt $c$ and applying the CS4ML machinery to the secant class $\mathbb{F}_{c}-\mathbb{F}_{c}$. Under this specialization, the optimal density $\nu^*$ is proportional to the Christoffel function of the recovery secant class. For the corresponding regularized sampling distribution generated with prompt $c_s$, we have that 
\[
  \mu_{c_s}(i)
  =
  (1-\zeta)
  \frac{
    K(\mathbb{F}_{c_s}-\mathbb{F}_{c_s})(i)
  }{
    \kappa(\mathbb{F}_{c_s}-\mathbb{F}_{c_s})
  }
  +
  \frac{\zeta}{n},
  \qquad
  i\in D.
\]
Its mismatch with the recovery class generated by the prompt $c_r$ is quantified by
\[
  \operatorname*{ess\,sup}_{i\sim\rho}
  \frac{K(\mathbb{F}_{c_r}-\mathbb{F}_{c_r})(i)}{\mu_{c_s}(i)}
  =
  \max_{i \in D}
  \frac{K(\mathbb{F}_{c_r}-\mathbb{F}_{c_r})(i)}{\mu_{c_s}(i)}
  =
  \Lambda(c_r,c_r,c_s).
\]
The regularization guarantees that $\mu_{c_s}(i)\geq\zeta/n>0$ for every $i\in D$. When $c_s=c_r=c$, we have that 
$
  \Lambda(c,c,c)
  \leq
  \kappa(\mathbb{F}_{c}-\mathbb{F}_{c}) / (1-\zeta),
$
retaining the optimal Christoffel complexity quantity up to the constant factor $(1-\zeta)^{-1}$. The optimal sampling seminorm then reduces to
\[
  \snorm{\mathbf f}{\mu_{c_s}}
  =
  \left(
    \max_{i\in D}
    \frac{\left|\mathbf P_i\mathbf F\mathbf f\right|^2}
         {\mu_{c_s}(i)}
  \right)^{1/2}.
\]
For sampled Fourier indices $\Omega = \{I_1, \dots, I_m\}$, we have
\[
  \frac{1}{m}
  \sum_{j=1}^m
  \frac{\left|L_{I_j}(\mathbf f)\right|^2}
       {\mu_{c_s}(I_j)}
  =
  \frac{1}{m}
  \sum_{j=1}^m
  \frac{\left|\mathbf P_{I_j}\mathbf F\mathbf f\right|^2}
       {\mu_{c_s}(I_j)}
  =
  \left\|\mathbf A_\Omega\mathbf f\right\|_2^2.
\]

We also write
\[
  \mathbb B_{c_r}
  =
  \left\{
    \frac{\mathbf h}{\left\|\mathbf h\right\|_2}
    \colon
    \mathbf h\in
    (\mathbb F_{c_r}-\mathbb F_{c_r})\setminus\{0\}
  \right\}
\]
for the normalized recovery secant set.

Under this specialization, empirical nondegeneracy is given by the following, which we repeat for completeness:

\begin{definition}[Empirical nondegeneracy]
Let $0<\tau<1$. We say that $\mathbf A_\Omega$ is empirically nondegenerate on $\mathbb F_{c_r}-\mathbb F_{c_r}$ with distortion $\tau$ if
\[
  (1-\tau)\left\|\mathbf h\right\|_2^2
  \leq
  \left\|\mathbf A_\Omega\mathbf h\right\|_2^2
  \leq
  (1+\tau)\left\|\mathbf h\right\|_2^2,
  \qquad
  \mathbf h\in\mathbb F_{c_r}-\mathbb F_{c_r}.
\]
\end{definition}

Finally, the theorem below is the discrete subsampled Fourier measurement sample complexity version of~\cite[Theorem~4.2]{adcock2023cs4ml}, which will be particularly useful when proving Theorem~\ref{thm:relu-srec}.

\begin{theorem}[Empirical nondegeneracy and S-REC under Christoffel sampling]
\label{thm:christoffel-srec}
Fix $0<\tau,\delta<1$, and let
$\mathcal N\subseteq\mathbb B_{c_r}$ be an $\eta$-net for $\mathbb B_{c_r}$
in the sampling seminorm $\snorm{\cdot}{\mu_{c_s}}$ with
$0<\eta\leq\tau/8$. If
\[
  m
  \gtrsim
  \tau^{-2}
  \Lambda(c_r,c_r,c_s)
  \log\left(\frac{2|\mathcal N|}{\delta}\right),
\]
then with probability at least $1-\delta$ over the draw of $\Omega$,
$\mathbf A_\Omega$ is empirically nondegenerate on
$\mathbb F_{c_r}-\mathbb F_{c_r}$ with distortion $\tau$ and hence
satisfies
$
  \mathrm{S\text{-}REC}
  (
    \mathbb F_{c_r},
    \sqrt{1-\tau},
    0
  ).
$
\end{theorem}

We defer the proof to Appendix~\ref{app:theory-proofs}.

\newpage\section{Numerical approximation to the generalized Christoffel function}\label{app: further experiment details}
Here, we describe the numerical approximation of the generalized Christoffel function and assess the convergence and stability of the resulting sampling distributions.

\subsection{Numerical approximation}\label{app: ktilde estimation}
In practice, the Christoffel function $K(\mathbb F_{c_s}-\mathbb F_{c_s})$ appearing in the sampling distribution in Equation~\eqref{eq: christoffel sampling law} is not computable exactly when the prompt-conditioned class $\mathbb F_{c_s}$ is specified as the range of a deep generator $G(\cdot,c_s)$. Indeed, evaluating $K(\mathbb F_{c_s}-\mathbb F_{c_s})(i)$ requires taking a supremum over all nonzero secants of $\mathbb F_{c_s}$, equivalently over pairs of latent vectors passed through $G(\cdot,c_s)$, which amounts to a global optimization problem that is intractable at scale. We therefore use Algorithm~\ref{alg: ktilde}, adapted from~\cite[Algorithm~1]{adcock2023cs4ml}, to construct an empirical estimator $\widetilde K$ of $K(\mathbb F_{c_s}-\mathbb F_{c_s})$. This estimator is straightforward to compute in practice and is intended to approximate, and in the large-sample limit recover, the true Christoffel function. For our experiments, the latent distribution $\pi_z$ in Algorithm~\ref{alg: ktilde} is the normal distribution, i.e., $\mathbf{z} \sim \mathcal N (\mathbf 0, \mathbf I_k)$, where $\mathbf I_k \in \mathbb{R}^{k \times k}$ is the identity.

With $\widetilde K$ in hand, we define the empirical analogue of the Christoffel sampling distribution in Equation~\eqref{eq: christoffel sampling law} by
\begin{equation}\label{eq: empirical sampling law}
  \widetilde \mu_{c_s}(i)
  =
    (1-\zeta)
    \frac{\widetilde K(\mathbb F_{c_s}-\mathbb F_{c_s})(i)}{\widetilde \kappa(\mathbb F_{c_s}-\mathbb F_{c_s})}
    +
    \frac{\zeta}{n},
  \quad i \in D,
\end{equation}
where $\widetilde \kappa(\mathbb F_{c_s}-\mathbb F_{c_s}) = \sum_{i=1}^{n} \widetilde K(\mathbb F_{c_s}-\mathbb F_{c_s})(i)$.
This yields a computable approximation to the ideal sampling distribution $\mu_{c_s}$ and we then draw the measurement indices in $\Omega$ from $\widetilde \mu_{c_s}$. Finally, we can compute an empirical approximation to the prompt compatibility factor using the formula
\begin{equation}\label{eq: empirical lambda}
  \widetilde\Lambda(c_1,c_2,c_3)
  =
  \max_{i \in D}
  \frac{\widetilde K(\mathbb{F}_{c_1}-\mathbb{F}_{c_2})(i)}{\widetilde \mu_{c_3}(i)}.
\end{equation}

\begin{algorithm}
  \caption{Compute $\widetilde{K}(\mathbb F_{c_1}-\mathbb F_{c_2})$}
  \label{alg: ktilde}
  \begin{algorithmic}[1]
    \REQUIRE Generator $G$, prompts $c_1,c_2$, latent distribution $\pi_z$, number of iterations $M$
    \STATE Initialize $\widetilde K(\mathbb F_{c_1}-\mathbb F_{c_2})(i)\gets 0$ for $i=1,\dots,n$
    \FOR{$\ell=1$ to $M$}
    \STATE Sample latent vectors $\mathbf z_1,\mathbf z_2 \sim \pi_z$
    \STATE Generate $\mathbf f_1 \gets G(\mathbf z_1,c_1)$ and $\mathbf f_2 \gets G(\mathbf z_2,c_2)$
    \STATE Set $\mathbf h \gets \mathbf f_1-\mathbf f_2$ and $\epsilon \gets \left\|\mathbf h\right\|_2$
    \IF{$\epsilon = 0$}
    \STATE \textbf{continue}
    \ENDIF
    \STATE Compute the discrete Fourier transform $\mathbf y \gets \mathbf F\mathbf h$
    \FOR{$i=1$ to $n$}
    \STATE $a_i \gets |y_i|^2 / \epsilon^2$
    \STATE $\widetilde K(\mathbb F_{c_1}-\mathbb F_{c_2})(i) \gets \max\{\widetilde K(\mathbb F_{c_1}-\mathbb F_{c_2})(i), a_i\}$
    \ENDFOR
    \ENDFOR
  \end{algorithmic}
\end{algorithm}

\subsection{Convergence and stability}\label{app: ktilde convergence analysis}
We now assess the convergence and stability of Algorithm~\ref{alg: ktilde} for each of the four sampling prompts in Equation~\eqref{eq: prompt family}. For each prompt $c$, we first run Algorithm~\ref{alg: ktilde} for $10{,}000$ iterations to generate a reference estimate $\widetilde K^*(\mathbb F_c-\mathbb F_c)$. We then generate five additional estimates with $10{,}000$ iterations using independent, nonoverlapping latent-seed blocks. At iteration counts $M=10,20,\ldots,10{,}000$, let $\widetilde K_{M}^{(j)}(\mathbb F_c-\mathbb F_c)$ denote the estimate obtained from the first $M$ iterations of trial $j$. From this estimate, we construct $\widetilde\mu_{c,M}^{(j)}$ according to Equation~\eqref{eq: empirical sampling law}, using $\zeta=1/2$ as in the numerical experiments.

To probe the small-probability tail that controls the prompt compatibility factor, we fix the reference estimate $\widetilde K^*(\mathbb F_c-\mathbb F_c)$ in the numerator and vary only the iteration-$M$ sampling distribution in the denominator. That is, for a fixed prompt $c$ and trial $j$, we compute
\begin{equation}
  \widetilde\Lambda(M)
  =
  \max_{i\in D}
  \frac{
    \widetilde K^*(\mathbb F_c-\mathbb F_c)(i)
  }{
    \widetilde\mu_{c,M}^{(j)}(i)
  }.
  \label{eq: ktilde-convergence-lambda}
\end{equation}
Because the numerator is fixed and takes a maximum over coordinates, stabilization of $\widetilde\Lambda(M)$ measures whether the estimated regularized distribution has stabilized in the tail relevant to the prompt compatibility factor. To complement this quantity, we also compare the relative $\ell^2$ error between the empirical Christoffel estimates,
\begin{equation}
    \operatorname{Rel}_2(M) = 
    \frac{
        \left\|\widetilde K_{M}^{(j)}(\mathbb F_c-\mathbb F_c)-\widetilde K^*(\mathbb F_c-\mathbb F_c)\right\|_2
    }{
        \left\|\widetilde K^*(\mathbb F_c-\mathbb F_c)\right\|_2
    }.
    \label{eq: ktilde-relative-l2-error}
\end{equation}

As indicated by Figure~\ref{fig: ktilde convergence}, both diagnostics decrease as the budget increases, with the largest changes occurring early and then plateauing at larger iterations $M$. Averaged across the four prompts, $\widetilde\Lambda(M)$ decreases from $412.0$ at $M=10$ to $139.2$ at $M=10{,}000$, while $\operatorname{Rel}_2(M)$ decreases from $0.646$ to $0.168$. Neither diagnostic, however, establishes convergence to the corresponding exact Christoffel function or induced prompt compatibility factor. This is due to the fact that all generated estimates, including the reference $\widetilde K^*$, are finite-sample approximations of the range of the (highly nonlinear) conditioned generator. Further stability from the uniform mixture with $\zeta=1/2$ guarantees $\widetilde\mu_{c,M}^{(j)}(i)\geq 1/(2n)$, diminishing instabilities in the prompt compatibility factor that might otherwise appear from low assigned probabilities in the empirical sampling distributions.

\begin{figure}
    \centering

    \includegraphics[width=\linewidth]{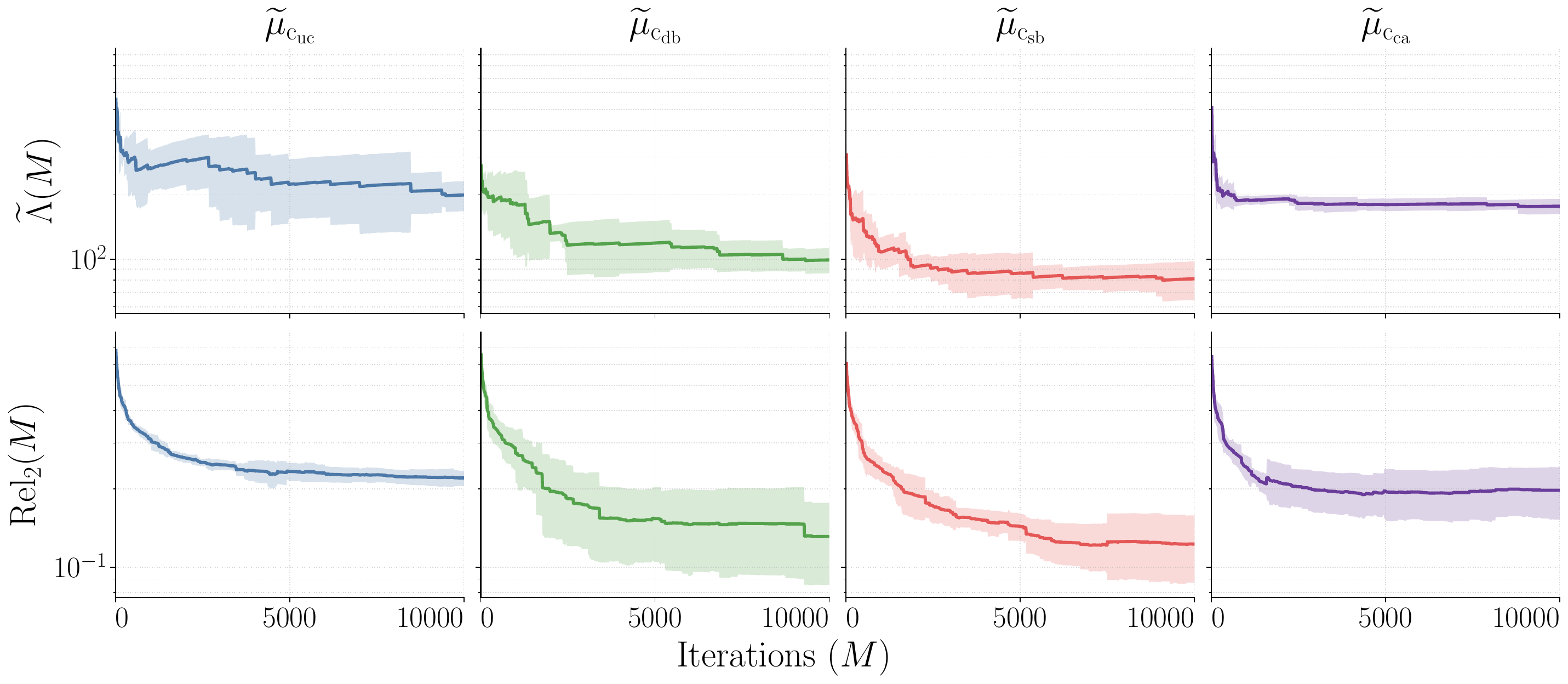}

    \caption{
    Convergence of the empirical Christoffel estimates for prompts
    $c \in \{c_{\mathrm{uc}}, c_{\mathrm{db}}, c_{\mathrm{sb}}, c_{\mathrm{ca}}\}$.
    The top row shows $\widetilde{\Lambda}(M)$ in Equation~\eqref{eq: ktilde-convergence-lambda}, using $\zeta=1/2$, and the bottom row shows the relative $\ell^2$ error $\operatorname{Rel}_2(M)$ in Equation~\eqref{eq: ktilde-relative-l2-error}.
    Shaded regions denote $95\%$ confidence intervals.
    }
    \label{fig: ktilde convergence}
    \label{fig: ktilde convergence lambda}
    \label{fig: ktilde convergence relative l2}
\end{figure}

\clearpage\section{Further experimental results}\label{app: weighted main tables}
This appendix accompanies Section~\ref{sec: experiments}, providing additional reconstruction results for the three recovery experiments and examining the empirical cross-class Christoffel function.

\subsection{Additional results and analysis for conditional generative compressed sensing}\label{app: reconstruction metrics}
Figure~\ref{fig:weighted-main-forest} summarizes all $1{,}800$ image reconstructions from Section~\ref{sec: experiments}. For each fixed sampling-distribution and recovery-prompt pair, we average each metric equally over the five sampling ratios and five trials. We compare LPIPS and the best weighted loss attained when solving the objective function in Equation~\eqref{eq: least squares}, alongside PSNR and SSIM. Table~\ref{tab:weighted-main-best-pairs} provides the corresponding numerical values displayed in the figure for the best observed mean, indicated with a star, for each metric in the experiments. We observe that uniform sampling with $\mu_{\mathrm{MCS}}$ consistently results in poor reconstruction quality across all image-quality metrics, suggesting that the selected frequencies are less effective for recovery. Among the four empirical Christoffel sampling distributions $\widetilde{\mu}_c$ and the inverse-square variable-density distribution $\mu_{\mathrm{IS}}$, we observe a distinction between recovered image quality and fit to the observed measurements through the weighted objective. Specifically, sampling with $\mu_{\mathrm{IS}}$ gives the best LPIPS, PSNR, or SSIM in eight of the nine corresponding rows in Table~\ref{tab:weighted-main-best-pairs}, while sampling with an empirical Christoffel distribution gives the best weighted loss in all three experiments. The comparable performance of the four empirical Christoffel distributions may reflect regularization with $\zeta=1/2$, which mixes each Christoffel distribution with the same uniform component.

Whereas Table~\ref{tab:weighted-main-best-pairs} reports the best individual sampling-distribution and recovery-prompt pairs, Table~\ref{tab:weighted-main-recovery-prompts} holds each recovery prompt fixed and averages equally over all six sampling distributions, five sampling ratios, and five trials. Recovery with $c_r=c_{\mathrm{sb}}$ gives the lowest observed mean LPIPS and best weighted loss in all three experiments, suggesting that a recovery prompt that describes the image can be useful for both perceptual similarity and conformity to the observed measurements in image recovery. For PSNR and SSIM, recovery with $c_r=c_{\mathrm{db}}$ gives the highest mean in four of the six comparisons. In the remaining two, both from the prompt-mismatched experiment, $c_r=c_{\mathrm{sb}}$ gives the highest PSNR, while $c_r=c_{\mathrm{uc}}$ and $c_r=c_{\mathrm{sb}}$ give similar mean SSIM. Thus, although the preferred recovery prompt depends on the metric, the two beach-related prompts, which (partly) describe the images being recovered, account for nearly all the best observed averages.

\begin{table}[h]
\centering
\small
\setlength{\tabcolsep}{1.5pt}
\caption{Best sampling-distribution and recovery-prompt pairs under weighted recovery, averaged over five sampling ratios and five trials. Values are means $\pm$ 95\% confidence intervals.}
\label{tab:weighted-main-best-pairs}
\begin{tabular*}{\linewidth}{@{\extracolsep{\fill}}llccr@{}}
\toprule
Experiment & Metric & \begin{tabular}[c]{@{}c@{}}Sampling\\Distribution\end{tabular} & \begin{tabular}[c]{@{}c@{}}Recovery\\Prompt\end{tabular} & Mean $\pm$ 95\% CI \\
\midrule
In-range prompt-matched & LPIPS $\downarrow$ & $\mu_{\mathrm{IS}}$ & $c_r=c_{\mathrm{sb}}$ & 0.3425 $\pm$ 0.0460 \\
 & Best Weighted Loss $\downarrow$ & $\widetilde\mu_{c_{\mathrm{db}}}$ & $c_r=c_{\mathrm{db}}$ & 21.2 $\pm$ 1.3 \\
 & PSNR (dB) $\uparrow$ & $\widetilde\mu_{c_{\mathrm{uc}}}$ & $c_r=c_{\mathrm{db}}$ & 25.17 $\pm$ 1.54 \\
 & SSIM $\uparrow$ & $\mu_{\mathrm{IS}}$ & $c_r=c_{\mathrm{uc}}$ & 0.580 $\pm$ 0.036 \\
\midrule
In-range prompt-mismatched & LPIPS $\downarrow$ & $\mu_{\mathrm{IS}}$ & $c_r=c_{\mathrm{sb}}$ & 0.3559 $\pm$ 0.0173 \\
 & Best Weighted Loss $\downarrow$ & $\widetilde\mu_{c_{\mathrm{sb}}}$ & $c_r=c_{\mathrm{sb}}$ & 25.3 $\pm$ 4.8 \\
 & PSNR (dB) $\uparrow$ & $\mu_{\mathrm{IS}}$ & $c_r=c_{\mathrm{sb}}$ & 24.07 $\pm$ 1.98 \\
 & SSIM $\uparrow$ & $\mu_{\mathrm{IS}}$ & $c_r=c_{\mathrm{sb}}$ & 0.568 $\pm$ 0.042 \\
\midrule
Out-of-range & LPIPS $\downarrow$ & $\mu_{\mathrm{IS}}$ & $c_r=c_{\mathrm{ca}}$ & 0.5259 $\pm$ 0.0267 \\
 & Best Weighted Loss $\downarrow$ & $\widetilde\mu_{c_{\mathrm{uc}}}$ & $c_r=c_{\mathrm{db}}$ & 49.5 $\pm$ 3.1 \\
 & PSNR (dB) $\uparrow$ & $\mu_{\mathrm{IS}}$ & $c_r=c_{\mathrm{db}}$ & 21.02 $\pm$ 0.37 \\
 & SSIM $\uparrow$ & $\mu_{\mathrm{IS}}$ & $c_r=c_{\mathrm{ca}}$ & 0.415 $\pm$ 0.018 \\
\bottomrule
\end{tabular*}
\end{table}

\vspace{-10pt}
\begin{table}
\centering
\small
\setlength{\tabcolsep}{3pt}
\caption{Reconstruction metrics for each recovery prompt under weighted recovery, averaged over six sampling distributions, five sampling ratios, and five trials. Values are means $\pm$ 95\% confidence intervals. Grey shading marks the best unrounded mean per row.}
\label{tab:weighted-main-recovery-prompts}
\begin{tabular*}{\linewidth}{@{\extracolsep{\fill}}lcccc@{}}
\toprule
Experiment & $c_r=c_{\mathrm{uc}}$ & $c_r=c_{\mathrm{db}}$ & $c_r=c_{\mathrm{sb}}$ & $c_r=c_{\mathrm{ca}}$ \\
\midrule
\multicolumn{5}{c}{LPIPS $\downarrow$} \\
\midrule
In-range prompt-matched & $0.476\pm 0.026$ & $0.464\pm 0.016$ & $\bestmean{0.460\pm 0.005}$ & $0.478\pm 0.020$ \\
In-range prompt-mismatched & $0.462\pm 0.022$ & $0.460\pm 0.016$ & $\bestmean{0.454\pm 0.012}$ & $0.468\pm 0.024$ \\
Out-of-range & $0.621\pm 0.006$ & $0.617\pm 0.003$ & $\bestmean{0.615\pm 0.007}$ & $0.624\pm 0.016$ \\
\midrule
\multicolumn{5}{c}{Best Weighted Loss $\downarrow$} \\
\midrule
In-range prompt-matched & $73.0\pm 11.2$ & $65.8\pm 11.0$ & $\bestmean{64.8\pm 5.3}$ & $69.8\pm 9.7$ \\
In-range prompt-mismatched & $86.2\pm 9.2$ & $71.3\pm 4.4$ & $\bestmean{67.5\pm 8.2}$ & $78.9\pm 7.8$ \\
Out-of-range & $134.5\pm 11.6$ & $138.4\pm 20.3$ & $\bestmean{130.9\pm 7.0}$ & $139.4\pm 14.9$ \\
\midrule
\multicolumn{5}{c}{PSNR (dB) $\uparrow$} \\
\midrule
In-range prompt-matched & $21.05\pm 1.08$ & $\bestmean{21.73\pm 0.68}$ & $21.49\pm 0.72$ & $21.21\pm 1.21$ \\
In-range prompt-mismatched & $20.00\pm 0.80$ & $20.61\pm 1.10$ & $\bestmean{21.11\pm 0.75}$ & $20.09\pm 1.14$ \\
Out-of-range & $17.08\pm 0.26$ & $\bestmean{17.22\pm 0.16}$ & $17.18\pm 0.12$ & $16.95\pm 0.47$ \\
\midrule
\multicolumn{5}{c}{SSIM $\uparrow$} \\
\midrule
In-range prompt-matched & $0.471\pm 0.032$ & $\bestmean{0.479\pm 0.012}$ & $0.476\pm 0.016$ & $0.474\pm 0.019$ \\
In-range prompt-mismatched & $\bestmean{0.478\pm 0.020}$ & $0.473\pm 0.024$ & $0.478\pm 0.026$ & $0.475\pm 0.021$ \\
Out-of-range & $0.311\pm 0.008$ & $\bestmean{0.316\pm 0.006}$ & $0.314\pm 0.007$ & $0.312\pm 0.016$ \\
\bottomrule
\end{tabular*}
\end{table}

\begin{landscape}
\begin{figure}[p]
    \centering
    \includegraphics[width=0.935\linewidth]{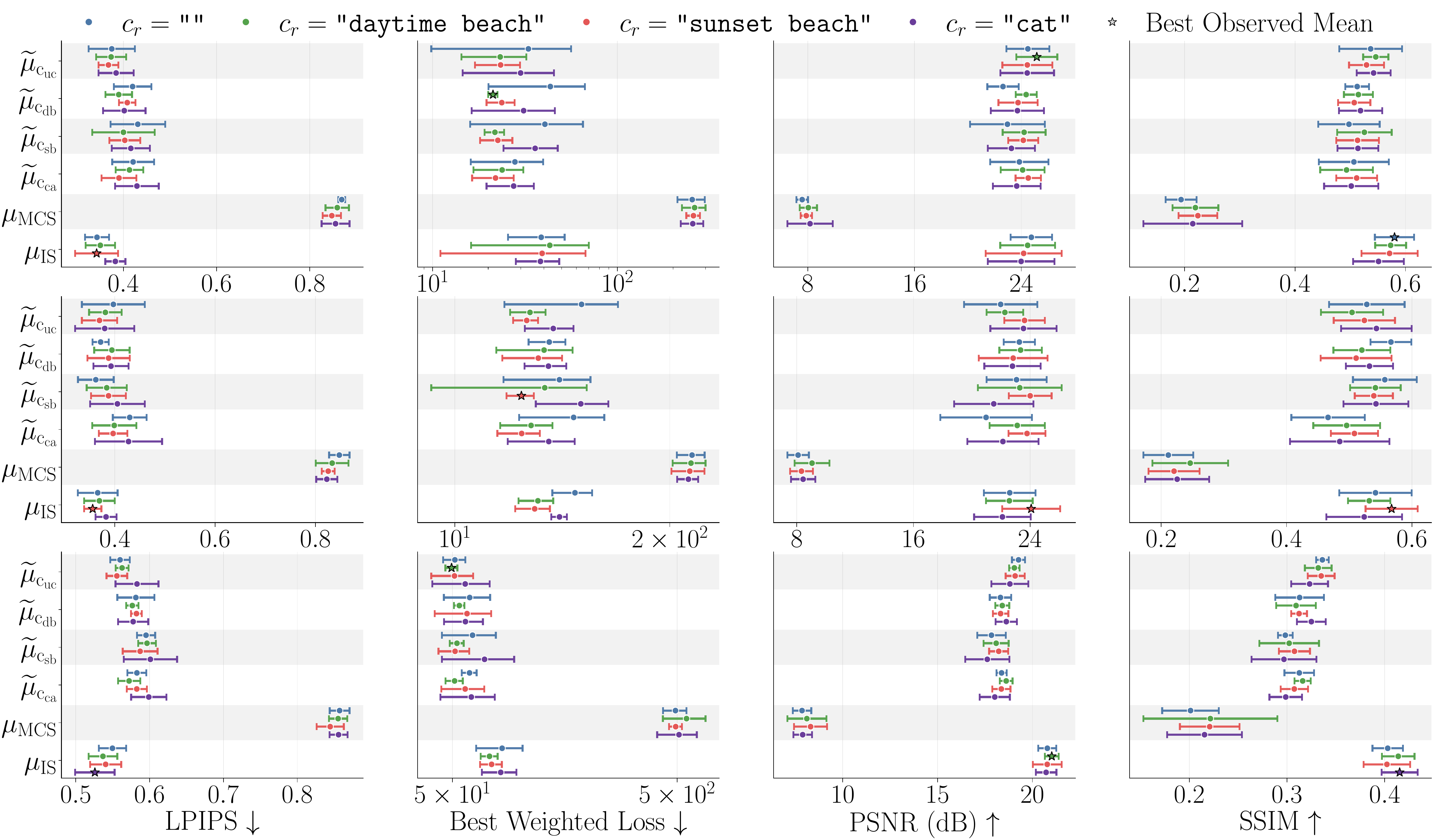}
    \caption{Aggregate reconstruction metrics under weighted recovery for the in-range prompt-matched (top), in-range prompt-mismatched (middle), and out-of-range (bottom) experiments. For each fixed sampling-distribution and recovery-prompt pair, points show means over the five sampling ratios and five trials. Bars denote $95\%$ confidence intervals.}
    \label{fig:weighted-main-forest}
\end{figure}
\end{landscape}

\subsection{Empirical cross-class Christoffel functions}\label{app: cross class ktilde}
Theorem~\ref{thm:prompt-mismatched-recovery} involves the cross-class compatibility factor $\Lambda(c_*,c_r,c_s)$ when the target and recovery classes differ. We study $\Lambda(c_1,c_{\mathrm{sb}},c_s)$ by fixing $c_2=c_{\mathrm{sb}}$ to connect this analysis to the in-range prompt-matched target and letting $c_1$ and $c_s$ vary over the recovery and sampling prompts in Equation~\eqref{eq: prompt family}. Figure~\ref{fig: cross class ktilde}(a) presents these cross-class estimates as distributions solely to visualize their Fourier structure, where $\widetilde\mu_{c_1,c_2}(i)=\widetilde K(\mathbb F_{c_1}-\mathbb F_{c_2})(i)/\widetilde\kappa(\mathbb F_{c_1}-\mathbb F_{c_2})$ is the unregularized empirical cross-class distribution. Figure~\ref{fig: cross class ktilde}(b) then shows a heatmap for the cross-class prompt compatibility factor $\widetilde\Lambda(c_1,c_{\mathrm{sb}},c_s)$. Among the entries shown, the fully matched $\widetilde\Lambda(c_{\mathrm{sb}},c_{\mathrm{sb}},c_{\mathrm{sb}})$ attains the smallest observed value of $12$. The cross-class values, however, can be either larger or smaller than the self-class values in Figure~\ref{fig: distribution and heatmap}(b). For example, under sampling with $\widetilde\mu_{c_{\mathrm{ca}}}$, we obtain $\widetilde\Lambda(c_{\mathrm{uc}},c_{\mathrm{sb}},c_{\mathrm{ca}})=318$ and $\widetilde\Lambda(c_{\mathrm{ca}},c_{\mathrm{sb}},c_{\mathrm{ca}})=156$, compared with the self-difference value $\widetilde\Lambda(c_{\mathrm{sb}},c_{\mathrm{sb}},c_{\mathrm{ca}})=208$. To explain this phenomenon, we note that, in general, neither of the difference sets $\mathbb{F}_{c_1}-\mathbb{F}_{c_2}$ and $\mathbb{F}_{c_2}-\mathbb{F}_{c_2}$ necessarily contains the other. Consequently, their Christoffel functions need not be pointwise ordered, and neither are their corresponding compatibility factors under a fixed sampling distribution. Moreover, as discussed in Appendix~\ref{app: ktilde estimation}, the empirical Christoffel functions are finite-sample approximations, which may further influence the observed values.

\begin{figure}
    \centering
    \begin{subfigure}{0.46\linewidth}
        \centering
        \includegraphics[width=\linewidth]{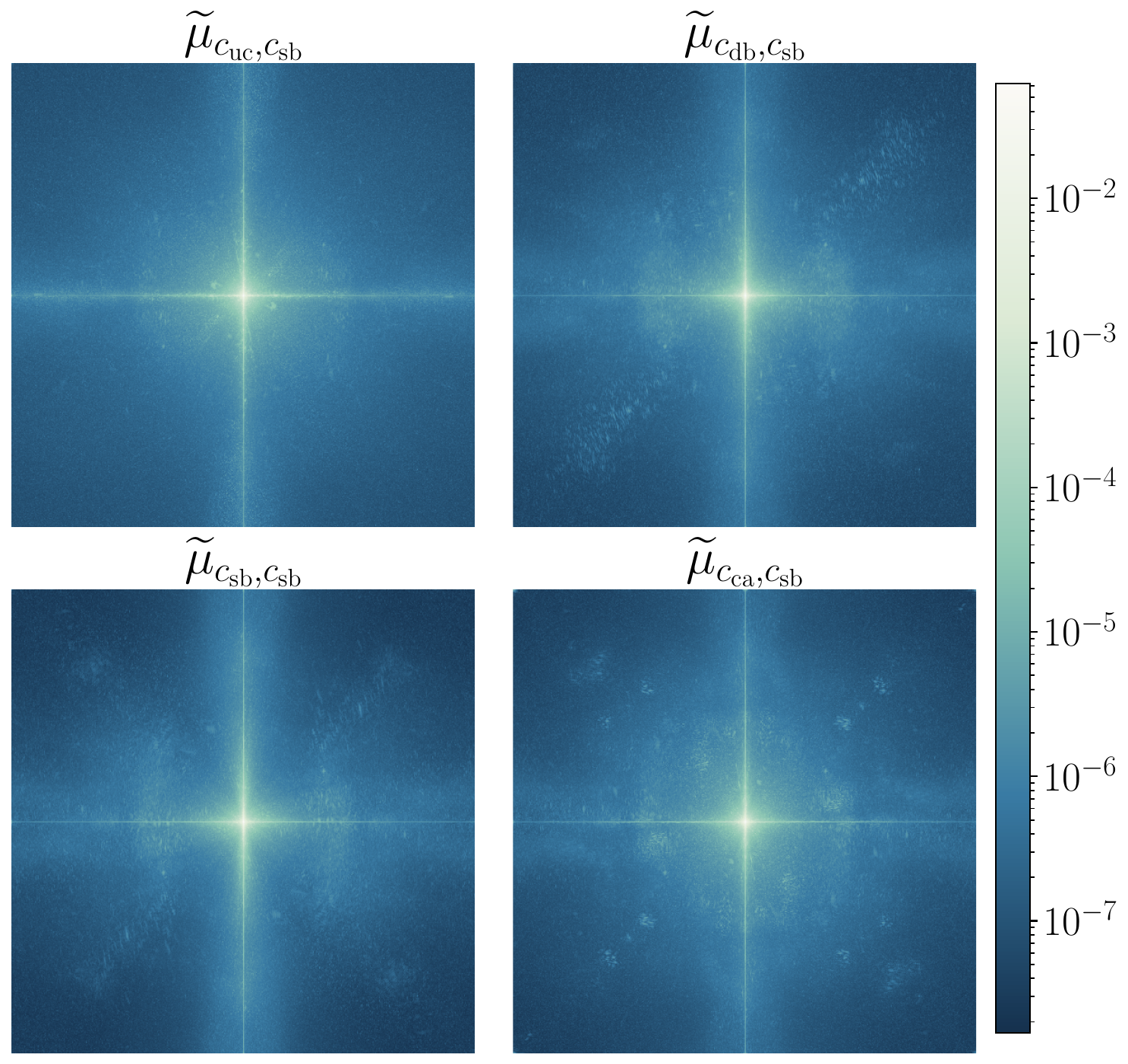}
        \caption{}
        \label{fig: cross class distributions}
    \end{subfigure}
    \hfill
    \begin{subfigure}{0.53\linewidth}
        \centering
        \includegraphics[width=\linewidth]{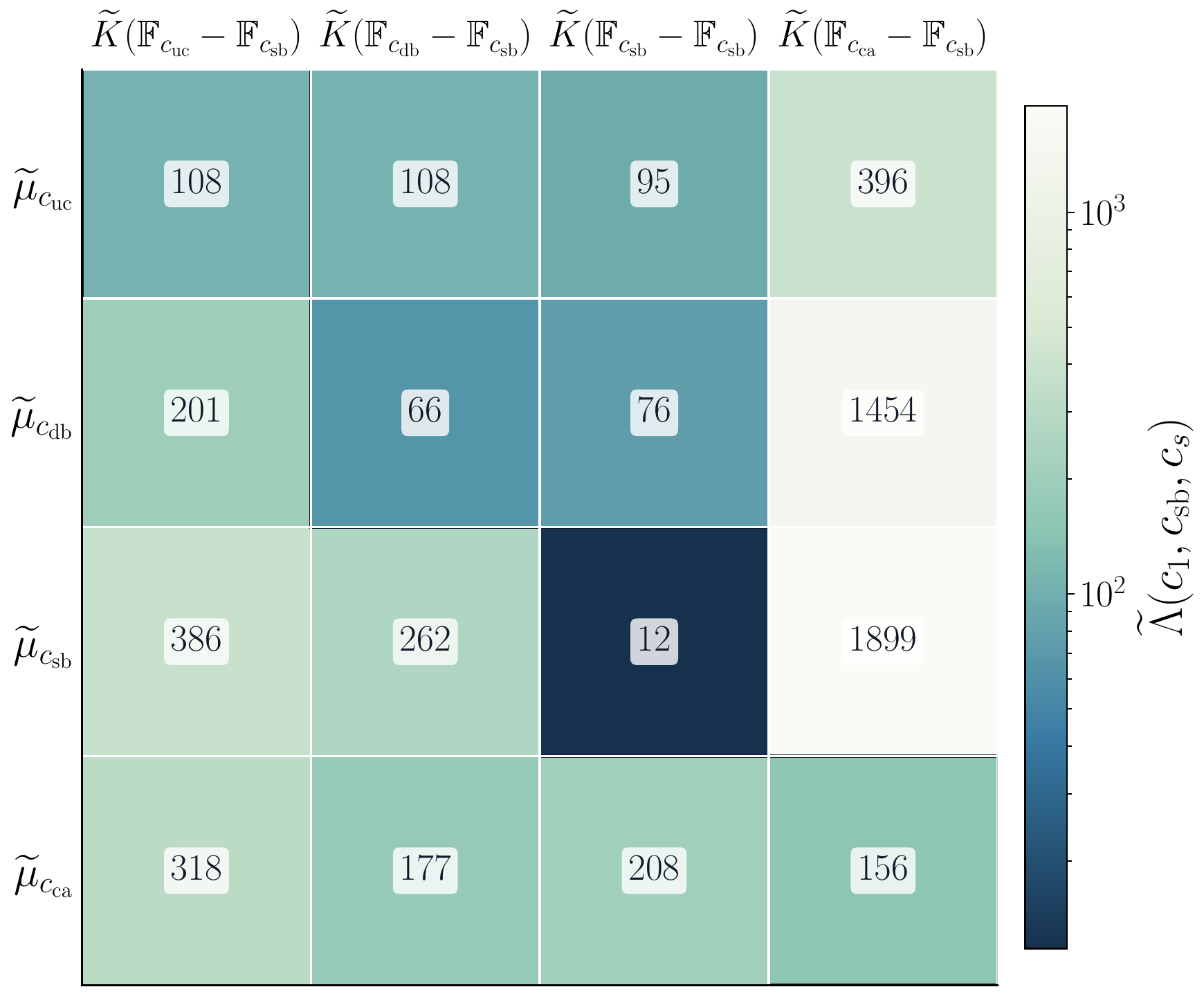}
        \caption{}
        \label{fig: cross class compatibility}
    \end{subfigure}
    \caption{\textbf{(a)} Empirical cross-class sampling distributions $\widetilde\mu_{c_1,c_{\mathrm{sb}}}$ for prompts $c_1\in\{c_{\mathrm{uc}},c_{\mathrm{db}},c_{\mathrm{sb}},c_{\mathrm{ca}}\}$. Color indicates sampling probability over Fourier frequencies. \textbf{(b)} The corresponding empirical prompt compatibility factor $\widetilde\Lambda(c_1,c_{\mathrm{sb}},c_s)$ for the four self-class sampling distributions (rows: varying prompt $c_s$, columns: varying prompt $c_1$). Both panels use $\zeta=0$. Heatmap values are rounded to the nearest integer.}
    \label{fig: cross class ktilde}
\end{figure}

\subsection{Recovered image comparisons}\label{app: weighted recovered img comp}

Here we display selected recovered images for the in-range prompt-matched, in-range prompt-mismatched, and out-of-range image recovery tasks in Section~\ref{sec: experiments}. All displayed reconstructions use a sampling ratio of $m/n=0.03$. For each sampling distribution and recovery prompt, we select the reconstruction with the lowest LPIPS among the five trials, breaking ties by higher PSNR. Rows correspond to sampling with $\widetilde\mu_{c_{\mathrm{sb}}}$, $\mu_{\mathrm{MCS}}$, and $\mu_{\mathrm{IS}}$. Each row shows the ground-truth image followed by reconstructions with $c_r\in\{c_{\mathrm{uc}},c_{\mathrm{db}},c_{\mathrm{sb}},c_{\mathrm{ca}}\}$. Each image is selected independently, so recovery prompts in the same row may be reconstructed from different measurements.

\begin{figure}
    \centering
    \includegraphics[width=\linewidth]{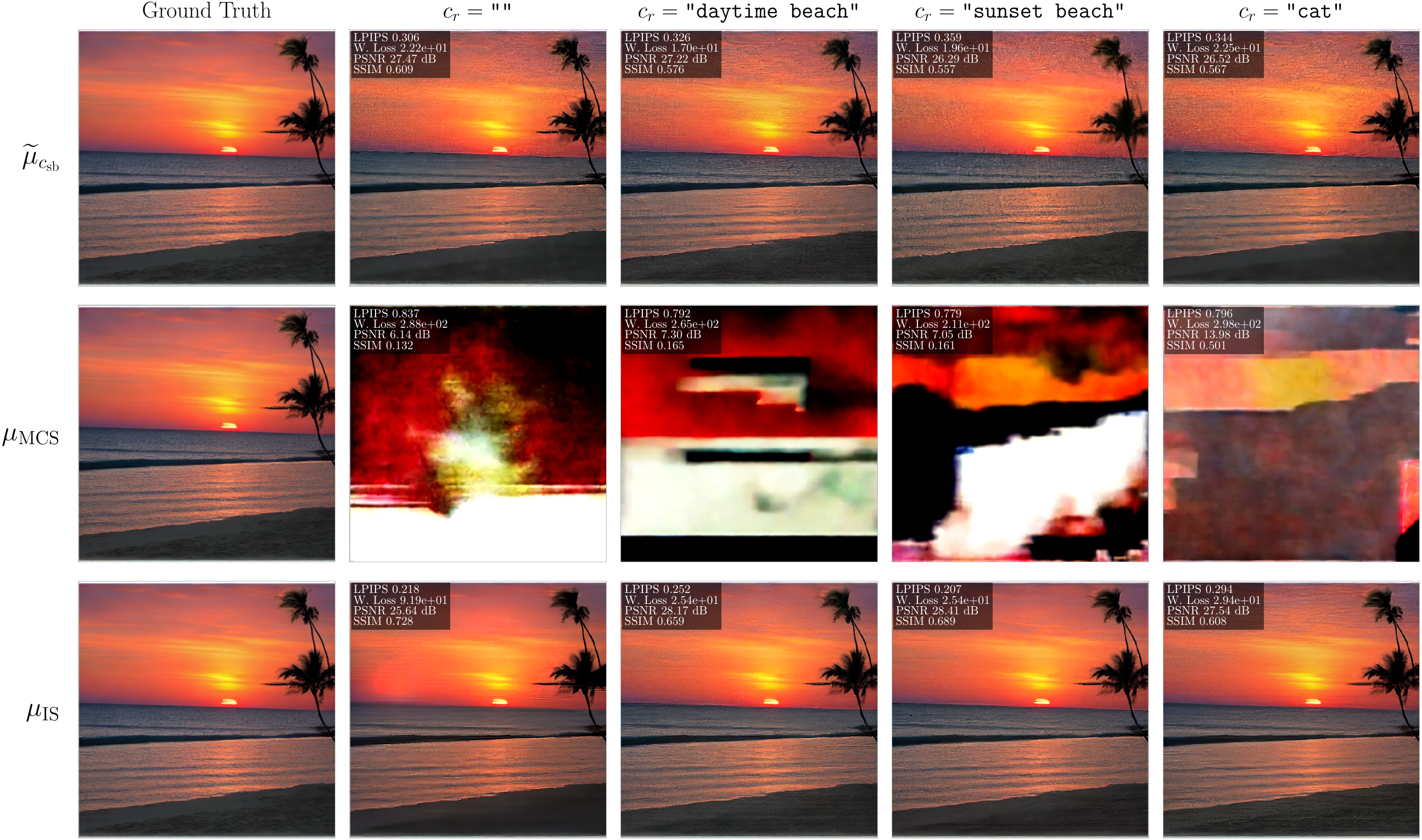}
    \caption{Best recovered images by LPIPS for the in-range prompt-matched experiment (Section~\ref{subsec: weighted prompt matched}) at $m/n=0.03$. The ground truth was generated with $c_*=\texttt{"sunset beach"}$.}
    \label{fig:weighted-prompt-matched-recovered-images}
\end{figure}

\begin{figure}
    \centering
    \includegraphics[width=\linewidth]{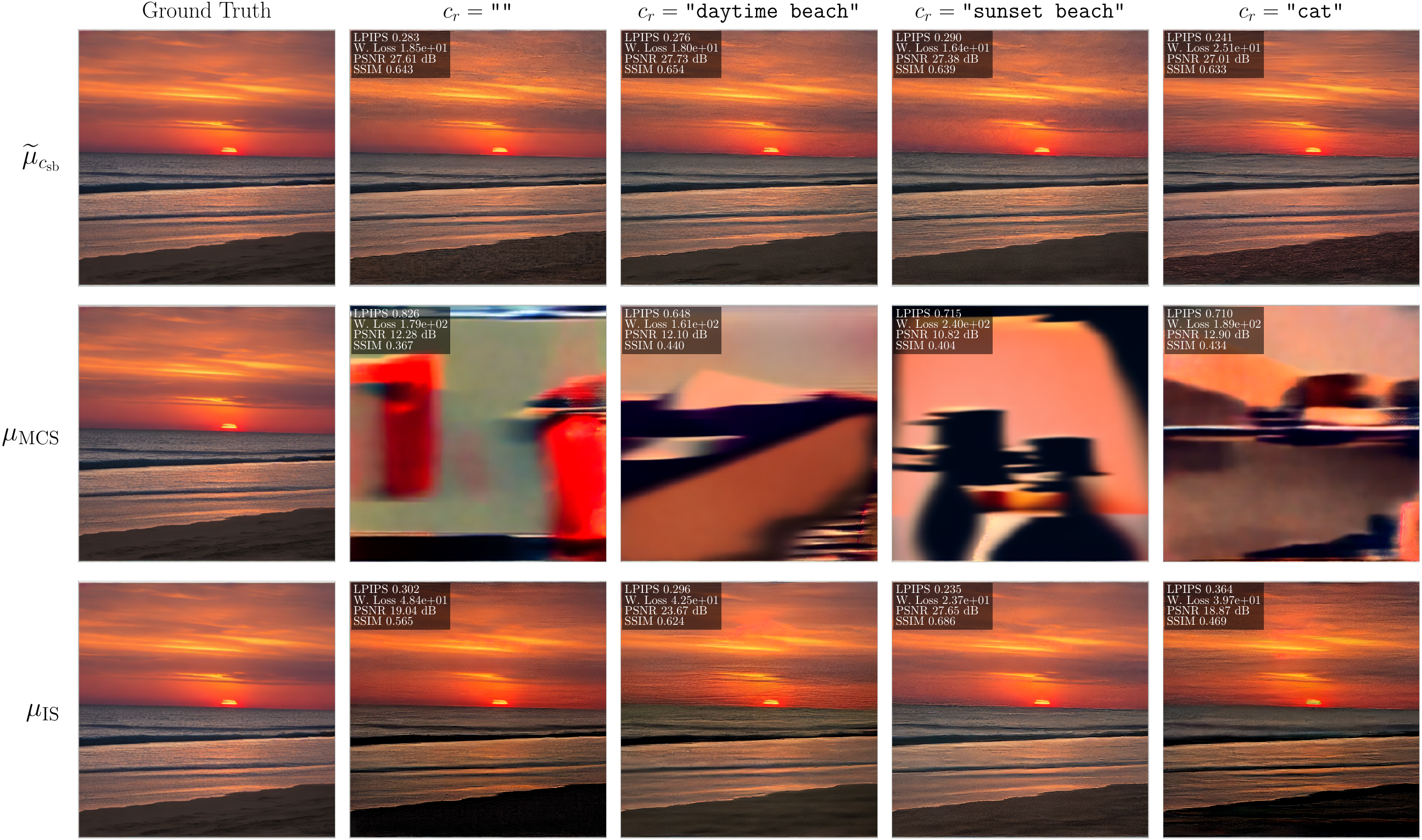}
    \caption{Best recovered images by LPIPS for the in-range prompt-mismatched experiment (Section~\ref{subsec: weighted prompt mismatched}) at $m/n=0.03$. The ground truth was generated with $c_*=\texttt{"sunset over a sandy coast"}$.}
    \label{fig:weighted-prompt-mismatched-recovered-images}
\end{figure}

\begin{figure}
    \centering
    \includegraphics[width=\linewidth]{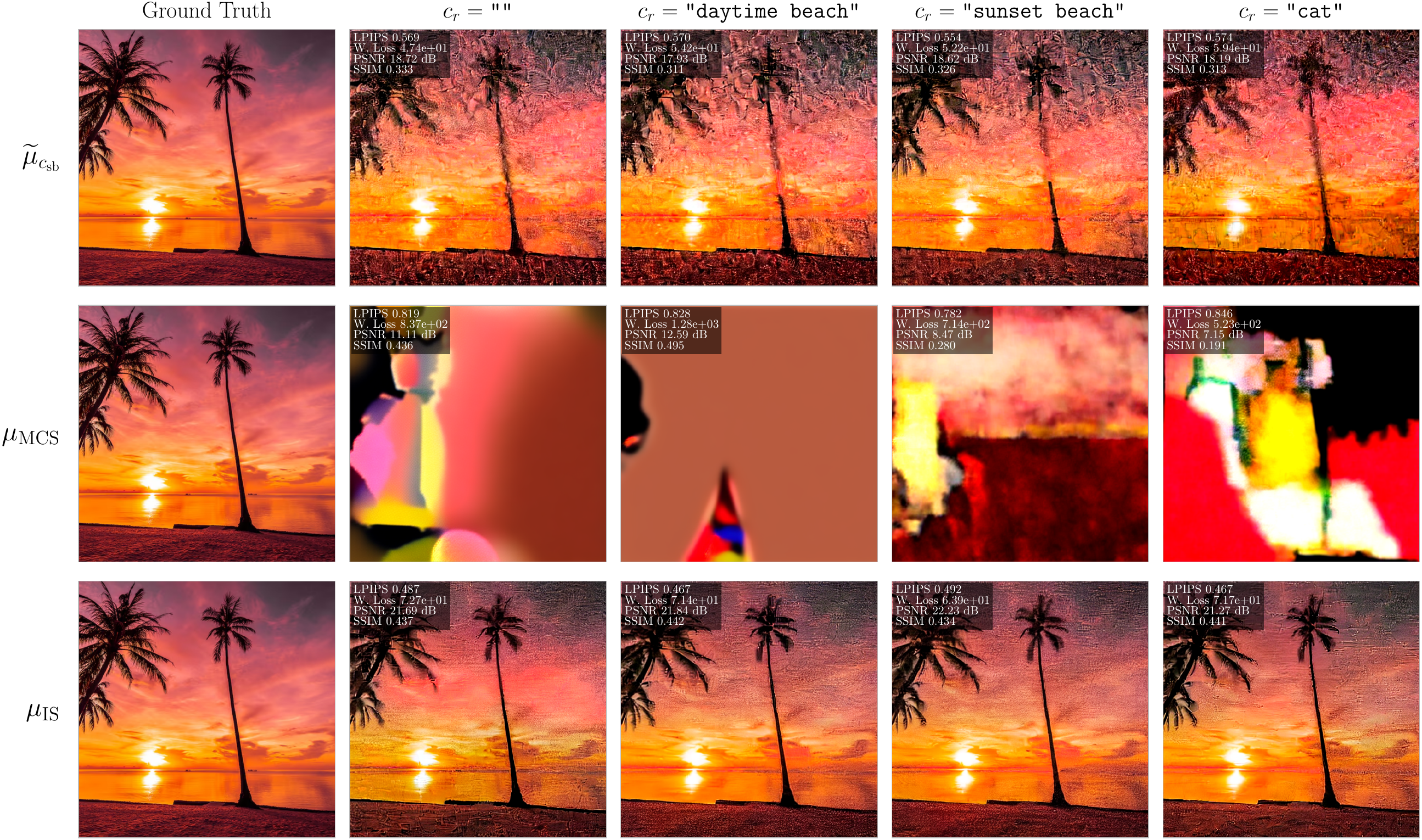}
    \caption{Best recovered images by LPIPS for the out-of-range experiment (Section~\ref{subsec: weighted out of range}) at $m/n=0.03$. The ground truth is the sunset photograph from~\cite{magnificSunsetBeachPalmTree}.}
    \label{fig:weighted-out-of-range-recovered-images}
\end{figure}

\newpage\section{CFG ablation}\label{app: cfg ablation}
For diffusion generative image models, CFG controls how strongly the generated samples adhere to conditioning, in our case a text prompt. Higher CFG values produce images with better semantic alignment to the text prompt, whereas lower CFG values allow the model more diversity and natural variability~\cite{ho2021classifierfree}. For SD15, a CFG of $1$ corresponds to the baseline conditional generation process without additional CFG.

In the context of our framework, changing CFG effectively changes the prompt-conditioned generator itself: a low CFG value allows the reconstruction class to remain relatively broad, while a high CFG value restricts the generator more strongly toward images that match the prompt. Thus, CFG can affect both the sampling side and the recovery side of the problem. On the sampling side, it can change the geometry of the prompt-conditioned image class and thereby alter the induced empirical Christoffel sampling distribution. On the recovery side, it can change how much flexibility the optimizer has to match the observed Fourier measurements. We therefore include a CFG ablation to study the tradeoff between stronger prompt adherence and reconstruction fidelity. 

\subsection{Sampling distributions}
We begin by exploring how changing CFG affects the prompt-conditioned empirical Christoffel sampling distributions. Figure~\ref{fig: ktilde ablation} displays the sampling distributions obtained by estimating the empirical Christoffel function using Algorithm~\ref{alg: ktilde} with $10{,}000$ iterations under varying CFG, and then applying Equation~\eqref{eq: empirical sampling law} to obtain $\widetilde \mu_c$. We observe that as CFG increases, the sampling distributions become less diffuse and place more mass on a smaller set of Fourier frequencies. This is consistent with the role of CFG: stronger guidance restricts the generator more tightly to images aligned with the prompt, reducing variability within the prompt-conditioned class. Consequently, the secants used to estimate the Christoffel function vary along fewer dominant directions, making the most informative Fourier frequencies easier to identify. Conversely, at lower CFG, the generator produces a broader and more variable collection of images, so the resulting secant energy is spread across a wider range of frequencies and the induced sampling distribution is more diffuse.

\begin{figure}
    \centering
    \includegraphics[width=1.0\linewidth]{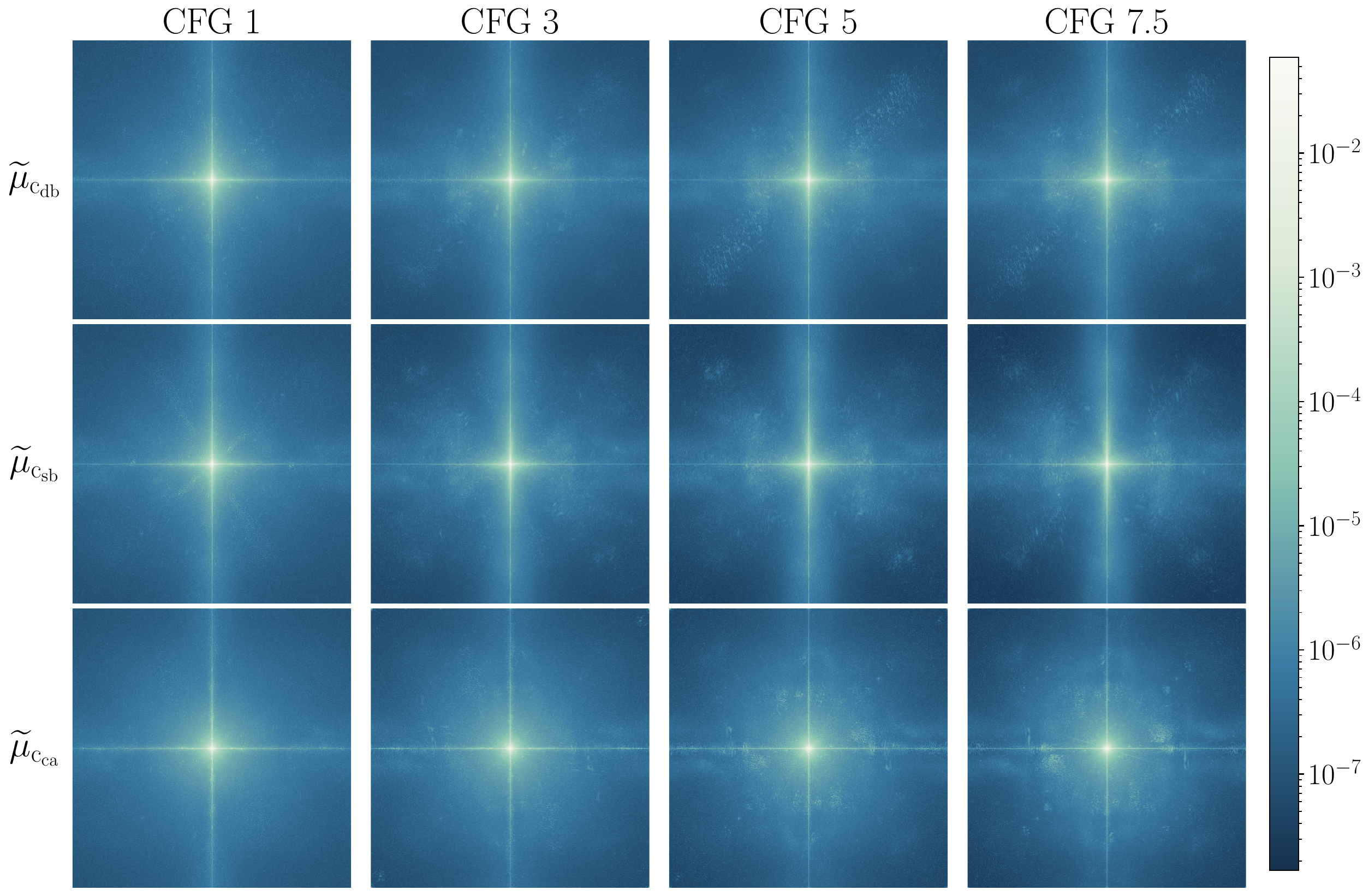}
    \caption{Empirical Christoffel sampling distributions $\widetilde{\mu}_c$ with $\zeta=0$ for prompts $c \in \{ c_{\mathrm{db}}, c_{\mathrm{sb}}, c_{\mathrm{ca}}\}$ and $\mathrm{CFG}\in\{1,3,5,7.5\}$. Color indicates sampling probability over Fourier frequencies.}
    \label{fig: ktilde ablation}
\end{figure}

Figure~\ref{fig:cfg-compatibility} shows that changing CFG affects the empirical prompt compatibility factor $\widetilde \Lambda$ even when the sampling and recovery prompts remain the same. Matching the sampling and recovery CFG gives the smallest value in each column. For the prompt $c_{\mathrm{db}}$, the off-diagonal values are relatively consistent, ranging from $87$ to $147$, whereas they reach $1078$ for $c_{\mathrm{sb}}$ and $983$ for $c_{\mathrm{ca}}$. The effect also depends on which CFG is used for sampling and which is used for recovery, and these differences suggest that CFG should be considered alongside the prompt when choosing the sampling distribution and recovery class.

\subsection{Image recovery}
We next perform the ablation on recovery CFG for the in-range prompt-matched, in-range prompt-mismatched, and out-of-range experiments under the reconstruction prompt $c_r=c_{\mathrm{sb}}=\texttt{"sunset beach"}$. We consider recovery CFG values of $1$, $3$, $5$, and $7.5$. Following the same procedure as in Section~\ref{sec: experiments}, we optimize Equation~\eqref{eq: least squares} for each recovery CFG and sampling distribution pair. We use five reconstruction trials for each CFG value, reusing the CFG $1$ data from the main experiments as a reference, with sampling ratios $m/n\in\{0.01,0.02,0.03,0.04,0.05\}$. All measurements are sampled from the original four empirical prompt-conditioned sampling distributions constructed with sampling CFG $7.5$ in the main experiments, regularized with $\zeta=1/2$. We do not consider the baseline distributions $\mu_{\mathrm{IS}}$ or $\mu_{\mathrm{MCS}}$ in the CFG ablation.

Figure~\ref{fig:weighted-cfg-ablation-forest} shows that increasing recovery CFG generally reduces reconstruction quality in the in-range prompt-matched and out-of-range experiments, while the prompt-mismatched experiment actually shows improvements in LPIPS and SSIM. Table~\ref{tab:weighted-cfg-ablation-best-pairs} reports the best observed sampling-distribution and recovery-CFG pairs after averaging equally over the five sampling ratios and five trials. In this scenario, sampling with $\widetilde{\mu}_{c_{\mathrm{uc}}}$ gives the best image-quality (LPIPS, PSNR, and SSIM) results in eight of the nine comparisons, though the preferred recovery CFG varies. By contrast, the best weighted loss is always attained with recovery CFG $1$. Matching the CFG of $7.5$ used to generate the empirical Christoffel sampling distributions during recovery does not always produce the best reconstruction.

Table~\ref{tab:weighted-cfg-ablation-recovery-cfg} compares each fixed recovery CFG by averaging equally over the four empirical Christoffel sampling distributions, five sampling ratios, and five trials. In this scenario, all metrics worsen monotonically as recovery CFG increases in the in-range prompt-matched and out-of-range experiments. Like before, the in-range prompt-mismatched experiment differs, with mean LPIPS and SSIM improving monotonically, even though CFG $1$ gives the highest mean PSNR and lowest mean weighted loss. As we mention in Section~\ref{sec: results discussion}, changing recovery CFG alters both the recovery class and optimization problem, so the observed differences may reflect changes in how well the class approximates the target, the difficulty of optimization, or both. Allowing the generator more flexibility with CFG $1$, however, seems to improve reconstruction quality and adherence to the observed measurements.

\vspace{5em}
\begin{figure}[h]
    \centering
    \includegraphics[width=\linewidth]{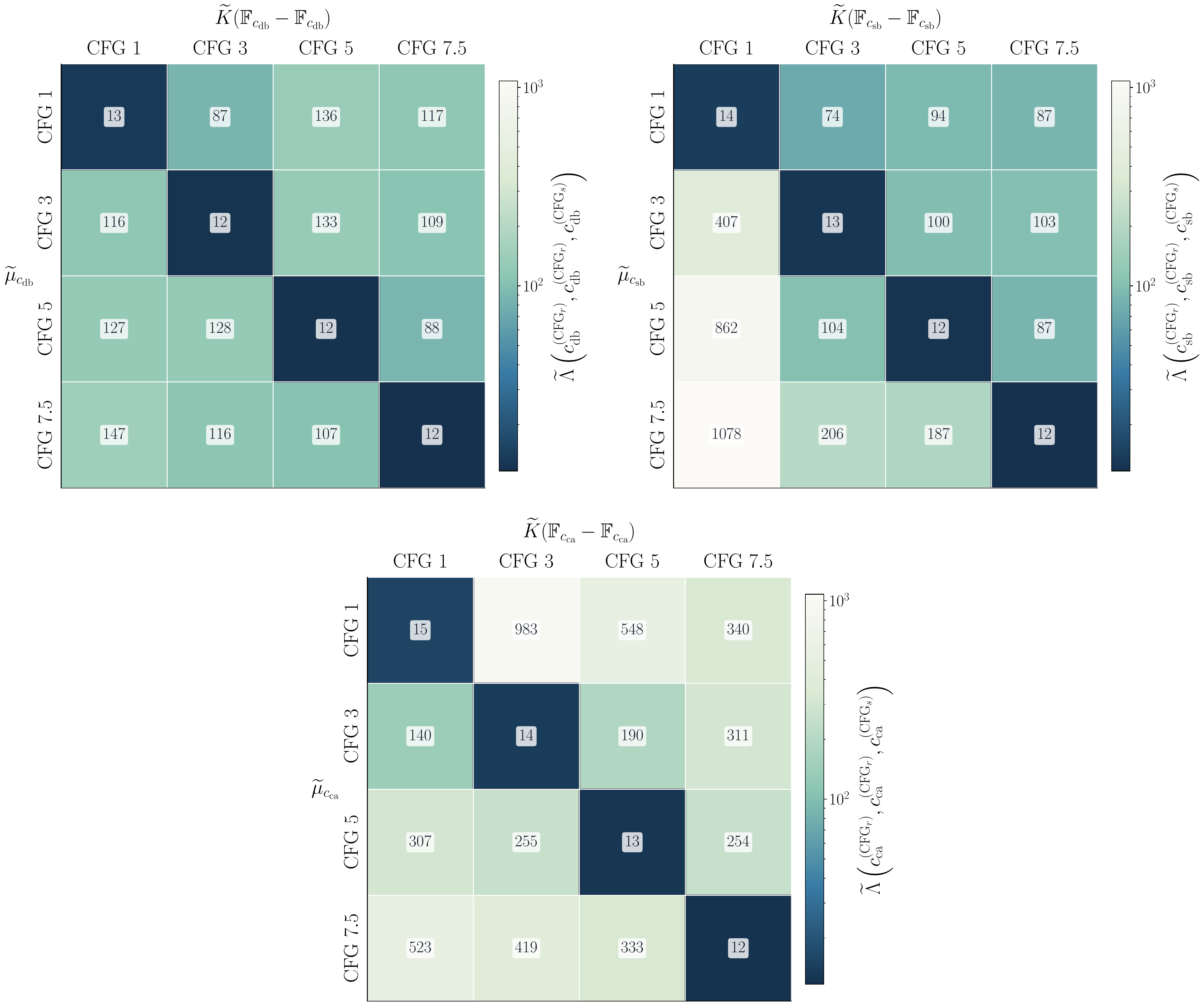}
    \caption{Empirical prompt compatibility factors for $\mathrm{CFG}\in\{1,3,5,7.5\}$. Each heatmap fixes a single prompt: $c_{\mathrm{db}}$ (top left), $c_{\mathrm{sb}}$ (top right), and $c_{\mathrm{ca}}$ (bottom). Sampling distributions use $\zeta=0$. Rows indicate the CFG used to construct $\widetilde{\mu}_c$, and columns indicate the CFG used to estimate $\widetilde K(\mathbb F_c-\mathbb F_c)$. Heatmap values are rounded to the nearest integer.}
    \label{fig:cfg-compatibility}
\end{figure}

\newpage
\begin{table}
\centering
\small
\setlength{\tabcolsep}{2.5pt}
\caption{Best sampling-distribution and recovery-CFG pairs under weighted recovery with $c_r=c_{\mathrm{sb}}$, averaged over five sampling ratios and five trials. Values are means $\pm$ 95\% confidence intervals.}
\label{tab:weighted-cfg-ablation-best-pairs}
\begin{tabular*}{\linewidth}{@{\extracolsep{\fill}}llccr@{}}
\toprule
Experiment & Metric & \begin{tabular}[c]{@{}c@{}}Sampling\\Distribution\end{tabular} & \begin{tabular}[c]{@{}c@{}}Recovery\\CFG\end{tabular} & Mean $\pm$ 95\% CI \\
\midrule
In-range prompt-matched & LPIPS $\downarrow$ & $\widetilde\mu_{c_{\mathrm{uc}}}$ & CFG $3$ & 0.3610 $\pm$ 0.0369 \\
 & Best Weighted Loss $\downarrow$ & $\widetilde\mu_{c_{\mathrm{ca}}}$ & CFG $1$ & 21.9 $\pm$ 5.5 \\
 & PSNR (dB) $\uparrow$ & $\widetilde\mu_{c_{\mathrm{uc}}}$ & CFG $3$ & 25.04 $\pm$ 1.79 \\
 & SSIM $\uparrow$ & $\widetilde\mu_{c_{\mathrm{uc}}}$ & CFG $5$ & 0.545 $\pm$ 0.061 \\
\midrule
In-range prompt-mismatched & LPIPS $\downarrow$ & $\widetilde\mu_{c_{\mathrm{uc}}}$ & CFG $7.5$ & 0.3270 $\pm$ 0.0527 \\
 & Best Weighted Loss $\downarrow$ & $\widetilde\mu_{c_{\mathrm{sb}}}$ & CFG $1$ & 25.3 $\pm$ 4.8 \\
 & PSNR (dB) $\uparrow$ & $\widetilde\mu_{c_{\mathrm{sb}}}$ & CFG $1$ & 24.01 $\pm$ 1.47 \\
 & SSIM $\uparrow$ & $\widetilde\mu_{c_{\mathrm{uc}}}$ & CFG $7.5$ & 0.578 $\pm$ 0.036 \\
\midrule
Out-of-range & LPIPS $\downarrow$ & $\widetilde\mu_{c_{\mathrm{uc}}}$ & CFG $1$ & 0.5557 $\pm$ 0.0141 \\
 & Best Weighted Loss $\downarrow$ & $\widetilde\mu_{c_{\mathrm{uc}}}$ & CFG $1$ & 51.0 $\pm$ 10.8 \\
 & PSNR (dB) $\uparrow$ & $\widetilde\mu_{c_{\mathrm{uc}}}$ & CFG $1$ & 19.09 $\pm$ 0.50 \\
 & SSIM $\uparrow$ & $\widetilde\mu_{c_{\mathrm{uc}}}$ & CFG $1$ & 0.335 $\pm$ 0.014 \\
\bottomrule
\end{tabular*}
\end{table}

\begin{table}
\centering
\small
\setlength{\tabcolsep}{3pt}
\caption{Reconstruction metrics for each recovery CFG under weighted recovery with $c_r=c_{\mathrm{sb}}$, averaged over four sampling distributions, five sampling ratios, and five trials. Values are \mbox{means $\pm$ 95\%} confidence intervals. Grey shading marks the best unrounded mean per row.}
\label{tab:weighted-cfg-ablation-recovery-cfg}
\begin{tabular*}{\linewidth}{@{\extracolsep{\fill}}lcccc@{}}
\toprule
Experiment & CFG $1$ & CFG $3$ & CFG $5$ & CFG $7.5$ \\
\midrule
\multicolumn{5}{c}{LPIPS $\downarrow$} \\
\midrule
In-range prompt-matched & $\bestmean{0.392\pm 0.010}$ & $0.406\pm 0.030$ & $0.419\pm 0.014$ & $0.461\pm 0.029$ \\
In-range prompt-mismatched & $0.385\pm 0.015$ & $0.365\pm 0.027$ & $0.363\pm 0.032$ & $\bestmean{0.346\pm 0.013}$ \\
Out-of-range & $\bestmean{0.577\pm 0.009}$ & $0.591\pm 0.011$ & $0.597\pm 0.015$ & $0.607\pm 0.010$ \\
\midrule
\multicolumn{5}{c}{Best Weighted Loss $\downarrow$} \\
\midrule
In-range prompt-matched & $\bestmean{22.9\pm 3.4}$ & $31.3\pm 7.5$ & $39.7\pm 7.2$ & $63.0\pm 13.2$ \\
In-range prompt-mismatched & $\bestmean{27.5\pm 4.3}$ & $33.4\pm 9.4$ & $43.9\pm 3.5$ & $41.5\pm 5.4$ \\
Out-of-range & $\bestmean{54.3\pm 8.3}$ & $60.3\pm 7.5$ & $76.0\pm 14.8$ & $77.1\pm 9.8$ \\
\midrule
\multicolumn{5}{c}{PSNR (dB) $\uparrow$} \\
\midrule
In-range prompt-matched & $\bestmean{24.22\pm 0.69}$ & $23.49\pm 1.07$ & $22.18\pm 0.63$ & $20.56\pm 0.65$ \\
In-range prompt-mismatched & $\bestmean{23.56\pm 0.73}$ & $22.76\pm 1.22$ & $22.52\pm 1.11$ & $22.67\pm 0.32$ \\
Out-of-range & $\bestmean{18.50\pm 0.11}$ & $18.04\pm 0.35$ & $17.91\pm 0.42$ & $17.50\pm 0.50$ \\
\midrule
\multicolumn{5}{c}{SSIM $\uparrow$} \\
\midrule
In-range prompt-matched & $\bestmean{0.515\pm 0.014}$ & $0.504\pm 0.036$ & $0.489\pm 0.028$ & $0.443\pm 0.023$ \\
In-range prompt-mismatched & $0.521\pm 0.025$ & $0.529\pm 0.033$ & $0.537\pm 0.046$ & $\bestmean{0.552\pm 0.006}$ \\
Out-of-range & $\bestmean{0.316\pm 0.008}$ & $0.302\pm 0.016$ & $0.302\pm 0.017$ & $0.289\pm 0.016$ \\
\bottomrule
\end{tabular*}
\end{table}

\begin{landscape}
\begin{figure}
    \centering
    \includegraphics[width=0.925\linewidth]{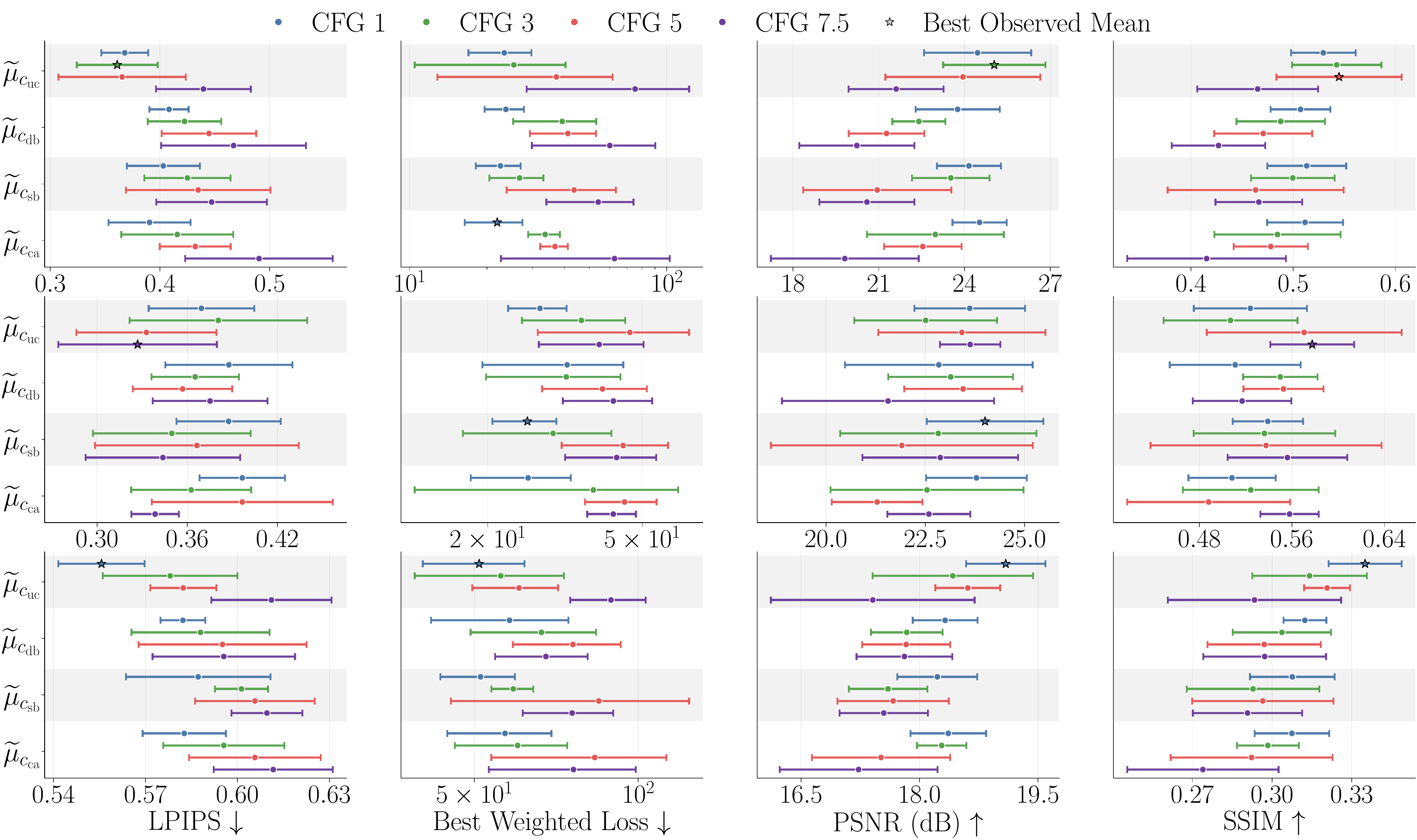}
    \caption{Recovery-CFG ablation results under weighted recovery with fixed prompt $c_r=c_{\mathrm{sb}}$: in-range prompt-matched (top), in-range prompt-mismatched (middle), and out-of-range (bottom). For each fixed sampling-distribution and recovery-CFG pair, points show means over the five sampling ratios and five trials. Bars denote $95\%$ confidence intervals.}
    \label{fig:weighted-cfg-ablation-forest}
\end{figure}
\end{landscape}

\newpage\section{Unweighted experiments}\label{app: unweighted results}
Christoffel sampling provides near-optimal sample complexity for solving weighted least-squares optimization problems. However, weighting the objective function, paired with backpropagating through the nonlinear and nonconvex VAE decoder and $20$ DDIM diffusion steps, results in a cumbersome and brittle optimization procedure. In the highly undersampled regime studied in the numerical experiments, this results in noisy recovered images, as evidenced in Figures \ref{fig:weighted-prompt-matched-recovered-images}--\ref{fig:weighted-out-of-range-recovered-images}. To explore an alternative recovery setup, we perform a series of numerical experiments with an unweighted objective, higher CFG, and a reduced optimization budget.

We adapt the procedure in Section~\ref{sec: experiments} to unweighted least-squares recovery by omitting the inverse-probability weights $1/\sqrt{\mu_{c_s}(I_j)}$ from the objective function. Given measurements $\mathbf y^{\mathrm{uw}}=\mathbf P_\Omega\mathbf F\mathbf f^*$, we seek $\widehat{\mathbf f}=G(\widehat{\mathbf z},c_r)$ by solving
\begin{equation}\label{eq: unweighted least squares}
  \widehat{\mathbf z}
  \in
  \arg\min_{\mathbf z}
  \frac{1}{2Cm}
  \left\|
    \mathbf P_\Omega\mathbf F G(\mathbf z,c_r)-\mathbf y^{\mathrm{uw}}
  \right\|_2^2,
\end{equation}
using the same channel-wise norm convention as Equation~\eqref{eq: least squares}.
The unweighted experiments use the same three target images, generator, and initialization, but construct the four empirical Christoffel sampling distributions from $500$ secants with sampling CFG $7.5$ and $\zeta=0$. We also include $\mu_{\mathrm{MCS}}$ and $\mu_{\mathrm{IS}}$ as a comparison to empirical Christoffel sampling distributions. All four recovery prompts are tested with recovery CFG $7.5$ at seven sampling ratios $m/n\in\{0.00015625,0.0003125,0.000625,0.00125,0.0025,0.005,0.01\}$, using five trials per setting. We optimize Equation~\eqref{eq: unweighted least squares} using Adam for at most $500$ steps, with a base learning rate $10^{-2}$ followed by cosine decay to $10^{-3}$, and early stopping after $35$ steps without loss improvement.

\subsection{Unweighted conditional generative compressed sensing}\label{app: unweighted reconstruction metrics}

Figure~\ref{fig:unweighted-main-forest} summarizes all $2{,}520$ reconstructions, averaging each fixed sampling-distribution and recovery-prompt pair equally over the seven sampling ratios and five trials, with Table~\ref{tab:unweighted-main-best-pairs} providing the numerical values for the best observed pair for each metric. Sampling with $\widetilde\mu_{c_{\mathrm{sb}}}$ and recovering with $c_r=c_{\mathrm{uc}}$ gives the lowest mean LPIPS in all three experiments, whereas $\mu_{\mathrm{MCS}}$ with the same recovery prompt gives the lowest unweighted loss despite poorer image quality. The same Christoffel sampling and recovery-prompt pair also gives the highest mean PSNR in all three experiments and the highest mean SSIM in two, bolstering its strong performance in image quality. Table~\ref{tab:unweighted-main-recovery-prompts} holds each recovery prompt fixed and further averages over all six sampling distributions. Recovery with $c_r=c_{\mathrm{uc}}$ gives the best observed mean in all but one row of Table~\ref{tab:unweighted-main-recovery-prompts}, including the lowest LPIPS and unweighted loss in all three experiments. Its advantage in mean LPIPS over the conditioned recovery prompts is largest in the out-of-range experiment. We defer discussion of these results as well as the following ablation study and how they compare to the weighted experiments to Section~\ref{app: unweighted weighted discussion}.

\begin{table}[h]
\centering
\small
\setlength{\tabcolsep}{1.5pt}
\caption{Best sampling-distribution and recovery-prompt pairs under unweighted recovery, averaged over seven sampling ratios and five trials. Values are means $\pm$ 95\% confidence intervals.}
\label{tab:unweighted-main-best-pairs}
\begin{tabular*}{\linewidth}{@{\extracolsep{\fill}}llccr@{}}
\toprule
Experiment & Metric & \begin{tabular}[c]{@{}c@{}}Sampling\\Distribution\end{tabular} & \begin{tabular}[c]{@{}c@{}}Recovery\\Prompt\end{tabular} & Mean $\pm$ 95\% CI \\
\midrule
In-range prompt-matched & LPIPS $\downarrow$ & $\widetilde\mu_{c_{\mathrm{sb}}}$ & $c_r=c_{\mathrm{uc}}$ & 0.4462 $\pm$ 0.0160 \\
 & Best Unweighted Loss $\downarrow$ & $\mu_{\mathrm{MCS}}$ & $c_r=c_{\mathrm{uc}}$ & $(2.11\pm 0.84)\times 10^{-3}$ \\
 & PSNR (dB) $\uparrow$ & $\widetilde\mu_{c_{\mathrm{sb}}}$ & $c_r=c_{\mathrm{uc}}$ & 23.40 $\pm$ 0.52 \\
 & SSIM $\uparrow$ & $\widetilde\mu_{c_{\mathrm{sb}}}$ & $c_r=c_{\mathrm{uc}}$ & 0.636 $\pm$ 0.014 \\
\midrule
In-range prompt-mismatched & LPIPS $\downarrow$ & $\widetilde\mu_{c_{\mathrm{sb}}}$ & $c_r=c_{\mathrm{uc}}$ & 0.3864 $\pm$ 0.0097 \\
 & Best Unweighted Loss $\downarrow$ & $\mu_{\mathrm{MCS}}$ & $c_r=c_{\mathrm{uc}}$ & $(1.52\pm 0.55)\times 10^{-3}$ \\
 & PSNR (dB) $\uparrow$ & $\widetilde\mu_{c_{\mathrm{sb}}}$ & $c_r=c_{\mathrm{uc}}$ & 23.78 $\pm$ 0.57 \\
 & SSIM $\uparrow$ & $\widetilde\mu_{c_{\mathrm{sb}}}$ & $c_r=c_{\mathrm{ca}}$ & 0.642 $\pm$ 0.014 \\
\midrule
Out-of-range & LPIPS $\downarrow$ & $\widetilde\mu_{c_{\mathrm{sb}}}$ & $c_r=c_{\mathrm{uc}}$ & 0.5236 $\pm$ 0.0065 \\
 & Best Unweighted Loss $\downarrow$ & $\mu_{\mathrm{MCS}}$ & $c_r=c_{\mathrm{uc}}$ & $(2.84\pm 0.70)\times 10^{-3}$ \\
 & PSNR (dB) $\uparrow$ & $\widetilde\mu_{c_{\mathrm{sb}}}$ & $c_r=c_{\mathrm{uc}}$ & 19.93 $\pm$ 0.13 \\
 & SSIM $\uparrow$ & $\widetilde\mu_{c_{\mathrm{sb}}}$ & $c_r=c_{\mathrm{uc}}$ & 0.574 $\pm$ 0.001 \\
\bottomrule
\end{tabular*}
\end{table}

\begin{table}[p]
\centering
\small
\setlength{\tabcolsep}{3pt}
\caption{Reconstruction metrics for each recovery prompt under unweighted recovery, averaged over six sampling distributions, seven sampling ratios, and five trials. Values are means $\pm$ 95\% confidence intervals. Grey shading marks the best unrounded mean per row.}
\label{tab:unweighted-main-recovery-prompts}
\begin{tabular*}{\linewidth}{@{\extracolsep{\fill}}lcccc@{}}
\toprule
Experiment & $c_r=c_{\mathrm{uc}}$ & $c_r=c_{\mathrm{db}}$ & $c_r=c_{\mathrm{sb}}$ & $c_r=c_{\mathrm{ca}}$ \\
\midrule
\multicolumn{5}{c}{LPIPS $\downarrow$} \\
\midrule
In-range prompt-matched & $\bestmean{0.539\pm 0.010}$ & $0.556\pm 0.005$ & $0.547\pm 0.012$ & $0.584\pm 0.010$ \\
In-range prompt-mismatched & $\bestmean{0.490\pm 0.010}$ & $0.514\pm 0.011$ & $0.503\pm 0.012$ & $0.551\pm 0.010$ \\
Out-of-range & $\bestmean{0.586\pm 0.005}$ & $0.657\pm 0.004$ & $0.656\pm 0.007$ & $0.672\pm 0.007$ \\
\midrule
\multicolumn{5}{c}{Best Unweighted Loss $\downarrow$ ($\times 10^{-2}$)} \\
\midrule
In-range prompt-matched & $\bestmean{0.86\pm 0.31}$ & $65.04\pm 13.24$ & $61.90\pm 16.39$ & $31.56\pm 15.15$ \\
In-range prompt-mismatched & $\bestmean{0.52\pm 0.13}$ & $39.43\pm 10.80$ & $39.83\pm 12.33$ & $19.37\pm 8.43$ \\
Out-of-range & $\bestmean{0.88\pm 0.23}$ & $54.56\pm 18.91$ & $78.41\pm 23.07$ & $39.28\pm 4.20$ \\
\midrule
\multicolumn{5}{c}{PSNR (dB) $\uparrow$} \\
\midrule
In-range prompt-matched & $\bestmean{20.52\pm 0.44}$ & $18.76\pm 0.38$ & $18.64\pm 0.51$ & $18.78\pm 0.21$ \\
In-range prompt-mismatched & $\bestmean{20.72\pm 0.49}$ & $19.50\pm 0.62$ & $19.65\pm 0.55$ & $19.80\pm 0.33$ \\
Out-of-range & $\bestmean{18.25\pm 0.08}$ & $16.45\pm 0.17$ & $16.32\pm 0.26$ & $16.48\pm 0.15$ \\
\midrule
\multicolumn{5}{c}{SSIM $\uparrow$} \\
\midrule
In-range prompt-matched & $\bestmean{0.585\pm 0.009}$ & $0.570\pm 0.011$ & $0.570\pm 0.010$ & $0.574\pm 0.006$ \\
In-range prompt-mismatched & $0.593\pm 0.012$ & $0.597\pm 0.012$ & $\bestmean{0.600\pm 0.011}$ & $0.595\pm 0.011$ \\
Out-of-range & $\bestmean{0.533\pm 0.005}$ & $0.508\pm 0.004$ & $0.507\pm 0.006$ & $0.517\pm 0.007$ \\
\bottomrule
\end{tabular*}
\end{table}

\begin{landscape}
\begin{figure}[p]
    \centering
    \includegraphics[width=0.935\linewidth]{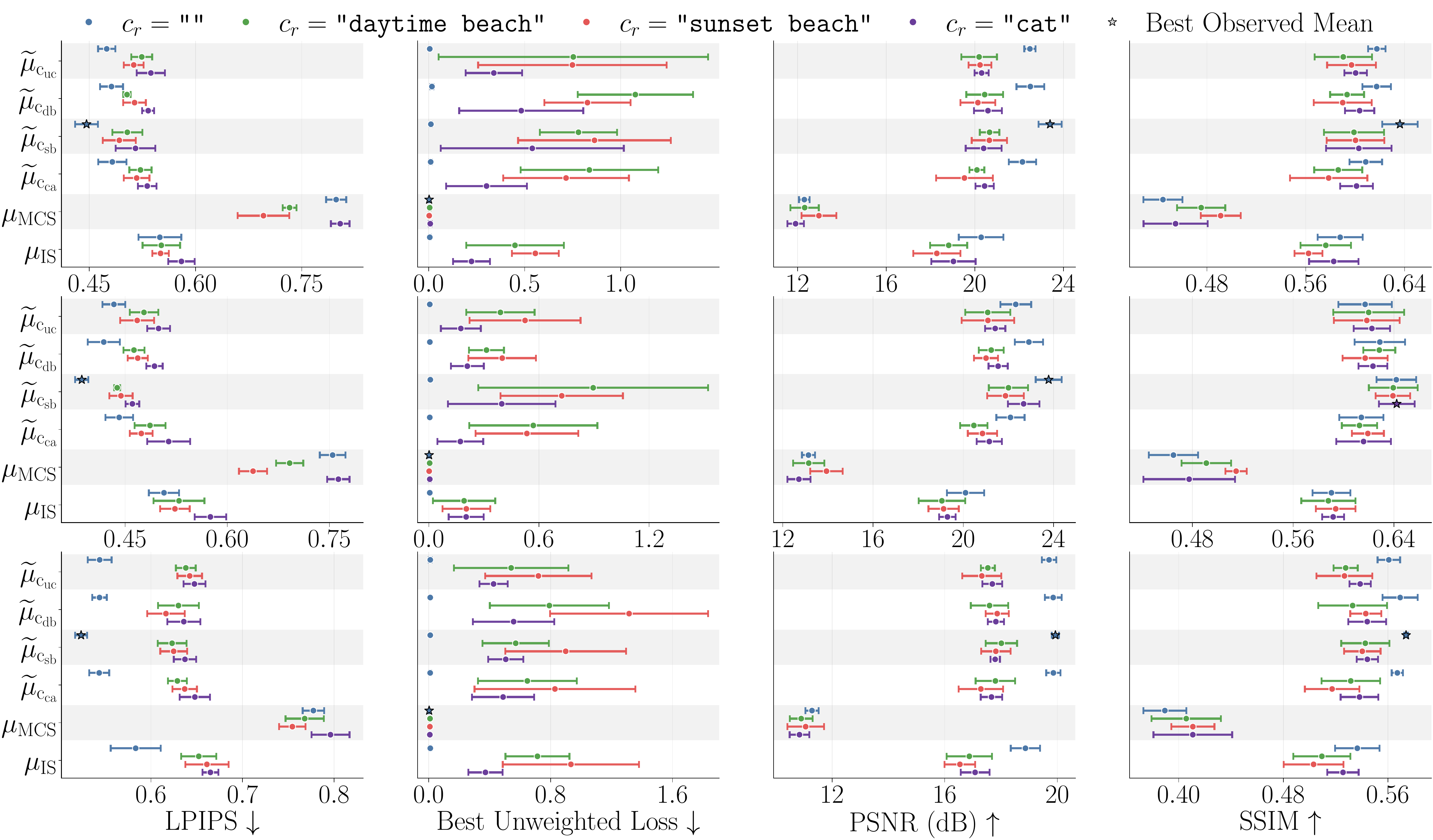}
    \caption{Aggregate reconstruction metrics under unweighted recovery for the in-range prompt-matched (top), in-range prompt-mismatched (middle), and out-of-range (bottom) experiments. For each fixed sampling-distribution and recovery-prompt pair, points show means over the seven sampling ratios and five trials. Bars denote $95\%$ confidence intervals.}
    \label{fig:unweighted-main-forest}
\end{figure}
\end{landscape}

\subsection{Unweighted CFG ablation}\label{app: unweighted cfg results}

The preceding unweighted experiments use recovery CFG $7.5$, but the weighted CFG ablation in Appendix~\ref{app: cfg ablation} suggests that lowering recovery CFG can improve reconstruction quality. We study whether this also holds for unweighted recovery by varying recovery CFG over $\{1,3,5,7.5\}$ with fixed prompt $c_r=c_{\mathrm{sb}}$, following the setup in the main unweighted experiments. We compare the four empirical Christoffel sampling distributions, all constructed with sampling CFG $7.5$ and $\zeta=0$.

Figure~\ref{fig:unweighted-cfg-forest} compares sampling-distribution and recovery-CFG pairs, averaged equally over the seven sampling ratios and five trials, with Table~\ref{tab:unweighted-cfg-ablation-best-pairs} providing the numerical values for the best observed pair for each metric. Sampling with $\widetilde\mu_{c_{\mathrm{sb}}}$ and recovering with CFG $1$ gives the lowest mean LPIPS in all three experiments. The lowest unweighted losses also result from recovery CFG $1$, but with different preferred distributions. Table~\ref{tab:unweighted-cfg-ablation-recovery-cfg} further averages over the four sampling distributions and shows that mean LPIPS and unweighted loss both worsen monotonically as recovery CFG increases in all three experiments. As in the weighted experiments, lower recovery CFG can improve image quality and measurement fit. Recovery CFG $1$ gives the best observed mean in ten of the twelve rows, with the only exceptions being SSIM in the two in-range experiments. 

\vspace{-1em}
\begin{table}[H]
\centering
\small
\setlength{\tabcolsep}{1.5pt}
\caption{Best sampling-distribution and recovery-CFG pairs under unweighted recovery with $c_r=c_{\mathrm{sb}}$, averaged over seven sampling ratios and five trials. Values are means $\pm$ 95\% confidence intervals.}
\label{tab:unweighted-cfg-ablation-best-pairs}
\begin{tabular*}{\linewidth}{@{\extracolsep{\fill}}llccr@{}}
\toprule
Experiment & Metric & \begin{tabular}[c]{@{}c@{}}Sampling\\Distribution\end{tabular} & \begin{tabular}[c]{@{}c@{}}Recovery\\CFG\end{tabular} & Mean $\pm$ 95\% CI \\
\midrule
In-range prompt-matched & LPIPS $\downarrow$ & $\widetilde\mu_{c_{\mathrm{sb}}}$ & CFG $1$ & 0.4366 $\pm$ 0.0149 \\
 & Best Unweighted Loss $\downarrow$ & $\widetilde\mu_{c_{\mathrm{uc}}}$ & CFG $1$ & $(8.70\pm 6.15)\times 10^{-3}$ \\
 & PSNR (dB) $\uparrow$ & $\widetilde\mu_{c_{\mathrm{sb}}}$ & CFG $1$ & 23.45 $\pm$ 0.55 \\
 & SSIM $\uparrow$ & $\widetilde\mu_{c_{\mathrm{sb}}}$ & CFG $3$ & 0.642 $\pm$ 0.015 \\
\midrule
In-range prompt-mismatched & LPIPS $\downarrow$ & $\widetilde\mu_{c_{\mathrm{sb}}}$ & CFG $1$ & 0.3906 $\pm$ 0.0058 \\
 & Best Unweighted Loss $\downarrow$ & $\widetilde\mu_{c_{\mathrm{ca}}}$ & CFG $1$ & $(5.87\pm 1.40)\times 10^{-3}$ \\
 & PSNR (dB) $\uparrow$ & $\widetilde\mu_{c_{\mathrm{sb}}}$ & CFG $1$ & 23.63 $\pm$ 0.33 \\
 & SSIM $\uparrow$ & $\widetilde\mu_{c_{\mathrm{sb}}}$ & CFG $5$ & 0.649 $\pm$ 0.007 \\
\midrule
Out-of-range & LPIPS $\downarrow$ & $\widetilde\mu_{c_{\mathrm{sb}}}$ & CFG $1$ & 0.5199 $\pm$ 0.0156 \\
 & Best Unweighted Loss $\downarrow$ & $\widetilde\mu_{c_{\mathrm{uc}}}$ & CFG $1$ & $(6.86\pm 1.64)\times 10^{-3}$ \\
 & PSNR (dB) $\uparrow$ & $\widetilde\mu_{c_{\mathrm{sb}}}$ & CFG $1$ & 19.99 $\pm$ 0.11 \\
 & SSIM $\uparrow$ & $\widetilde\mu_{c_{\mathrm{sb}}}$ & CFG $1$ & 0.575 $\pm$ 0.002 \\
\bottomrule
\end{tabular*}
\end{table}

\vspace{-2em}
\begin{table}[H]
\centering
\small
\setlength{\tabcolsep}{3pt}
\caption{Reconstruction metrics for each recovery CFG under unweighted recovery with $c_r=c_{\mathrm{sb}}$, averaged over four sampling distributions, seven sampling ratios, and five trials. Values are \mbox{means $\pm$ 95\%} confidence intervals. Grey shading marks the best unrounded mean per row.}
\label{tab:unweighted-cfg-ablation-recovery-cfg}
\begin{tabular*}{\linewidth}{@{\extracolsep{\fill}}lcccc@{}}
\toprule
Experiment & CFG $1$ & CFG $3$ & CFG $5$ & CFG $7.5$ \\
\midrule
\multicolumn{5}{c}{LPIPS $\downarrow$} \\
\midrule
In-range prompt-matched & $\bestmean{0.461\pm 0.013}$ & $0.470\pm 0.019$ & $0.491\pm 0.016$ & $0.509\pm 0.007$ \\
In-range prompt-mismatched & $\bestmean{0.417\pm 0.014}$ & $0.432\pm 0.016$ & $0.446\pm 0.011$ & $0.464\pm 0.012$ \\
Out-of-range & $\bestmean{0.532\pm 0.011}$ & $0.596\pm 0.007$ & $0.616\pm 0.010$ & $0.630\pm 0.011$ \\
\midrule
\multicolumn{5}{c}{Best Unweighted Loss $\downarrow$ ($\times 10^{-2}$)} \\
\midrule
In-range prompt-matched & $\bestmean{1.07\pm 0.52}$ & $8.14\pm 6.61$ & $31.23\pm 19.82$ & $78.89\pm 23.00$ \\
In-range prompt-mismatched & $\bestmean{0.87\pm 0.48}$ & $4.24\pm 3.24$ & $16.59\pm 5.56$ & $54.58\pm 16.00$ \\
Out-of-range & $\bestmean{0.94\pm 0.26}$ & $12.37\pm 5.29$ & $37.00\pm 23.21$ & $94.07\pm 28.45$ \\
\midrule
\multicolumn{5}{c}{PSNR (dB) $\uparrow$} \\
\midrule
In-range prompt-matched & $\bestmean{22.72\pm 0.45}$ & $22.42\pm 0.77$ & $21.44\pm 0.68$ & $20.14\pm 0.38$ \\
In-range prompt-mismatched & $\bestmean{22.78\pm 0.44}$ & $22.60\pm 0.59$ & $22.19\pm 0.58$ & $21.21\pm 0.57$ \\
Out-of-range & $\bestmean{19.92\pm 0.11}$ & $19.10\pm 0.20$ & $18.46\pm 0.25$ & $17.57\pm 0.25$ \\
\midrule
\multicolumn{5}{c}{SSIM $\uparrow$} \\
\midrule
In-range prompt-matched & $0.620\pm 0.009$ & $\bestmean{0.626\pm 0.016}$ & $0.610\pm 0.015$ & $0.591\pm 0.012$ \\
In-range prompt-mismatched & $0.628\pm 0.014$ & $0.634\pm 0.014$ & $\bestmean{0.636\pm 0.011}$ & $0.623\pm 0.014$ \\
Out-of-range & $\bestmean{0.570\pm 0.005}$ & $0.564\pm 0.007$ & $0.552\pm 0.005$ & $0.532\pm 0.006$ \\
\bottomrule
\end{tabular*}
\end{table}

\begin{landscape}
\begin{figure}[p]
    \centering
    \includegraphics[width=0.925\linewidth]{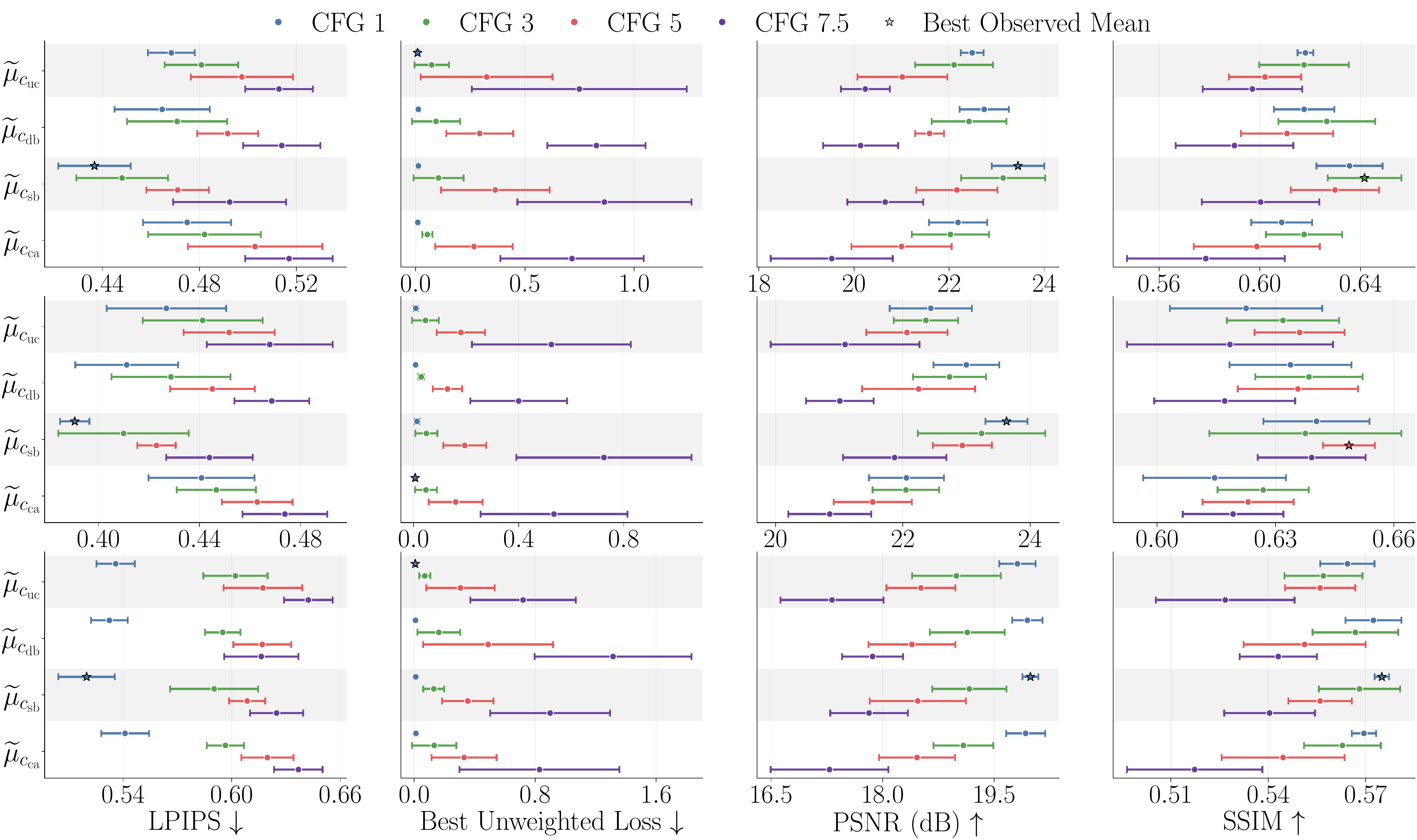}
    \caption{Recovery-CFG ablation results under unweighted recovery with fixed prompt $c_r=c_{\mathrm{sb}}$: in-range prompt-matched (top), in-range prompt-mismatched (middle), and out-of-range (bottom). For each fixed sampling-distribution and recovery-CFG pair, points show means over the seven sampling ratios and five trials. Bars denote $95\%$ confidence intervals.}
    \label{fig:unweighted-cfg-forest}
\end{figure}
\end{landscape}

\subsection{Comparison with the weighted experiments}\label{app: unweighted weighted discussion}

The unweighted experiments produce useful reconstructions from substantially fewer observed frequencies and with a smaller optimization budget. For example, sampling with $\widetilde\mu_{c_{\mathrm{sb}}}$ and recovering with $c_r=c_{\mathrm{uc}}$ gives mean LPIPS $0.457$ in the out-of-range experiment at $m/n=0.0025$, using only $655$ frequencies. This is on par with or better than the results at any sampling ratio tested in the main weighted experiments. At the shared sampling ratio $m/n=0.01$ between the weighted and unweighted experiments, we observe a few distinct patterns. With $c_r=c_{\mathrm{sb}}$ and recovery CFG fixed at $1$, and averaging over the four empirical Christoffel distributions and five trials, we find that unweighted recovery reduces mean LPIPS from $0.534$ to $0.290$, $0.476$ to $0.266$, and $0.678$ to $0.341$ in the prompt-matched, prompt-mismatched, and out-of-range experiments, respectively. Figure~\ref{fig:weighted-unweighted-ood-comparison} shows how the unweighted optimization problem results in a more natural recovered image compared to the substantially noisier image recovered using the weighted objective in the out-of-range experiment. Notably, these improvements occur within at most $500$ optimization steps rather than $2{,}000$. Although these comparisons also involve other changes in the experimental setup, for practical conditional generative compressed sensing, it is clear that the optimization procedure, including the loss function, deserves attention alongside the sampling distributions and recovery prompts to produce better results.

\begin{figure}
    \centering
    \includegraphics[width=\linewidth]{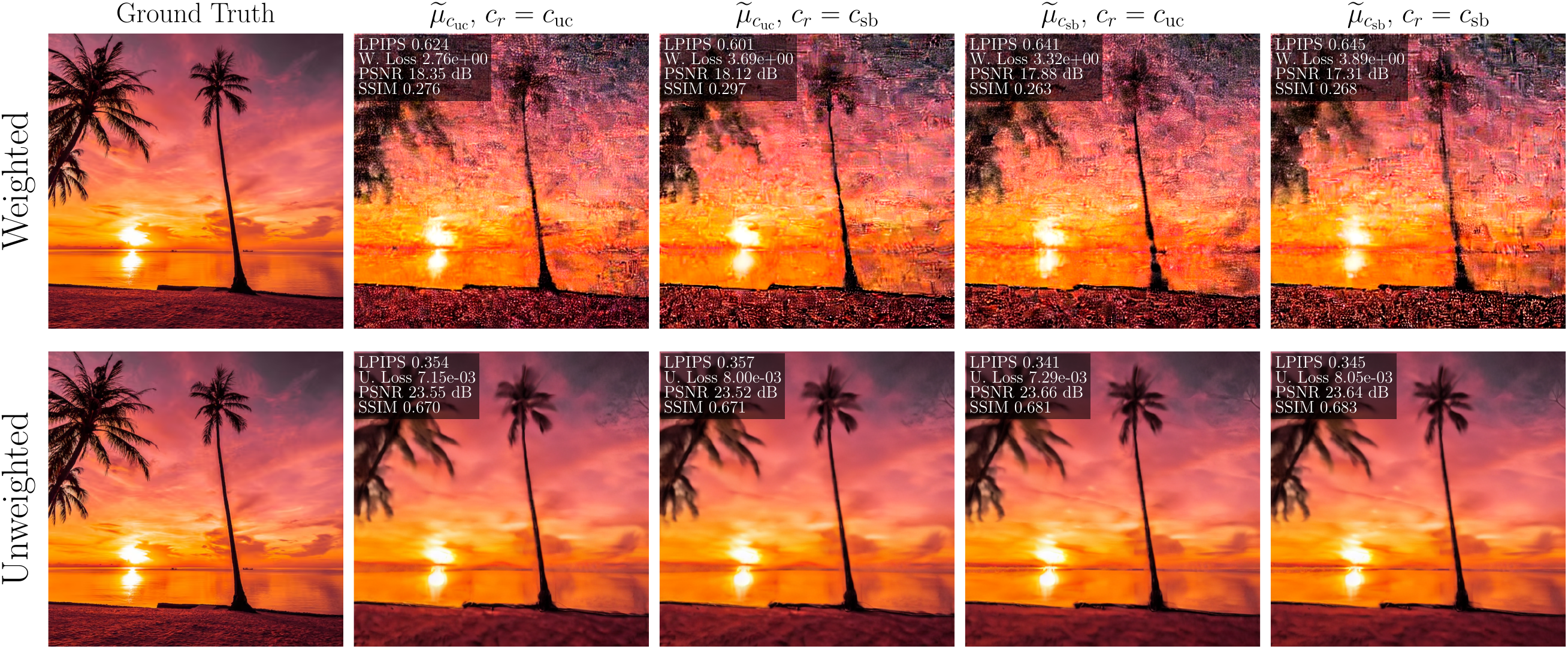}
    \caption{Selected weighted (top) and unweighted (bottom) reconstructions at $m/n=0.01$ for the out-of-range experiment. Columns compare sampling with $\widetilde\mu_{c_{\mathrm{uc}}}$ and $\widetilde\mu_{c_{\mathrm{sb}}}$, each with recovery prompts $c_r\in\{c_{\mathrm{uc}},c_{\mathrm{sb}}\}$. Prompt-conditioned recovery uses CFG $1$ in both rows.}
    \label{fig:weighted-unweighted-ood-comparison}
\end{figure}

Despite the promising results suggested by the improved LPIPS scores, high recovery CFG and an unweighted objective do not ensure that the recovered content is correct. In fact, Figure~\ref{fig:unweighted-selected-hallucinations} depicts how some recovered images in the unweighted regime exhibit substantial hallucinations with recovery CFG $7.5$. With $\mu_{\mathrm{MCS}}$ and $c_r=c_{\mathrm{ca}}$, the generator can replace the beach with a cat that follows the recovery prompt rather than the target image. By comparison, the weighted $\mu_{\mathrm{MCS}}$ reconstructions in Figures~\ref{fig:weighted-prompt-matched-recovered-images}--\ref{fig:weighted-out-of-range-recovered-images} show little recognizable scene structure even at $m/n=0.03$. Hallucinations are not restricted to uniform sampling. At $m/n=0.01$, sampling with $\widetilde\mu_{c_{\mathrm{sb}}}$ and recovering with $c_r=c_{\mathrm{ca}}$ can replace parts of the trees with silhouettes resembling a cat's head or body. Even recovery with $c_r=c_{\mathrm{sb}}$, which describes the target content, can introduce incorrect structures. Figure~\ref{fig:unweighted-selected-hallucinations} shows an invented person when sampling with $\mu_{\mathrm{MCS}}$ and missing trees with $\widetilde\mu_{c_{\mathrm{sb}}}$. We observed no such hallucinations in any of the recovered images in the weighted experiments.

\begin{figure}
    \centering
    \includegraphics[width=\linewidth]{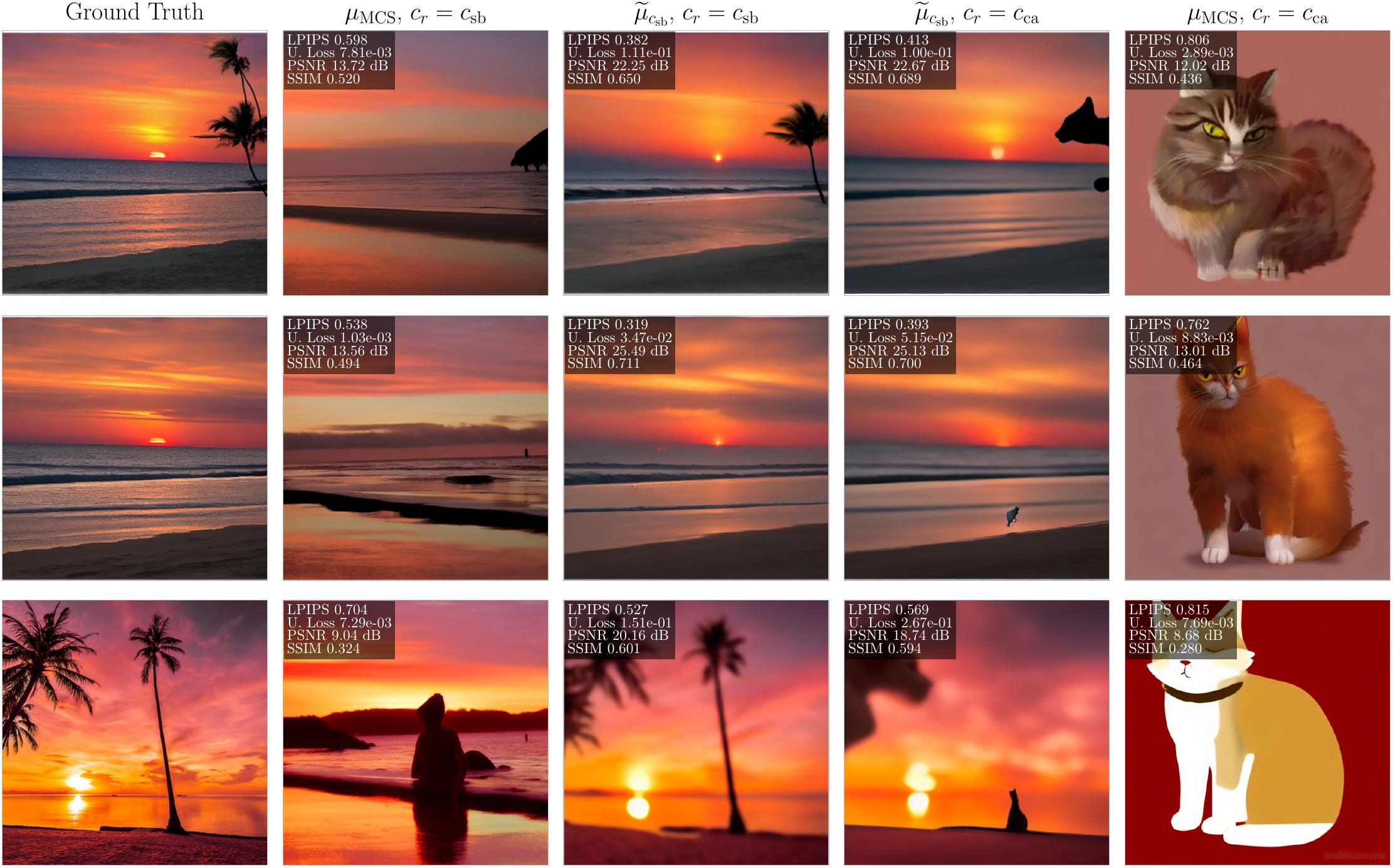}
    \caption{Selected unweighted reconstructions at $m/n=0.01$ with recovery CFG $7.5$ for the in-range prompt-matched (top), in-range prompt-mismatched (middle), and out-of-range (bottom) experiments.}
    \label{fig:unweighted-selected-hallucinations}
\end{figure}

The recovery prompt and CFG help explain some of these differences. Although $c_r=c_{\mathrm{uc}}$ gives the best observed mean in all but one row of Table~\ref{tab:unweighted-main-recovery-prompts}, lowering recovery CFG to $1$ with $c_r=c_{\mathrm{sb}}$ largely removes this advantage when comparing the same four Christoffel sampling distributions. Table~\ref{tab:unweighted-cfg-ablation-recovery-cfg} shows that lower CFG improves mean LPIPS and unweighted loss in all three scenarios, much like in the weighted recovery experiments. Even when sampling distributions are constructed with high CFG, lower recovery CFG can improve image reconstruction, despite a potential increase in approximation penalty governed by mismatched classes in the prompt compatibility factor.

In contrast to the weighted experiments, sampling with $\widetilde\mu_{c_{\mathrm{sb}}}$ gives the best LPIPS pairs throughout the unweighted experiments, as shown in Tables~\ref{tab:unweighted-main-best-pairs} and~\ref{tab:unweighted-cfg-ablation-best-pairs}. Its stronger performance may partly reflect the absence of sampling regularization. Unlike the distributions in the weighted experiments, which mix equally with a uniform component, the unweighted reconstructions sample from distributions with $\zeta = 0$. Recall that the regularization parameter $\zeta$ stabilizes the Christoffel distributions and their low-probability tails and ensures they have full support so quantities like the prompt compatibility factor $\Lambda$ are well-defined. It is thus important to determine how to set $\zeta$ in practice to balance the benefits afforded by Christoffel sampling with the stability of the distributions.

Together, these results show how many choices must be considered simultaneously in conditional generative compressed sensing. Our theory motivates Christoffel sampling and a weighted objective through sample complexity and recovery guarantees, but practical reconstruction also depends on the generator itself, sampling and recovery prompts, CFG, regularization, and optimization procedure. Lower recovery CFG can improve reconstruction even when it differs from sampling CFG, and omitting sampling regularization may retain useful prompt-dependent differences between distributions. Theoretically motivated choices may not give the best practical results, though deviating too far may result in recovered images with substantial hallucinations. Future work should separate the effects of these choices and study how to optimize them jointly for accurate and reliable recovery in conditional generative compressed sensing.

\newpage\section{Proofs}\label{app:theory-proofs}

In this appendix we collect the proofs for the arguments used throughout the paper. As a reminder, we redefine the following items. Let 
\[
  \mathbb B_{c_r}
  =
  \left\{
    \frac{\mathbf h}{\left\|\mathbf h\right\|_2}
    \colon
    \mathbf h\in
    \left(\mathbb F_{c_r}-\mathbb F_{c_r}\right)
    \setminus\{\mathbf 0\}
  \right\}
\]
denote the normalized secant set. For $\xi>0$, we also define the truncated normalized secant set by
\[
  \mathbb{B}_{c_r,\xi}
  =
  \left\{
    \frac{\mathbf h}{\left\|\mathbf h\right\|_2}
    \colon
    \mathbf h\in
    \left(\mathbb{F}_{c_r}-\mathbb{F}_{c_r}\right)\setminus\left\{\mathbf 0\right\},
    \ \left\|\mathbf h\right\|_2\geq \xi
  \right\},
\]
which will come into play when proving Lipschitz guarantees. Finally, we define
\[
\snorm{\mathbf f}{\mu_{c_s}} = 
\left(
    \max_{i \in D}
    \frac{|\mathbf P_i\mathbf F\mathbf f|^2}{\mu_{c_s}(i)}
  \right)^{1/2}
  \quad \text{and} \qquad 
  \underline{\mu}
  =
  \min_{i \in D}\mu_{c_s}\left(i\right) > 0.
\]

\subsection{Proofs of the CS4ML secant-stability statements}\label{app:end-to-srec}

We begin with a quantitative secant-stability estimate in the sampling seminorm, from which the empirical-nondegeneracy and zero-slack S-REC statement under Christoffel sampling in Theorem~\ref{thm:christoffel-srec} follows.

\begin{theorem}[Uniform secant stability from a finite cover]
  \label{thm:secant-stability}
    Fix $0<\varepsilon,\delta<1$, and let $\mathcal N\subseteq \mathbb{B}_{c_r}$ be an $\eta$-net for $\mathbb{B}_{c_r}$ in the sampling seminorm $\snorm{\cdot}{\mu_{c_s}}$, where $0<\eta<\sqrt{1-\varepsilon}$. If
  \[
    m \gtrsim \varepsilon^{-2}\Lambda(c_r,c_r,c_s)\log\left(\frac{2\left|\mathcal N\right|}{\delta}\right),
  \]
  then with probability at least $1-\delta$,
  \[
    \left(\sqrt{1-\varepsilon}-\eta\right)\left\|\mathbf{h}\right\|_2
    \leq
    \left\|\mathbf{A}_\Omega \mathbf{h}\right\|_2
    \leq
    \left(\sqrt{1+\varepsilon}+\eta\right)\left\|\mathbf{h}\right\|_2,
    \qquad
    \mathbf{h}\in \mathbb{F}_{c_r}-\mathbb{F}_{c_r}.
  \]
\end{theorem}

Before proving the theorem, we record the auxiliary lemmas used in the net-extension argument.

\begin{lemma}[Measurement norm bound by the sampling seminorm]
  \label{lem:seminorm-domination}
  For every $\mathbf{g}\in\mathbb{X}$, we have that 
  $
    \left\|\mathbf{A}_\Omega\mathbf{g}\right\|_2
    \leq
    \snorm{\mathbf{g}}{\mu_{c_s}}.
  $
\end{lemma}

\begin{proof}
\phantomsection\label{proof:seminorm-domination}
  For each sampled index $I_j$,
  \[
    \frac{\left|\mathbf{P}_{I_j}\mathbf{F}\mathbf{g}\right|^2}{\mu_{c_s}\left(I_j\right)}
    \leq 
    \snorm{\mathbf g}{\mu_{c_s}}^2.
  \]
  Averaging over $j=1,\dots,m$ gives
  \[
    \left\|\mathbf{A}_\Omega\mathbf{g}\right\|_2^2
    =
    \frac{1}{m}\sum_{j=1}^m
    \frac{\left|\mathbf{P}_{I_j}\mathbf{F}\mathbf{g}\right|^2}{\mu_{c_s}\left(I_j\right)}
    \leq
    \snorm{\mathbf g}{\mu_{c_s}}^2.
  \]
  Taking square roots yields the claim.
\end{proof}

Lemma~\ref{lem:seminorm-domination} is the basic comparison that lets us pass from the sampling seminorm to the actual measurement operator. The following lemma also shows how the sampling seminorm is controlled by the Euclidean norm.

\begin{lemma}[Sampling seminorm bound by the Euclidean norm]
    \label{lem:active-support-seminorm-comparison}
  For every $\mathbf g\in \mathbb X$, we have that
  $
    \snorm{\mathbf g}{\mu_{c_s}}
    \leq
    \underline{\mu}^{-1/2}
    \left\|\mathbf g\right\|_2.
  $
\end{lemma}

\begin{proof}
\phantomsection\label{proof:active-support-seminorm-comparison}
  Let $\mathbf g \in \mathbb X$. Then
  \[
    \snorm{\mathbf g}{\mu_{c_s}}^2
    =
    \max_{i \in D}
    \frac{\left|\mathbf P_i\mathbf F\mathbf g\right|^2}{\mu_{c_s}\left(i\right)}
    \leq
    \underline{\mu}^{-1}
    \max_{i \in D}
    \left|\mathbf P_i\mathbf F\mathbf g\right|^2.
  \]
  Using $\max_i a_i^2 \leq \sum_i a_i^2$ and that $\mathbf{F}$ is the discrete unitary Fourier transform, Parseval's theorem gives us
  \[
    \snorm{\mathbf g}{\mu_{c_s}}^2
    \leq
    \underline{\mu}^{-1}
    \sum_{i \in D}
    \left|\mathbf P_i\mathbf F\mathbf g\right|^2
    =
    \underline{\mu}^{-1}
    \left\|\mathbf g\right\|_2^2.
  \]
  Taking square roots gives the claim.
\end{proof}

We now record a fixed-direction concentration estimate for the weighted subsampled Fourier operator, which will later be combined with a union bound over the covering net.

\begin{lemma}[Concentration for a fixed secant]
  \label{lem:fixed-direction}
  Let $\mathbf{h}\in \left(\mathbb{F}_{c_r}-\mathbb{F}_{c_r}\right)\setminus\left\{\mathbf{0}\right\}$, and let $0<\varepsilon<1$. Then
  \[
    \mathbb{P}\left(
      \left|
      \left\|\mathbf{A}_\Omega\mathbf{h}\right\|_2^2
      -
      \left\|\mathbf{h}\right\|_2^2
      \right|
      >
      \varepsilon\left\|\mathbf{h}\right\|_2^2
    \right)
    \leq
    2\exp\left(-a\,\frac{m\varepsilon^2}{\Lambda(c_r,c_r,c_s)}\right)
  \]
  for some absolute constant $a > 0$.
\end{lemma}

\begin{proof}
\phantomsection\label{proof:fixed-direction}
  Set
  \[
    X_j
    =
    \frac{\left|\mathbf{P}_{I_j}\mathbf{F}\mathbf{h}\right|^2}{\mu_{c_s}\left(I_j\right)},
    \qquad
    j=1,\dots,m.
  \]
  Since the indices $I_j$ are i.i.d. draws from $\mu_{c_s}$, the random variables $X_1,\dots,X_m$ are i.i.d., and
  \[
    \left\|\mathbf{A}_\Omega\mathbf{h}\right\|_2^2
    =
    \frac{1}{m}\sum_{j=1}^m X_j.
  \]
Since $\mu_{c_s}$ has full support,
\[
  \mathbb E[X_j]
  =
  \sum_{i\in D}
  \mu_{c_s}(i)
  \frac{|\mathbf P_i\mathbf F\mathbf h|^2}{\mu_{c_s}(i)}
  =
  \sum_{i\in D}
  |\mathbf P_i\mathbf F\mathbf h|^2
  =
  \|\mathbf h\|_2^2.
\]

  Now since $\mathbf{h}\in\mathbb{F}_{c_r}-\mathbb{F}_{c_r}$, the definition of $K(\mathbb{F}_{c_r}-\mathbb{F}_{c_r})(i)$ gives
  \[
    \frac{\left|\mathbf{P}_i\mathbf{F}\mathbf{h}\right|^2}{\left\|\mathbf{h}\right\|_2^2}
    \leq
    K(\mathbb{F}_{c_r}-\mathbb{F}_{c_r})(i),
    \qquad
    i\in D.
  \]
  Dividing by $\mu_{c_s}\left(i\right)$ yields
  \[
    \frac{\left|\mathbf{P}_i\mathbf{F}\mathbf{h}\right|^2}{\mu_{c_s}\left(i\right)}
    \leq
    \frac{K(\mathbb{F}_{c_r}-\mathbb{F}_{c_r})(i)}{\mu_{c_s}\left(i\right)}
    \left\|\mathbf{h}\right\|_2^2
    \leq
    \Lambda(c_r,c_r,c_s)\left\|\mathbf{h}\right\|_2^2.
  \]
  Hence
  \[
    0\leq X_j\leq \Lambda(c_r,c_r,c_s)\left\|\mathbf{h}\right\|_2^2.
  \]
  Let
  \[
    M=\Lambda(c_r,c_r,c_s)\left\|\mathbf{h}\right\|_2^2.
  \]
  Since $0\leq X_j\leq M$, one has $X_j^2\leq M X_j$, and therefore
  \[
    \mathbb{E}\left[X_j^2\right]
    \leq
    M\,\mathbb{E}\left[X_j\right]
    =
    \Lambda(c_r,c_r,c_s)\left\|\mathbf{h}\right\|_2^4.
  \]
  Thus
  \[
    \operatorname{Var}\left(X_j\right)
    \leq
    \Lambda(c_r,c_r,c_s)\left\|\mathbf{h}\right\|_2^4.
  \]

  Now we apply Bernstein's inequality to
  \[
    Y_j=X_j-\mathbb{E}\left[X_j\right].
  \]
  These random variables are i.i.d., mean zero, satisfy $\left|Y_j\right|\leq M$, and have variance bounded as above. Therefore
  \[
    \mathbb{P}\left(
      \left|
      \frac{1}{m}\sum_{j=1}^m Y_j
      \right|
      >
      \varepsilon\left\|\mathbf{h}\right\|_2^2
    \right)
    \leq
    2\exp\left(-a\,\frac{m\varepsilon^2}{\Lambda(c_r,c_r,c_s)}\right),
  \]
  for some absolute constant $a > 0$, which is exactly the claim.
\end{proof}

Finally, we show that control on a finite net extends to the entire normalized secant set, provided the net is taken in the sampling seminorm. That is to say, once the sampling operator $\mathbf A_\Omega$ is controlled on a net, this lemma extends the same bounds to the full class by using closeness in the sampling seminorm.

\begin{lemma}[Extension from a net to the secant set]
  \label{lem:net-extension}
  Let $\mathbb{B}\subseteq \left\{\mathbf{h}\in\mathbb{X}\colon\left\|\mathbf{h}\right\|_2=1\right\}$ and let $\mathcal N\subseteq \mathbb{B}$ be an $\eta$-net for $\mathbb{B}$ in the seminorm $\snorm{\cdot}{\mu_{c_s}}$. Suppose that for $0<\varepsilon<1$,
  \[
    \sqrt{1-\varepsilon}
    \leq
    \left\|\mathbf{A}_\Omega\mathbf{u}\right\|_2
    \leq
    \sqrt{1+\varepsilon},
    \qquad
    \mathbf{u}\in\mathcal N.
  \]
  Then, for every $\mathbf{h}\in\mathbb{B}$,
  \[
    \sqrt{1-\varepsilon}-\eta
    \leq
    \left\|\mathbf{A}_\Omega\mathbf{h}\right\|_2
    \leq
    \sqrt{1+\varepsilon}+\eta.
  \]
\end{lemma}

\begin{proof}
\phantomsection\label{proof:net-extension}
  Fix $\mathbf{h}\in\mathbb{B}$. Since $\
  mathcal N$ is an $\eta$-net for $\mathbb{B}$ in the sampling seminorm, there exists some $\mathbf{u}\in\mathcal N$ such that
  \[
    \snorm{\mathbf h - \mathbf u}{\mu_{c_s}}\leq \eta.
  \]
  By Lemma~\ref{lem:seminorm-domination},
  \[
    \left\|\mathbf{A}_\Omega\left(\mathbf{h}-\mathbf{u}\right)\right\|_2\leq \eta.
  \]
  Since $\mathbf{u}\in\mathcal N$,
  \[
    \sqrt{1-\varepsilon}
    \leq
    \left\|\mathbf{A}_\Omega\mathbf{u}\right\|_2
    \leq
    \sqrt{1+\varepsilon}.
  \]
  The triangle and reverse-triangle inequalities now yield
  \[
    \left\|\mathbf{A}_\Omega\mathbf{h}\right\|_2
    \leq
    \sqrt{1+\varepsilon}+\eta
    \qquad\text{and}\qquad
    \left\|\mathbf{A}_\Omega\mathbf{h}\right\|_2
    \geq
    \sqrt{1-\varepsilon}-\eta,
  \]
  which gives us the desired bounds.
\end{proof}

We are now equipped to prove Theorem~\ref{thm:secant-stability}, which follows by combining fixed-direction concentration on the net with the extension step above.

\begin{proof}[Proof of Theorem~\ref{thm:secant-stability}]
\phantomsection\label{proof:secant-stability}
  For each point $\mathbf u\in\mathcal N$, define the event
  \[
    E_{\mathbf u}
    =
    \left\{
      \left|
      \left\|\mathbf A_\Omega\mathbf u\right\|_2^2
      -
      \left\|\mathbf u\right\|_2^2
      \right|
      \leq
      \varepsilon\left\|\mathbf u\right\|_2^2
    \right\}.
  \]
  Since $\mathcal N\subseteq \mathbb B_{c_r}$, for each
  $\mathbf u\in\mathcal N$ there exists
  \[
    \mathbf s_{\mathbf u}
    \in
    \left(\mathbb F_{c_r}-\mathbb F_{c_r}\right)
    \setminus\{\mathbf 0\}
  \]
  such that
  \[
    \mathbf u
    =
    \frac{\mathbf s_{\mathbf u}}
    {\left\|\mathbf s_{\mathbf u}\right\|_2}.
  \]
  Applying Lemma~\ref{lem:fixed-direction} to $\mathbf s_{\mathbf u}$ gives
      \begin{equation}\label{eq:fixed-dir-event}
      \mathbb P\left(
        \left|
        \left\|\mathbf A_\Omega \mathbf s_{\mathbf u}\right\|_2^2
        -
        \left\|\mathbf s_{\mathbf u}\right\|_2^2
        \right|
        >
        \varepsilon\left\|\mathbf s_{\mathbf u}\right\|_2^2
      \right)
      \leq
      2\exp\left(
        -a\,\frac{m\varepsilon^2}{\Lambda(c_r,c_r,c_s)}
      \right)
    \end{equation}
  for some absolute constant $a>0$. Since
  \[
    \left\|\mathbf A_\Omega \mathbf u\right\|_2^2
    =
    \frac{\left\|\mathbf A_\Omega \mathbf s_{\mathbf u}\right\|_2^2}
    {\left\|\mathbf s_{\mathbf u}\right\|_2^2},
    \qquad
    \left\|\mathbf u\right\|_2=1,
  \]
  the event in Equation~\eqref{eq:fixed-dir-event} is exactly $E_{\mathbf u}^c$. Therefore
  \[
    \mathbb P\left(E_{\mathbf u}^c\right)
    \leq
    2\exp\left(
      -a\,\frac{m\varepsilon^2}{\Lambda(c_r,c_r,c_s)}
    \right).
  \]

  Taking a union bound over $\mathcal N$ and using the lower bound on $m$,
  we find that with probability at least $1-\delta$,
  \[
    \sqrt{1-\varepsilon}
    \leq
    \left\|\mathbf A_\Omega\mathbf u\right\|_2
    \leq
    \sqrt{1+\varepsilon},
    \qquad
    \mathbf u\in\mathcal N.
  \]
  Applying Lemma~\ref{lem:net-extension} with
  $\mathbb B=\mathbb B_{c_r}$ gives
  \begin{equation}\label{eq: Aw bounded}
    \sqrt{1-\varepsilon}-\eta
    \leq
    \left\|\mathbf A_\Omega\mathbf w\right\|_2
    \leq
    \sqrt{1+\varepsilon}+\eta,
    \qquad
    \mathbf w\in\mathbb B_{c_r}.
  \end{equation}

  Now let $\mathbf h \in (\mathbb F_{c_r}-\mathbb F_{c_r}) \setminus \{ \mathbf 0 \}$ and set
  $
    \mathbf w
    =
    \mathbf h / \left\|\mathbf h\right\|_2.
  $
  Then we have $\mathbf w\in\mathbb B_{c_r}$, and Equation~\eqref{eq: Aw bounded} gives
  \[
    \left(\sqrt{1-\varepsilon}-\eta\right)\left\|\mathbf h\right\|_2
    \leq
    \left\|\mathbf A_\Omega \mathbf h\right\|_2
    \leq
    \left(\sqrt{1+\varepsilon}+\eta\right)\left\|\mathbf h\right\|_2.
  \]
  This proves the theorem.
\end{proof}

Finally, we can prove the general Christoffel sample complexity statement in Theorem~\ref{thm:christoffel-srec}.

\begin{proof}[Proof of Theorem~\ref{thm:christoffel-srec}]
\phantomsection\label{proof:christoffel-srec}
  Apply Theorem~\ref{thm:secant-stability} with $\varepsilon=\tau/2$. Then, with probability at least $1-\delta$,
  \[
    \left(\sqrt{1-\tau/2}-\eta\right)\left\|\mathbf{h}\right\|_2
    \leq
    \left\|\mathbf{A}_\Omega \mathbf{h}\right\|_2
    \leq
    \left(\sqrt{1+\tau/2}+\eta\right)\left\|\mathbf{h}\right\|_2,
    \qquad
    \mathbf{h}\in \mathbb{F}_{c_r}-\mathbb{F}_{c_r}.
  \]

  For the lower bound,
  \[
    \left\|\mathbf{A}_\Omega \mathbf{h}\right\|_2^2
    \geq
    \left(\sqrt{1-\tau/2}-\eta\right)^2\left\|\mathbf{h}\right\|_2^2
    \geq
    \left(1-\tau/2-2\eta\right)\left\|\mathbf{h}\right\|_2^2
    \geq
    \left(1-\tau\right)\left\|\mathbf{h}\right\|_2^2,
  \]
  since $\eta\leq \tau/8$.

  For the upper bound,
  \[
    \left\|\mathbf{A}_\Omega \mathbf{h}\right\|_2^2
    \leq
    \left(\sqrt{1+\tau/2}+\eta\right)^2\left\|\mathbf{h}\right\|_2^2
    =
    \left(
      1+\tau/2+2\eta\sqrt{1+\tau/2}+\eta^2
    \right)\left\|\mathbf{h}\right\|_2^2.
  \]
  Since $0<\tau<1$, we have $\sqrt{1+\tau/2}\leq \sqrt{3/2}<5/4$, and therefore
  \[
    2\eta\sqrt{1+\tau/2}+\eta^2
    \leq
    \frac{5}{2}\eta+\eta^2
    \leq
    \frac{5\tau}{16}+\frac{\tau^2}{64}
    <
    \frac{\tau}{2}.
  \]
  Hence
  \[
    \left\|\mathbf{A}_\Omega \mathbf{h}\right\|_2^2
    \leq
    \left(1+\tau\right)\left\|\mathbf{h}\right\|_2^2.
  \]
  Thus $\mathbf A_\Omega$ is empirically nondegenerate on $\mathbb F_{c_r}-\mathbb F_{c_r}$ with distortion $\tau$. Applying the lower half of this estimate to $\mathbf h=\mathbf f_1-\mathbf f_2$ for arbitrary $\mathbf f_1,\mathbf f_2\in\mathbb F_{c_r}$ yields that $\mathbf A_\Omega$ satisfies
  $
  \mathrm{S\text{-}REC}(\mathbb{F}_{c_r},\sqrt{1-\tau},0).
  $
\end{proof}

\subsection{A ReLU specialization}\label{app:relu-specialization}

We now formalize the piecewise-linear structural assumption used in the main text and prove the corresponding sample complexity theorem. The result follows under the notion that the activation patterns of a bias-free ReLU network partition the latent space into finitely many polyhedral cones, and on each such cone the network is
linear~\cite{bora2017compressed, berk2023modeladapted, berk2022fourier, plan2025denoising, naderi2021beyond}. Motivated by this structure, we define the following piecewise-linear hypothesis:

\begin{definition}[$(N,k)$-Piecewise-linear]
  \label{def:generator-structure}
  Fix a recovery prompt $c_r$.
  We say that $G(\cdot,c_r)$ is \emph{$(N,k)$-piecewise-linear} on $B_2^k(R)$ if there exist polyhedral cones $Q_1,\dots,Q_{N}\subseteq \mathbb{R}^k$ and linear maps $T_1,\dots,T_{N}\colon\mathbb{R}^k\to\mathbb{X}$ such that
  \[
    B_2^k(R)
    \subseteq
    \bigcup_{j=1}^{N} Q_j,
    \qquad
    G(\mathbf z,c_r)=T_j\mathbf z,
    \quad
    \mathbf z\in B_2^k(R)\cap Q_j.
  \]
\end{definition}

This structural assumption is useful because it reduces the geometry of the generator range to a finite union of low-dimensional linear subspaces. Consequently, the normalized secant set can be covered by combining Euclidean nets on at most $2k$-dimensional subspaces, leading to the following S-REC guarantee:

\begin{theorem}[Sample complexity for piecewise-linear networks]
  \label{thm:piecewise-linear-srec}
  Assume that $G(\cdot,c_r)$ is $(N,k)$-piecewise-linear on $B_2^k(R)$. Fix $0<\tau,\delta<1$. If
  \[
    m
    \gtrsim
    \tau^{-2}
    \Lambda(c_r,c_r,c_s)
    \left[
      \log\left(N\right)
      +
      k\log\left(
        1+
        \frac{1}
        {\tau\sqrt{\underline{\mu}}}
      \right)
      +
      \log\left(\frac{1}{\delta}\right)
    \right],
  \]
  then with probability at least $1-\delta$ over the draw of $\Omega$, $\mathbf A_\Omega$ satisfies
  $
  \mathrm{S\text{-}REC}
  (
    \mathbb{F}_{c_r},
    \sqrt{1-\tau},
    0
  ).
  $
\end{theorem}

The key geometric input for Theorem~\ref{thm:piecewise-linear-srec} is the following covering estimate for the normalized secant set induced by a piecewise-linear generator.

\begin{proposition}[Covering bound for piecewise-linear secants]
  \label{prop:piecewise-linear-cover-transfer}
  Assume that $G(\cdot,c_r)$ is $(N,k)$-piecewise-linear on
  $B_2^k(R)$ in the sense of Definition~\ref{def:generator-structure}. Then for every $\eta>0$, there
  exists a set $\mathcal N\subseteq\mathbb B_{c_r}$ such that $\mathcal N$ is
  an $\eta$-net for $\mathbb B_{c_r}$ in the sampling seminorm
  $\snorm{\cdot}{\mu_{c_s}}$ and
  \[
    |\mathcal N|
    \leq
    N^2
    \left(
      1+
      \frac{2}
      {\eta\sqrt{\underline{\mu}}}
    \right)^{2k}.
  \]
\end{proposition}

\begin{proof}
\phantomsection\label{proof:piecewise-linear-cover-transfer}
  Let $Q_1,\dots,Q_{N}$ and
  $T_1,\dots,T_{N}$ be as in
  Definition~\ref{def:generator-structure}. Define
  \[
    C_j=T_j(Q_j),
    \qquad
    j=1,\dots,N.
  \]
  Since each $Q_j$ is a cone and each $T_j$ is linear, each $C_j$ is a cone
  contained in a subspace of dimension at most $k$. Moreover,
  \[
    \mathbb F_{c_r}
    \subseteq
    \bigcup_{j=1}^{N} C_j.
  \]

  For each pair $(a,b)$, define
  \[
    D_{a,b}=C_a-C_b,
    \qquad
    L_{a,b}=\operatorname{span}(D_{a,b}).
  \]
  Then
  \[
    \mathbb F_{c_r}-\mathbb F_{c_r}
    \subseteq
    \bigcup_{a,b=1}^{N}D_{a,b},
  \]
  and
  \[
    \dim L_{a,b}\leq 2k.
  \]

  Now define
  \[
    E_{a,b}
    =
    \mathbb B_{c_r}\cap D_{a,b}.
  \]
  Then
  \[
    \mathbb B_{c_r}
    \subseteq
    \bigcup_{a,b=1}^{N} E_{a,b}.
  \]
  Indeed, if $\mathbf u\in\mathbb B_{c_r}$, then
  $\mathbf u=\mathbf h/\|\mathbf h\|_2$ for some nonzero
  $\mathbf h\in\mathbb F_{c_r}-\mathbb F_{c_r}$. Since
  $\mathbf h\in D_{a,b}$ for some pair $(a,b)$ and $D_{a,b}$ is a cone, also
  $\mathbf u\in D_{a,b}$. Thus $\mathbf u\in E_{a,b}$.

  For each pair $(a,b)$ with $E_{a,b}\neq\emptyset$, let
  $\mathcal N_{a,b}\subseteq E_{a,b}$ be a maximal
  $\eta\sqrt{\underline\mu}$-separated subset of $E_{a,b}$ in
  Euclidean norm. Then $\mathcal N_{a,b}$ is an
  $\eta\sqrt{\underline\mu}$-net for $E_{a,b}$ in Euclidean norm.

  Since $E_{a,b}$ is contained in the unit sphere of the subspace $L_{a,b}$
  and $\dim L_{a,b}\leq 2k$, the standard volumetric packing bound gives
  \[
    |\mathcal N_{a,b}|
    \leq
    \left(
      1+
      \frac{2}
      {\eta\sqrt{\underline\mu}}
    \right)^{2k}.
  \]

  Define
  \[
    \mathcal N
    =
    \bigcup_{a,b=1}^{N}\mathcal N_{a,b},
  \]
  omitting empty sets. Since each $\mathcal N_{a,b}\subseteq E_{a,b}$ and
  each $E_{a,b}\subseteq\mathbb B_{c_r}$, we have
  $
    \mathcal N\subseteq\mathbb B_{c_r}.
  $
  Also,
  \[
    |\mathcal N|
    \leq
    N^2
    \left(
      1+
      \frac{2}
      {\eta\sqrt{\underline\mu}}
    \right)^{2k}.
  \]

  It remains to show that $\mathcal N$ is an $\eta$-net for
  $\mathbb B_{c_r}$ in $\snorm{\cdot}{\mu_{c_s}}$. Fix
  $\mathbf u\in\mathbb B_{c_r}$. Then $\mathbf u\in E_{a,b}$ for some
  $(a,b)$, so there exists $\mathbf v\in\mathcal N_{a,b}$ such that
  \[
    \|\mathbf u-\mathbf v\|_2
    \leq
    \eta\sqrt{\underline\mu}.
  \]
  By Lemma~\ref{lem:active-support-seminorm-comparison},
  \[
    \snorm{\mathbf u-\mathbf v}{\mu_{c_s}}
    \leq
    \underline\mu^{-1/2}
    \|\mathbf u-\mathbf v\|_2
    \leq
    \eta.
  \]
  Hence $\mathcal N$ is an $\eta$-net for $\mathbb B_{c_r}$ in
  $\snorm{\cdot}{\mu_{c_s}}$.
\end{proof}

With the covering estimate in hand, the piecewise-linear theorem is obtained by feeding that cover size into the general Christoffel sampling S-REC theorem from Appendix~\ref{app:cs4ml}.

\begin{proof}[Proof of Theorem~\ref{thm:piecewise-linear-srec}]
\phantomsection\label{proof:piecewise-linear-srec}
  Set
  $
    \eta=\tau/8.
  $
  By Proposition~\ref{prop:piecewise-linear-cover-transfer}, there exists an
  $\eta$-net $\mathcal N\subseteq\mathbb B_{c_r}$ in
  the sampling seminorm such that
  \[
    |\mathcal N|
    \leq
    N^2
    \left(
      1+
      \frac{16}
      {\tau\sqrt{\underline{\mu}}}
    \right)^{2k}.
  \]
  Therefore
  \[
    \log\left(|\mathcal N|\right)
    \leq
    2\log\left(N\right)
    +
    2k
    \log\left(
      1+
      \frac{16}
      {\tau\sqrt{\underline{\mu}}}
    \right).
  \]
  Hence the stated lower bound on $m$ implies, after adjusting the absolute
  constant hidden in $\gtrsim$, that
  \[
    m
    \gtrsim
    \tau^{-2}
    \Lambda(c_r,c_r,c_s)
    \log\left(\frac{2|\mathcal N|}{\delta}\right).
  \]
  Applying Theorem~\ref{thm:christoffel-srec} gives that, with probability
  at least $1-\delta$, $\mathbf A_\Omega$ satisfies
  $
    \mathrm{S\text{-}REC}
    (
      \mathbb F_{c_r},
      \sqrt{1-\tau},
      0
    ).
  $
\end{proof}

We now verify the ReLU sample complexity theorem from the main text by applying Theorem~\ref{thm:piecewise-linear-srec}, but first introduce one technical lemma from~\cite[Lemma S2.2]{berk2022fourier} (see also \cite{naderi2021beyond}, to which we refer the reader for the proof). 

\begin{lemma}[Finite cone decomposition for ReLU generators]
  \label{lem:relu-cone-decomposition}
  Assume that, for the fixed prompt $c_r$, the map $G(\cdot,c_r)$ is a bias-free depth-$d$ feedforward ReLU network in the sense of Definition~\ref{def:bias-free-relu} with layer widths
  $
    k = k_0 \leq k_1,\dots,k_{d-1},
    k_d = n
  $,
  with $d \geq 2$, and define
  \[
    \overline{k}_{c_r}
    =
    \left(
      \prod_{\ell=1}^{d-1} k_\ell
    \right)^{1/(d-1)}.
  \]
  Then $G(\cdot,c_r)$ is $(N,k)$-piecewise-linear on $B_2^k(R)$ in the sense of Definition~\ref{def:generator-structure}, with
  \[
      \log\left(N\right)
      \leq
      k(d-1)\log\left(\frac{2e\,\overline{k}_{c_r}}{k}\right).
    \]
\end{lemma}

Once the ReLU network has been rewritten as a piecewise-linear model with an explicit region count, we are ready to prove Theorem~\ref{thm:relu-srec}.

\begin{proof}[Proof of Theorem~\ref{thm:relu-srec}]
\phantomsection\label{proof:relu-srec}
  Lemma~\ref{lem:relu-cone-decomposition} verifies the hypotheses of Theorem~\ref{thm:piecewise-linear-srec} and supplies the stated bound on $\log\left(N\right)$. Applying Theorem~\ref{thm:piecewise-linear-srec} and absorbing constants into $\gtrsim$ gives the claim.
\end{proof}

\subsection{A Lipschitz specialization}\label{app:lipschitz-specialization}

We now prove the Lipschitz specialization stated in the main text. Throughout this subsection, assume that $G(\cdot,c_r)$ is $L$-Lipschitz on $B_2^k\left(R\right)$ with $L,R>0$. That is to say 
  \[
    \left\|
    G(\mathbf z_1,c_r)
    -
    G(\mathbf z_2,c_r)
    \right\|_2
    \leq
    L\left\|\mathbf z_1-\mathbf z_2\right\|_2,
    \qquad
    \mathbf z_1,\mathbf z_2\in B_2^k(R).
  \] 
The Euclidean covering argument below is the standard image-of-a-net construction with truncation on $B_2^k\left(R\right)$, similar to that in \cite[Lemma 4.1 and its proof]{bora2017compressed}. The only additional step is to convert the Euclidean cover into one measured in the sampling seminorm. That is to say, we first build a Euclidean cover in latent space, and then transfer that cover to the sampling seminorm.

\begin{proposition}[Covering bound for truncated Lipschitz secants]
  \label{prop:lipschitz-cover-transfer}
  Assume that $G(\cdot,c_r)$ is $L$-Lipschitz on $B_2^k(R)$. Let $0<\eta\leq 1$ and $\xi>0$. Then
  there exists a set $\mathcal N\subseteq\mathbb B_{c_r}$ such that
  $\mathcal N$ is an $\eta$-net for $\mathbb B_{c_r,\xi}$ in the sampling
  seminorm $\snorm{\cdot}{\mu_{c_s}}$ and
  \[
    |\mathcal N|
    \leq
    \left(
      1+
      \frac{8LR}
      {\xi\eta\sqrt{\underline{\mu}}}
    \right)^{2k}.
  \]
\end{proposition}

\begin{proof}
\phantomsection\label{proof:lipschitz-cover-transfer}
  Set
  \[
    s=\frac{\xi\eta\sqrt{\underline{\mu}}}{2}.
  \]
  Since $0<\eta,\underline{\mu}\leq1$, we have $s\leq\xi/2$.

  We first build an explicit Euclidean $s$-net for
  $\mathbb F_{c_r}-\mathbb F_{c_r}$ with centers inside
  $\mathbb F_{c_r}-\mathbb F_{c_r}$. Let $\mathcal Z$ be an
  $s/(2L)$-net for $B_2^k(R)$ in Euclidean norm. By the standard
  volumetric bound,
  \[
    |\mathcal Z|
    \leq
    \left(
      1+
      \frac{4LR}{s}
    \right)^k.
  \]
  Define
  \[
    \mathcal H
    =
    \left\{
      G(\mathbf z_1,c_r)-G(\mathbf z_2,c_r)
      \colon
      \mathbf z_1,\mathbf z_2\in\mathcal Z
    \right\}.
  \]
  Then
  \[
    |\mathcal H|
    \leq
    |\mathcal Z|^2
    \leq
    \left(
      1+
      \frac{4LR}{s}
    \right)^{2k}.
  \]
  Moreover, $\mathcal H$ is an $s$-net for
  $\mathbb F_{c_r}-\mathbb F_{c_r}$ in Euclidean norm. Indeed, for any
  \[
    \mathbf h
    =
    G(\mathbf z_1,c_r)-G(\mathbf z_2,c_r)
    \in
    \mathbb F_{c_r}-\mathbb F_{c_r},
  \]
  choose $\tilde{\mathbf z}_1,\tilde{\mathbf z}_2\in\mathcal Z$ such that
  \[
    \left\|\mathbf z_j-\tilde{\mathbf z}_j\right\|_2
    \leq
    \frac{s}{2L},
    \qquad
    j=1,2.
  \]
  Then, by the Lipschitz property,
  \[
    \left\|
      \mathbf h
      -
      \left(
        G(\tilde{\mathbf z}_1,c_r)
        -
        G(\tilde{\mathbf z}_2,c_r)
      \right)
    \right\|_2
    \leq
    L\left\|\mathbf z_1-\tilde{\mathbf z}_1\right\|_2
    +
    L\left\|\mathbf z_2-\tilde{\mathbf z}_2\right\|_2
    \leq
    s.
  \]

  Now define
  \[
    \mathcal N
    =
    \left\{
      \frac{\mathbf g}{\left\|\mathbf g\right\|_2}
      \colon
      \mathbf g\in\mathcal H,\ \mathbf g\neq \mathbf 0
    \right\}.
  \]
  Since every nonzero $\mathbf g\in\mathcal H$ lies in
  $\mathbb F_{c_r}-\mathbb F_{c_r}$, we have $\mathcal N\subseteq\mathbb B_{c_r}$. We claim that $\mathcal N$ covers $\mathbb B_{c_r,\xi}$ in
  $\snorm{\cdot}{\mu_{c_s}}$ with radius $\eta$. 
  
  To do so, first fix
  $\mathbf u\in\mathbb B_{c_r,\xi}$. Then
  \[
    \mathbf u=\frac{\mathbf h}{\left\|\mathbf h\right\|_2}
  \]
  for some $\mathbf h\in\mathbb F_{c_r}-\mathbb F_{c_r}$ with
  $\left\|\mathbf h \right\|_2\geq \xi$. Since $\mathcal H$ is an $s$-net for
  $\mathbb F_{c_r}-\mathbb F_{c_r}$, choose $\mathbf g\in\mathcal H$ such
  that
  \[
    \left\|\mathbf h-\mathbf g\right\|_2\leq s.
  \]
  Since $s\leq \xi/2$ and $\left\|\mathbf h\right\|_2\geq\xi$, we have
  \[
    \left\|\mathbf g\right\|_2
    \geq
    \left\|\mathbf h\right\|_2-\left\|\mathbf h-\mathbf g\right\|_2
    \geq
    \xi-s
    \geq
    \xi/2
    >
    0.
  \]
  Hence $\mathbf v=\mathbf g/\left\|\mathbf g\right\|_2$ belongs to $\mathcal N$.
  
We next estimate the Euclidean distance between the normalized vectors
\[
 \left\|\mathbf u-\mathbf v\right\|_2 = 
  \left\|
    \frac{\mathbf h}{\left\|\mathbf h\right\|_2}
    -
    \frac{\mathbf g}{\left\|\mathbf g\right\|_2}
  \right\|_2
  \leq
  \left\|
    \frac{\mathbf h-\mathbf g}{\left\|\mathbf h\right\|_2}
  \right\|_2
  +
  \left\|
    \frac{\mathbf g}{\left\|\mathbf h\right\|_2}
    -
    \frac{\mathbf g}{\left\|\mathbf g\right\|_2}
  \right\|_2.
\]
The first term is $\left\|\mathbf h-\mathbf g\right\|_2/\left\|\mathbf h\right\|_2$.
For the second term,
\[
  \left\|
    \frac{\mathbf g}{\left\|\mathbf h\right\|_2}
    -
    \frac{\mathbf g}{\left\|\mathbf g\right\|_2}
  \right\|_2
  =
  \left\|\mathbf g\right\|_2
  \left|
    \frac{1}{\left\|\mathbf h\right\|_2}
    -
    \frac{1}{\left\|\mathbf g\right\|_2}
  \right|
  =
  \frac{\left|\left\|\mathbf g\right\|_2-\left\|\mathbf h\right\|_2\right|}
       {\left\|\mathbf h\right\|_2}
  \leq
  \frac{\left\|\mathbf h-\mathbf g\right\|_2}{\left\|\mathbf h\right\|_2}.
\]
Therefore
\[
  \left\|\mathbf u-\mathbf v\right\|_2
  \leq
  \frac{2\left\|\mathbf h-\mathbf g\right\|_2}{\left\|\mathbf h\right\|_2}
  \leq
  \frac{2s}{\xi}
  =
  \eta\sqrt{\underline{\mu}}.
\]

By Lemma~\ref{lem:active-support-seminorm-comparison},
\[
  \snorm{\mathbf u-\mathbf v}{\mu_{c_s}}
  \leq
  \underline{\mu}^{-1/2}
  \left\|\mathbf u-\mathbf v\right\|_2
  \leq
  \eta.
\]
Thus $\mathcal N$ is an $\eta$-net for $\mathbb B_{c_r,\xi}$ in
$\snorm{\cdot}{\mu_{c_s}}$. Finally, we estimate
\[
  \left|\mathcal N\right|
  \leq
  \left|\mathcal H\right|
  \leq
  \left(
    1+
    \frac{4LR}{s}
  \right)^{2k}
  =
  \left(
    1+
    \frac{8LR}
         {\xi\eta\sqrt{\underline{\mu}}}
  \right)^{2k},
\]
as desired.
\end{proof}

 With this argument, we are now ready to prove Theorem~\ref{thm:lipschitz-srec}.

\begin{proof}[Proof of Theorem~\ref{thm:lipschitz-srec}]
\phantomsection\label{proof:lipschitz-srec}
  If $\xi\geq2LR$, the S-REC conclusion is immediate, so assume $\xi<2LR$.
  Set
  $
    \eta=\tau/8.
  $
  By Proposition~\ref{prop:lipschitz-cover-transfer}, there exists a set
  $\mathcal N\subseteq\mathbb B_{c_r}$ that is an $\eta$-net for
  $\mathbb B_{c_r,\xi}$ in the sampling seminorm and satisfies
  \[
    \left|\mathcal N\right|
    \leq
    \left(
      1+
      \frac{64LR}
           {\xi\tau\sqrt{\underline{\mu}}}
    \right)^{2k}.
  \]
  Therefore the stated lower bound on $m$ implies, after adjusting the
  absolute constant hidden in $\gtrsim$, that
  \[
    m
    \gtrsim
    \left(\frac{\tau}{2}\right)^{-2}
    \Lambda(c_r,c_r,c_s)
    \log\left(\frac{2\left|\mathcal N\right|}{\delta}\right).
  \]

  Since $\mathcal N\subseteq\mathbb B_{c_r}$, every
  $\mathbf u\in\mathcal N$ is a normalized nonzero secant. Applying
  Lemma~\ref{lem:fixed-direction} with $\varepsilon=\tau/2$ and taking a
  union bound over $\mathcal N$, we obtain that, with probability at least
  $1-\delta$,
  \[
    \sqrt{1-\tau/2}
    \leq
    \left\|\mathbf A_\Omega\mathbf u\right\|_2
    \leq
    \sqrt{1+\tau/2},
    \qquad
    \mathbf u\in\mathcal N.
  \]

  We now extend the lower bound from $\mathcal N$ to
  $\mathbb B_{c_r,\xi}$. Fix $\mathbf v\in\mathbb B_{c_r,\xi}$. Since
  $\mathcal N$ is an $\eta$-net for $\mathbb B_{c_r,\xi}$ in
  $\snorm{\cdot}{\mu_{c_s}}$, there exists $\mathbf u\in\mathcal N$ such that
  \[
    \snorm{\mathbf v-\mathbf u}{\mu_{c_s}}\leq \eta.
  \]
  By Lemma~\ref{lem:seminorm-domination},
  \[
    \left\|\mathbf A_\Omega\left(\mathbf v-\mathbf u\right)\right\|_2
    \leq
    \eta.
  \]
  Therefore,
  \[
    \left\|\mathbf A_\Omega\mathbf v\right\|_2
    \geq
    \left\|\mathbf A_\Omega\mathbf u\right\|_2
    -
    \left\|\mathbf A_\Omega\left(\mathbf v-\mathbf u\right)\right\|_2
    \geq
    \sqrt{1-\tau/2}-\eta.
  \]
  Since $\eta=\tau/8$ and
  \[
    \sqrt{1-\tau/2}-\sqrt{1-\tau}
    =
    \frac{\tau/2}
         {\sqrt{1-\tau/2}+\sqrt{1-\tau}}
    \geq
    \frac{\tau}{4},
  \]
  we have
  \[
    \sqrt{1-\tau/2}-\eta
    \geq
    \sqrt{1-\tau}
    +
    \frac{\tau}{4}
    -
    \frac{\tau}{8}
    \geq
    \sqrt{1-\tau}.
  \]
  Hence
  \begin{equation}
  \left\|\mathbf A_\Omega\mathbf v\right\|_2
  \geq
  \sqrt{1-\tau},
  \qquad
  \mathbf v\in\mathbb B_{c_r,\xi}.
  \label{eq:srec-lower-bound-on-normalized-secants}
\end{equation}

  Now let
  $
    \mathbf h\in\mathbb F_{c_r}-\mathbb F_{c_r}.
  $
  If $\left\|\mathbf h\right\|_2\geq\xi$, then
  $\mathbf v=\mathbf h/\left\|\mathbf h\right\|_2$ belongs to
  $\mathbb B_{c_r,\xi}$, and Equation~\eqref{eq:srec-lower-bound-on-normalized-secants} gives
  \[
    \left\|\mathbf A_\Omega\mathbf h\right\|_2
    \geq
    \sqrt{1-\tau}\left\|\mathbf h\right\|_2.
  \]
  If instead $\left\|\mathbf h\right\|_2<\xi$, then
  \[
    \sqrt{1-\tau}\left\|\mathbf h\right\|_2
    -
    \sqrt{1-\tau}\,\xi
    \leq
    0,
  \]
  and therefore
  \[
    \left\|\mathbf A_\Omega\mathbf h\right\|_2
    \geq
    0
    \geq
    \sqrt{1-\tau}\left\|\mathbf h\right\|_2
    -
    \sqrt{1-\tau}\,\xi.
  \]
  Thus in either case,
  \[
    \left\|\mathbf A_\Omega\mathbf h\right\|_2
    \geq
    \sqrt{1-\tau}\left\|\mathbf h\right\|_2
    -
    \sqrt{1-\tau}\,\xi,
    \qquad
    \mathbf h\in\mathbb F_{c_r}-\mathbb F_{c_r}.
  \]
  Hence $\mathbf A_\Omega$ satisfies
  $
    \mathrm{S\text{-}REC}
    (
      \mathbb F_{c_r},
      \sqrt{1-\tau},
      \sqrt{1-\tau}\,\xi
    ).
  $
\end{proof}

\subsection{Recovery guarantees}\label{app:agnostic-recovery}
We now prove the deterministic recovery guarantees. Once the S-REC holds, reconstruction error is controlled by the best model approximation, the measurement residual of that approximation, the noise, and the optimization error. We begin with a key relationship between the physical and weighted noise term in our observations. 

\begin{proposition}[Physical-to-weighted noise conversion]
\label{prop: physical-weighted-noise} 
For the physical noise $\mathbf u$ and weighted noise $\mathbf e$, we have that
\[
    \left \|\mathbf e \right \|_2
    \leq
    \frac{1}{\sqrt{m\underline{\mu}}}
    \left \|\mathbf u \right\|_2.
\]
\end{proposition}
\begin{proof}
\phantomsection\label{proof:physical-weighted-noise}
By the definition of $\mathbf W_\Omega$,
\[
    \left\|\mathbf e\right\|_2^2
    =
    \frac{1}{m}
    \sum_{j=1}^m
    \frac{\left|u_j\right|^2}{\mu_{c_s}(I_j)}
    \leq
    \frac{1}{m\underline{\mu}}
    \sum_{j=1}^m \left|u_j\right|^2
    =
    \frac{1}{m\underline{\mu}}
    \left\|\mathbf u\right\|_2^2.
\]
Taking square roots proves the claim.
\end{proof}

\begin{theorem}[Agnostic recovery under S-REC]
  \label{thm:agnostic-recovery-srec}
  Suppose that $\mathbf y=\mathbf A_\Omega\mathbf f^*+\mathbf e$. Let $c_r\in\mathcal C$ be the recovery prompt and suppose that
  $\widehat{\mathbf z}$ is an  $\omega$-minimizer in the sense of Equation~\eqref{eq: omega minimizer} with estimator $\widehat{\mathbf f}=G(\widehat{\mathbf z},c_r)$.
  If $\mathbf A_\Omega$ satisfies
  $\mathrm{S\text{-}REC}(\mathbb F_{c_r},\gamma,q)$
  for some $\gamma>0$ and $q\geq 0$, then
  \[
    \left\|\widehat{\mathbf f}-\mathbf f^*\right\|_2
    \leq
    \inf_{\mathbf f\in\mathbb F_{c_r}}
    \left[
      \left\|\mathbf f^*-\mathbf f\right\|_2
      +
      \frac{2}{\gamma}
      \left\|\mathbf A_\Omega\left(\mathbf f^*-\mathbf f\right)\right\|_2
    \right]
    +
    \frac{1}{\gamma}
    \left(
      \frac{2}{\sqrt{m\underline{\mu}}}
      \left\|\mathbf u\right\|_2
      +
      \omega+q
    \right).
  \]
\end{theorem}

\begin{proof}[Proof]
\phantomsection\label{proof:agnostic-recovery-srec}
  Let $\bar{\mathbf f}\in\mathbb{F}_{c_r}$ be arbitrary. Since $\widehat{\mathbf f}$ is an $\omega$-approximate weighted residual minimizer,
  \[
    \left\|
    \mathbf y
    -
    \mathbf A_\Omega\widehat{\mathbf f}
    \right\|_2
    \le
    \left\|
    \mathbf y
    -
    \mathbf A_\Omega\bar{\mathbf f}
    \right\|_2
    +
    \omega.
  \]
  Using
  \[
    \mathbf y
    =
    \mathbf A_\Omega\mathbf f^*+\mathbf e,
  \]
  this becomes
  \[
    \left\|
    \mathbf A_\Omega\left(\widehat{\mathbf f}-\mathbf f^*\right)
    -
    \mathbf e
    \right\|_2
    \le
    \left\|
    \mathbf A_\Omega\left(\bar{\mathbf f}-\mathbf f^*\right)
    -
    \mathbf e
    \right\|_2
    +
    \omega.
  \]
  Hence, by the triangle inequality,
  \[
    \left\|
    \mathbf A_\Omega\left(\widehat{\mathbf f}-\bar{\mathbf f}\right)
    \right\|_2
    \le
    2
    \left\|
    \mathbf A_\Omega\left(\bar{\mathbf f}-\mathbf f^*\right)
    \right\|_2
    +
    2\left\|\mathbf e\right\|_2
    +
    \omega.
  \]
  Since $\widehat{\mathbf f},\bar{\mathbf f}\in\mathbb{F}_{c_r}$, the S-REC assumption gives
  \[
    \gamma\left\|\widehat{\mathbf f}-\bar{\mathbf f}\right\|_2
    \le
    \left\|
    \mathbf A_\Omega\left(\widehat{\mathbf f}-\bar{\mathbf f}\right)
    \right\|_2
    +
    q.
  \]
  Therefore
  \[
    \left\|\widehat{\mathbf f}-\bar{\mathbf f}\right\|_2
    \le
    \frac{
      2
      \left\|
      \mathbf A_\Omega\left(\bar{\mathbf f}-\mathbf f^*\right)
      \right\|_2
      +
      2\left\|\mathbf e\right\|_2
      +
      \omega
      +
      q
    }{\gamma}.
  \]
  By the triangle inequality again,
  \[
    \left\|\widehat{\mathbf f}-\mathbf f^*\right\|_2
    \le
    \left\|\widehat{\mathbf f}-\bar{\mathbf f}\right\|_2
    +
    \left\|\bar{\mathbf f}-\mathbf f^*\right\|_2,
  \]
  so
  \[
    \left\|\widehat{\mathbf f}-\mathbf f^*\right\|_2
    \le
    \left\|\bar{\mathbf f}-\mathbf f^*\right\|_2
    +
    \frac{
      2
      \left\|
      \mathbf A_\Omega\left(\bar{\mathbf f}-\mathbf f^*\right)
      \right\|_2
      +
      2\left\|\mathbf e\right\|_2
      +
      \omega
      +
      q
    }{\gamma}.
  \]
  Taking the infimum over $\bar{\mathbf f}\in\mathbb{F}_{c_r}$ and applying
    Proposition~\ref{prop: physical-weighted-noise} yields the claim.
\end{proof}

We now relate this result to the best achievable measurement-domain mismatch over $\mathbb{F}_{c_r}$ when the signal we recover is in the range of the generator.

\begin{proposition}[Cross-class measurement error bound]
  \label{prop:cross-class-meas-sig}
  Let $\mathbf f^*\in \mathbb F_{c_*}$. Then, for every
  $\mathbf f\in\mathbb F_{c_r}$,
  \[
    \left\|
    \mathbf A_\Omega\left(\mathbf f^*-\mathbf f\right)
    \right\|_2
    \leq
    \sqrt{\Lambda(c_*,c_r,c_s)}
    \left\|
    \mathbf f^*-\mathbf f
    \right\|_2.
  \]
  Consequently,
  \[
    \inf_{\mathbf f\in\mathbb F_{c_r}}
    \left\|
    \mathbf A_\Omega\left(\mathbf f^*-\mathbf f\right)
    \right\|_2
    \leq
    \sqrt{\Lambda(c_*,c_r,c_s)}
    \inf_{\mathbf f\in\mathbb F_{c_r}}
    \left\|
    \mathbf f^*-\mathbf f
    \right\|_2.
  \]
\end{proposition}

\begin{proof}
\phantomsection\label{proof:cross-class-meas-sig}
  Fix $\mathbf f\in\mathbb F_{c_r}$ and set
  $
    \mathbf h=\mathbf f^*-\mathbf f.
  $
  If $\mathbf h=\mathbf 0$, the claim is immediate. Otherwise,
  since $\mathbf f^*\in\mathbb F_{c_*}$ and $\mathbf f\in\mathbb F_{c_r}$,
  we have
  $
    \mathbf h\in \mathbb F_{c_*}-\mathbb F_{c_r}.
  $
  By the definition of the cross-class Christoffel function,
  \[
    \frac{\left|\mathbf P_i\mathbf F\mathbf h\right|^2}
    {\left\|\mathbf h\right\|_2^2}
    \leq
    K\left(\mathbb F_{c_*}-\mathbb F_{c_r}\right)(i),
    \qquad i\in D.
  \]
  Therefore,
  \[
    \frac{\left|\mathbf P_i\mathbf F\mathbf h\right|^2}
    {\mu_{c_s}(i)}
    \leq
    \Lambda(c_*,c_r,c_s)\left\|\mathbf h\right\|_2^2,
    \qquad i\in D.
  \]
  Hence, for each sampled index $I_j$,
  \[
    \frac{\left|\mathbf P_{I_j}\mathbf F\mathbf h\right|^2}
    {\mu_{c_s}(I_j)}
    \leq
    \Lambda(c_*,c_r,c_s)\left\|\mathbf h\right\|_2^2.
  \]
  Averaging gives
  \[
    \left\|\mathbf A_\Omega\mathbf h\right\|_2^2
    =
    \frac{1}{m}
    \sum_{j=1}^m
    \frac{\left|\mathbf P_{I_j}\mathbf F\mathbf h\right|^2}
    {\mu_{c_s}(I_j)}
    \leq
    \Lambda(c_*,c_r,c_s)\left\|\mathbf h\right\|_2^2.
  \]
  Taking square roots proves the pointwise claim. Taking the infimum over
  $\mathbf f\in\mathbb F_{c_r}$ gives the final inequality.
\end{proof}

With this, we can now prove the recovery guarantee in Theorem~\ref{thm:prompt-mismatched-recovery}.
\begin{proof}[Proof of Theorem~\ref{thm:prompt-mismatched-recovery}]
\phantomsection\label{proof:prompt-mismatched-recovery}
  Combining Proposition~\ref{prop:cross-class-meas-sig} with
  Theorem~\ref{thm:agnostic-recovery-srec} gives the desired result.
\end{proof}

\end{document}